\documentclass{article}

\usepackage[nonatbib, preprint]{neurips_2019}

\usepackage[utf8]{inputenc} 
\usepackage[T1]{fontenc}    
\usepackage{url}            
\usepackage{booktabs}       
\usepackage{amsfonts}       
\usepackage{nicefrac}       
\usepackage{microtype}      

\usepackage{cite}

\usepackage[utf8]{inputenc}   
\usepackage[T1]{fontenc}      

\usepackage{comment}

\usepackage[tbtags]{amsmath}
\usepackage{amsthm}
\usepackage{bm}
\allowdisplaybreaks
\usepackage{amssymb,mathrsfs}
\usepackage{amsfonts}
\usepackage{upgreek}
\usepackage{graphicx}
\usepackage{float}
\usepackage{wrapfig}
\usepackage{subcaption}
\usepackage{color}
\usepackage[colorlinks = true, citecolor = blue]{hyperref}

\usepackage{stmaryrd}

\usepackage{array}
\newcolumntype{?}{!{\vrule width 1.5pt}}

\usepackage[inline]{enumitem}

\usepackage{graphicx} 
\usepackage{tikz}
\usepackage{pgfplots}
\usepackage{xcolor}
\usepackage{bbm}

\usepackage{ifthen}
\usepackage{xargs}

\usepackage[textwidth=1.8cm]{todonotes}
\usepackage{wrapfig}

\usepackage{aliascnt}
\usepackage{cleveref}
\usepackage{autonum}
\makeatletter
\newtheorem{theorem}{Theorem}
\crefname{theorem}{theorem}{Theorems}
\Crefname{Theorem}{Theorem}{Theorems}

\newtheorem*{lemma_nonumber*}{Lemma}

\newaliascnt{lemma}{theorem}

\aliascntresetthe{lemma}
\crefname{lemma}{lemma}{lemmas}
\Crefname{Lemma}{Lemma}{Lemmas}

\newaliascnt{corollary}{theorem}
\newtheorem{corollary}[corollary]{Corollary}
\aliascntresetthe{corollary}
\crefname{corollary}{corollary}{corollaries}
\Crefname{Corollary}{Corollary}{Corollaries}

\newaliascnt{proposition}{theorem}
\newtheorem{proposition}[proposition]{Proposition}
\aliascntresetthe{proposition}
\crefname{proposition}{proposition}{propositions}
\Crefname{Proposition}{Proposition}{Propositions}

\newaliascnt{definition}{theorem}

\aliascntresetthe{definition}
\crefname{definition}{definition}{definitions}
\Crefname{Definition}{Definition}{Definitions}

\newaliascnt{remark}{theorem}

\aliascntresetthe{remark}
\crefname{remark}{remark}{remarks}
\Crefname{Remark}{Remark}{Remarks}

\crefname{example}{example}{examples}
\Crefname{Example}{Example}{Examples}

\crefname{figure}{figure}{figures}
\Crefname{Figure}{Figure}{Figures}

\newtheorem{assumption}{\textbf{A}\hspace{-3pt}}
\crefformat{assumption}{{\textbf{A}}#2#1#3}

\newtheorem{assumptionF}{\textbf{F}\hspace{-3pt}}
\crefformat{assumptionF}{{\textbf{F}}#2#1#3}

\newtheorem{assumptionB}{\textbf{B}\hspace{-3pt}}
\Crefname{assumptionB}{\textbf{B}\hspace{-3pt}}{\textbf{B}\hspace{-3pt}}
\crefname{assumptionB}{\textbf{B}}{\textbf{B}}

\Crefname{assumptionC}{\textbf{C}\hspace{-3pt}}{\textbf{C}\hspace{-3pt}}
\crefname{assumptionC}{\textbf{C}}{\textbf{C}}

\Crefname{assumptionH}{\textbf{H}\hspace{-3pt}}{\textbf{H}\hspace{-3pt}}
\crefname{assumptionH}{\textbf{H}}{\textbf{H}}

\Crefname{assumptionT}{\textbf{T}\hspace{-3pt}}{\textbf{T}\hspace{-3pt}}
\crefname{assumptionT}{\textbf{T}}{\textbf{T}}

\Crefname{assumptionT}{\textbf{T}\hspace{-3pt}}{\textbf{T}\hspace{-3pt}}
\crefname{assumptionT}{\textbf{T}}{\textbf{T}}

\Crefname{assumptionL}{\textbf{L}\hspace{-3pt}}{\textbf{L}\hspace{-3pt}}
\crefname{assumptionL}{\textbf{L}}{\textbf{L}}

\Crefname{assumptionQ}{\textbf{Q}\hspace{-3pt}}{\textbf{Q}\hspace{-3pt}}
\crefname{assumptionQ}{\textbf{Q}}{\textbf{Q}}

\Crefname{assumptionAR}{\textbf{AR}\hspace{-3pt}}{\textbf{AR}\hspace{-3pt}}
\crefname{assumptionAR}{\textbf{AR}}{\textbf{AR}}

\makeatletter
\newcommand*{\centerfloat}{%
  \parindent \z@
  \leftskip \z@ \@plus 1fil \@minus \textwidth
  \rightskip\leftskip
  \parfillskip \z@skip}
\makeatother

\usetikzlibrary{spy}
\usetikzlibrary[patterns] 
\usetikzlibrary{calc,decorations.pathreplacing} 
\usetikzlibrary{shadows.blur} 
\usetikzlibrary{shapes.symbols}

\def\funsig{\varphi_{\Sigma}}
\def\rmg{\mathrm{g}}
\def\rmf{\mathrm{f}}

\def\Mip{\mathtt{M}}

\def\tupsigma{\tilde{\upsigma}}

\def\dist{\mathrm{d}}

\def\tw{\tilde{w}}

\def\bfUpsilon{\mathbf{\Upsilon}}
\def\bfW{\mathbf{W}}

\DeclareMathAlphabet{\mathpzc}{OT1}{pzc}{m}{it}

\def\wkN{w^{k,N}}

\def\WkN{W^{k,N}}

\def\datap{\pi}

\def\nablaw{\nabla_{w}}

\def\nunn{\nu_n^{N}}

\def\nun{\nu^N}

\def\MRp{\mathscr{P}(\mathbb{R}^p)}

\def\MRpdeux{\mathscr{P}_2(\mathbb{R}^p)}

\def\bfw{\mathbf{W}}

\def\bfW{\mathbf{W}}

\def\bfwt{\tilde{\mathbf{W}}}
\def\bfwkN{\mathbf{W}^{k,N}}

\def\bfwkNt{\tilde{\mathbf{W}}^{k,N}}
\def\bfnu{\bm{\nu}}

\def\bfmu{\bm{\mu}}
\def\nubf{\bfnu}

\def\bfnun{\bm{\nu}^N}

\def\bfnutN{\tilde{\bfnu}^N}

\newcommandx\pisys[1][1=t]{\Pi_{#1}}
\newcommandx\ctun[1][1=T]{\Capprox_{#1,1}}

\newcommand{\pinv}{^{-1}}

\newcommand{\rref}[1]{\tup{\Cref{#1}}}

\newcommand{\gua}{\gamma_{\alpha, \beta}(N)}
\newcommand{\guaD}[1]{\gamma^{#1}_{\alpha, \beta}(N)}

\newcommandx{\expec}[2]{{\mathbb E}\left[#1 \middle \vert #2  \right]} 

\newcommand{\nn}{_{n+1}}

\def\dim{d}
\newcommand{\tb}{\tilde{b}}

\newcommand{\rme}{\mathrm{e}}

\newcommand{\Lip}{\mathtt{L}}

\newcommand{\Mtt}{\mathtt{M}}
\newcommand{\Ktt}{\mathtt{K}}

\newcommandx{\norm}[2][1=]{\ifthenelse{\equal{#1}{}}{\left\Vert #2 \right\Vert}{\left\Vert #2 \right\Vert^{#1}}}
\newcommandx{\normLigne}[2][1=]{\ifthenelse{\equal{#1}{}}{\Vert #2 \Vert}{\Vert #2\Vert^{#1}}}

\def\bfc{\mathbf{c}}

\def\bfB{\mathbf{B}}
\def\tbfB{\tilde{\mathbf{B}}}

\def\msa{\mathsf{A}}

\def\msk{\mathsf{K}}
\def\mss{\mathsf{S}}

\def\mse{\mathsf{E}}
\def\msf{\mathsf{F}}

\def\msg{\mathsf{G}}

\def\msu{\mathsf{U}}

\def\msx{\mathsf{X}}
\def\msz{\mathsf{Z}}
\def\msy{\mathsf{Y}}

\newcommand{\mcb}[1]{\mathcal{B}(#1)}

\def\mcz{\mathcal{Z}}
\def\mcy{\mathcal{Y}}
\def\mcx{\mathcal{X}}
\def\mce{\mathcal{E}}

\def\mcf{\mathcal{F}}

\def\mcg{\mathcal{G}}

\def\rset{\mathbb{R}}

\def\nset{\mathbb{N}}
\def\nsets{\mathbb{N}^{\star}}

\def\bR{\mathbb{R}}

\def\rmu{\mathrm{u}}

\def\rmd{\mathrm{d}}
\def\rmy{\mathrm{y}}

\def\rmS{\mathrm{S}}

\def\rme{\mathrm{e}}

\def\rmc{\mathrm{C}}

\def\rmC{\mathrm{C}}
\def\rmD{\mathrm{D}}

\def\rmcc{\rmc_{\mathrm{c}}}
\def\rmcb{\rmc_{\mathrm{b}}}

\newcommand{\dd}{\mathrm{d}}

\def\trace{\operatorname{Tr}}
\newcommandx{\functionspace}[2][1=+]{\mathbb{F}_{#1}(#2)}

\newcommandx{\VarDeux}[3][3=]{\operatorname{Var}^{#3}_{#1}\left\{#2 \right\}}

\newcommand{\LeftEqNo}{\let\veqno\@@leqno}

\newcommand{\floor}[1]{\left\lfloor #1 \right\rfloor}

\newcommand{\N}{\ensuremath{\mathbb{N}}}

\newcommand{\PE}{\mathbb{E}}
\newcommand{\PP}{\mathbb{P}}

\newcommand{\abs}[1]{\left\vert #1 \right\vert}
\newcommand{\absLigne}[1]{\vert #1 \vert}

\newcommandx{\Vnorm}[2][1=V]{\| #2 \|_{#1}}
\newcommandx{\VnormEq}[2][1=V]{\left\| #2 \right\|_{#1}}

\newcommand{\parenthese}[1]{\left(#1 \right)}
\newcommand{\parentheseLigne}[1]{(#1 )}
\newcommand{\parentheseDeux}[1]{\left[ #1 \right]}
\newcommand{\parentheseDeuxLigne}[1]{[ #1 ]}
\newcommand{\defEns}[1]{\left\lbrace #1 \right\rbrace }
\newcommand{\defEnsLigne}[1]{\lbrace #1 \rbrace }

\newcommandx\probaMarkovTilde[2][2=]
{\ifthenelse{\equal{#2}{}}{{\widetilde{\mathbb{P}}_{#1}}}{\widetilde{\mathbb{P}}_{#1}\left[ #2\right]}}

\newcommand{\expe}[1]{\PE \left[ #1 \right]}

\newcommand{\expeLigne}[1]{\PE [ #1 ]}

\newcommand{\plusinfty}{+\infty}

\def\ie{\textit{i.e.}}

\def\eqsp{\;}
\newcommand{\coint}[1]{\left[#1\right)}
\newcommand{\ocint}[1]{\left(#1\right]}

\newcommand{\ccint}[1]{\left[#1\right]}

\newcommandx{\weight}[2][2=n]{\omega_{#1,#2}^N}

\definecolor{cadmiumgreen}{rgb}{0.0, 0.42, 0.24}

\newcommandx\sequence[3][2=,3=]
{\ifthenelse{\equal{#3}{}}{\ensuremath{\{ #1_{#2}\}}}{\ensuremath{\{ #1_{#2}, \eqsp #2 \in #3 \}}}}

\newcommandx\sequenceD[3][2=,3=]
{\ifthenelse{\equal{#3}{}}{\ensuremath{\{ #1_{#2}\}}}{\ensuremath{( #1)_{ #2 \in #3} }}}

\newcommandx{\sequencen}[2][2=n\in\N]{\ensuremath{\{ #1_n, \eqsp #2 \}}}
\newcommandx\sequenceDouble[4][3=,4=]
{\ifthenelse{\equal{#3}{}}{\ensuremath{\{ (#1_{#3},#2_{#3}) \}}}{\ensuremath{\{  (#1_{#3},#2_{#3}), \eqsp #3 \in #4 \}}}}
\newcommandx{\sequencenDouble}[3][3=n\in\N]{\ensuremath{\{ (#1_{n},#2_{n}), \eqsp #3 \}}}

\def\iid{i.i.d.}

\def\eg{\textit{e.g.}}

\newcommand{\opnorm}[1]{{\left\vert\kern-0.25ex\left\vert\kern-0.25ex\left\vert #1
    \right\vert\kern-0.25ex\right\vert\kern-0.25ex\right\vert}}

\def\Id{\operatorname{Id}}

\newcommandx{\CPE}[3][1=]{{\mathbb E}_{#1}\left[#2 \middle \vert #3  \right]} 
\newcommandx{\CPELigne}[3][1=]{{\mathbb E}_{#1}[#2 \vert #3 ]} 
\newcommandx{\CPEsq}[3][1=]{{\mathbb{E}^{1/2}}_{#1}\left[#2 \middle \vert #3  \right]} 
\newcommandx{\CPVar}[3][1=]{\mathrm{Var}^{#3}_{#1}\left\{ #2 \right\}}
\newcommand{\CPP}[3][]
{\ifthenelse{\equal{#1}{}}{{\mathbb P}\left(\left. #2 \, \right| #3 \right)}{{\mathbb P}_{#1}\left(\left. #2 \, \right | #3 \right)}}

\def\scrC{\mathscr{C}}

\newcommandx{\osc}[2][1=]{\mathrm{osc}_{#1}(#2)}
\newcommandx\dinfwass[1][1=T]{\wassersteinD[2, #1]}
\newcommandx\Mwiener[3][1=\ccint{0,T},2=p,3=]{\mathscr{P}_{#3}(\rmc(#1, \rset^{#2}))}

\def\Id{\operatorname{Id}}

\def\transpose{^\top}

\def\Capprox{C}

\newcommand{\ensembleLigne}[2]{\{#1\,:\eqsp #2\}}

\def\rmD{\mathrm{D}}

\newcommand\coupling[2]{\Gamma(\mu,\nu)}

\def\Leb{\mathrm{Leb}}

\def\vareps{\varepsilon}

\def\Phibf{\mathbf{\Phi}}

\newcommandx{\KL}[2]{\text{KL}\left( #1 | #2 \right)}
\newcommandx{\KLLigne}[2]{\text{KL}( #1 | #2 )}

\def\gaStep
\def\QKer{Q}

\def\distance{\mathbf{d}}
\newcommandx{\wasserstein}[3][1=\distance,3=]{\mathpzc{W}_{#1}^{#3}\left(#2\right)}
\newcommandx{\wassersteinLigne}[3][1=\distance,3=]{\mathpzc{W}_{#1}^{#3}(#2)}
\newcommandx{\wassersteinD}[1][1=\distance]{\mathpzc{W}_{#1}}
\newcommandx{\wassersteinDLigne}[1][1=\distance]{\mathpzc{W}_{#1}}

\def\sigmaD{\sigma^2}

\newcommandx{\phibfs}[1][1=]{\pmb{\varphi}_{\sigmaD_{#1}}}

\newcommandx\sequenceg[3][2=,3=]
{\ifthenelse{\equal{#3}{}}{\ensuremath{( #1_{#2})}}{\ensuremath{( #1_{#2})_{ #2 \geq #3}}}}

\newcommandx{\distV}[1][1=\bfc]{\mathbf{W}_{#1}}
\newcommandx{\distVdeux}[1][1=W_2]{\mathbf{d}_{#1}}

\def\mtt{\mathtt{m}}

\newcommand{\tup}[1]{\textup{#1}}

\def\Pens{\mathscr{P}}
\def\scrC{\mathscr{C}}

\def\bfrho{\boldsymbol{\rho}}
\def\bflambda{\boldsymbol{\lambda}}

\def\risk{\mathscr{R}}
\def\hrisk{\hat{\risk}}
\def\tbfW{\tilde{\mathbf{W}}}
\def\dist{\mathrm{m}}
\def\nbparticlem{m}

\usepackage{caption}

\title{Quantitative Propagation of Chaos for SGD in Wide Neural Networks}

\author{%
  Valentin De Bortoli \\
  Centre Borelli \\  
  ENS Paris Saclay
  \And
  Alain Durmus \\
  Centre Borelli \\  
  ENS Paris Saclay  
  \And
  Xavier Fontaine \\
  Centre Borelli \\  
  ENS Paris Saclay  
  \And  
  Umut \c{S}im\c{s}ekli \\ 
  LTCI, T\'{e}l\'{e}com Paris \\
  Institut Polytechnique de Paris
}

\begin{document}

\maketitle

\begin{abstract}
  In this paper, we investigate the limiting behavior of a
  continuous-time counterpart of the Stochastic Gradient Descent (SGD)
  algorithm applied to two-layer overparameterized neural networks, as
  the number or neurons (\ie, the size of the hidden layer)
  $N \to \plusinfty$.  Following a probabilistic approach, we show
  `propagation of chaos' for the particle system defined by this
  continuous-time dynamics under different scenarios, indicating that
  the statistical interaction between the particles asymptotically
  vanishes. In particular, we establish quantitative convergence with
  respect to $N$ of any particle to a solution of a mean-field
  McKean-Vlasov equation in the metric space endowed with the
  Wasserstein distance. In comparison to previous works on the
  subject, we consider settings in which the sequence of stepsizes in
  SGD can potentially depend on the number of neurons and the
  iterations. We then identify two regimes under which different
  mean-field limits are obtained, one of them corresponding to an
  implicitly regularized version of the minimization problem at
  hand. We perform various experiments on real datasets to validate
  our theoretical results, assessing the existence of these two
  regimes on classification problems and illustrating our convergence
  results.
\end{abstract}


\section{Introduction}

Due to their ability to tackle very challenging problems, neural
networks have been extremely popular and keystones in machine learning
\cite{goodfellow:bengio:courville:2016}. Thanks to their practical success,
they have become the de facto tool in many application domains, such as
image processing \cite{krizhevsky2012imagenet} and natural language
processing \cite{manning1999foundations}.  However, the mathematical
understanding of these models and their inherent inference mechanism still remains limited.

Among others, one suprising empirical observation about modern neural networks is that increasing the number of neurons in
a network often leads to better classification testing and training
errors \cite{zhang2016understanding}, contradicting the classical statistical learning theory \cite{shalev2014understanding}. These experimental results
suggest that neural network-based methods exhibit a limiting behavior
when the number of neurons is large, \ie, when the neural network is
\emph{overparameterized}.

In this paper, we contribute to the recent literature on the
theoretical analysis of this phenomenon. To this end, we consider a
simple two-layer (\ie, one hidden layer) neural network that is
parametrized by $N$ weights $w^{1:N} = \{w^{k,N}\}_{k=1}^N$ and
trained to minimize the structural risk
$\risk^{N}$ 
by Stochastic Gradient Descent (SGD) using independent and identically
distributed (\iid) samples $(X_i,Y_i)_{i\in\nsets}$.  Even in such a
simplified setting, the landscape of $\risk^N$ is in many cases
arduous to be explored, since $\risk^N$ is non-convex and might
exhibit many local minima and saddle points
\cite{li:et:al:2018:visualizing,ballard:2017:energy}; hence making the
minimization of $\risk^N$ challenging. However, for large $N$, the
analysis of the landscape of $\risk^N$ turns out to be much simpler in
some situations. For instance \cite{soltanolkotabi2019theoretical} has
shown that local minima are global minima when the activation function
is quadratic as soon as $N$ is larger than twice the size of the
original dataset. More generally, relying on approximation or random
matrix theory, several works
(\eg,\cite{fukumizu2000local,bray2007statistics,pascanu2014saddle,pennington2017geometry,kawaguchi2016deep,freeman2016topology,venturi2018neural,bach2017breaking,choromanska:et:al:2015,venturi:bandeira:bruna:2019,kuditipudi:2019:explaining_landscape})
establish favorable properties for the landscape of $\risk^N$ as
$N \to \plusinfty$, such as absence of saddle points, poor local
minima or connected optima. In addition, minimization by SGD in this
setting has also proved to be efficient for some models
\cite{allen:zho:li:song:2019,sarao:manelli:biroli:et:al:2019}.

In this paper we follow an increasingly popular line of research to
analyze the behavior of gradient descent-type algorithms (stochastic
or deterministic) used for overparameterized models. This approach
consists in establishing a `continuous mean-field limit' for these
algorithms as $N \to \plusinfty$, and has been successively applied in
\cite{sirignano2018mean,mei2018mean,rotskoff2018trainability,mei2019mean,javanmard2019analysis,
  chizat2019sparse,chizat2018global,jabir2019mean}. Based on this
result, the \emph{qualitative} long-time behavior of SGD applied to
overparameterized neural networks can be deduced: these studies all
identify an evolution equation on the limiting probability measure
which corresponds to a mean-field ordinary differential equation
(ODE), \ie, if the initialization is deterministic, then each hidden
unit of the network \emph{independently} evolves along the flow of a
specific ODE. This implies that, even though the update step is
intrinsically stochastic in SGD, the noise completely vanishes in the
limit $N\to \plusinfty$. In this context, two main strategies have
been followed to prove convergence of SGD to this mean-field
dynamics. The first one is based on gradient flows in Wasserstein
spaces
\cite{ambrosio2008gradient,erbar2010heat,ambrosio:savare:zambotti:2009}
and the second one is the `propagation of chaos' phenomenon
\cite{sznitman1991topics,gottlieb2000markov,meleard1988systemes},
indicating that the statistical interaction between the individual
entries of the network asymptotically vanishes. Both approaches are in
fact deeply connected, which stems from the duality between
probability and partial differential equation theories
\cite{jordan1998variational}. We follow in this paper the second
approach and establish that propagation of chaos holds for a
continuous counterpart of SGD to a solution of a McKean-Vlasov type
diffusion \cite{mckean:1967} as $N \to \plusinfty$.

The fact that no noise appears in the mean-field limit of SGD obtained
in previous work can seem surprising. Aiming to demystify this matter, we study in this paper the case where the stepsize in SGD can depend
on the number of neurons. Our main contribution is to identify two
mean-field regimes: The first one is the same as the deterministic
mean-field limit obtained in the described literature. The second one
is a McKean-Vlasov diffusion for which the covariance matrix is
non-zero and depends on the properties of the data distribution. To
the best of our knowledge, this limiting diffusion has not been
reported in the literature and brings interesting insights on the
behavior of neural networks in overparameterized settings. Our results
suggest that taking large stepsizes in the stochastic optimization
procedure corresponds to an \emph{implicit regularization} of the
original problem, which can potentially ease the minimization of the
structural risk.  In addition, in contrast to previous studies, we
establish strong quantitative propagation of chaos and we identify the
convergence rate of each neuron to its mean-field limit with respect
to $N$. Finally we numerically illustrate the existence of these two regimes and the
propagation of chaos phenomenon we derive on several classical
classification examples on MNIST and CIFAR-10 datasets.
In these experiments, the stochastic regime empirically exhibits slightly better generalization properties compared to the deterministic case
  identified in
  \cite{sirignano2018mean,mei2018mean,chizat2018global}.



\section{Overparametrized Neural Networks}
\label{sec:setting}
Consider some feature and label spaces denoted by $\msx$ and $\msy$
endowed with $\sigma$-fields $\mcx$ and $\mcy$ respectively.  In this
paper, we consider a one hidden layer neural network, whose purpose is to classify data
from $\msx$ with labels in $\msy$.  We suppose that the network has
$N \in \nsets$ neurons in the hidden layer whose weights are denoted
by $w^{1:N} = \{\wkN\}_{k=1}^N\in{ (\bR^p)}^N$.  We model the
non-linearity by a function $F : \bR^p \times \msx \to \rset$, and
consider a loss function $\ell:\bR\times \msy \to \bR_+$ and a penalty function
$V: \ \rset^p \to \rset$.  Then, the learning problem corresponding to
this space of hypothesis consists in minimizing the structural risk
\begin{equation}
  \label{eq:exp_loss}
  \textstyle{\risk^N(w^{1:N})= \int_{\msx \times \msy} \ell\parenthese{\dfrac{1}{N}\sum_{k=1}^N F(\wkN,x), y} \dd \datap(x,y) +\dfrac{1}{N} \sum_{k=1}^N V(w^{k,N}) \eqsp ,}
\end{equation}
where $\datap$ is the data distribution on $\msx \times \msy$.  Note
that, in this particular setting, the weights of the second layer are
fixed to $(1/N)$.  This setting is referred to as ``fixed
coefficients'' in \cite[Theorem 1]{mei2019mean} and is less realistic
than the fully-trainable setting.  Nevertheless, we believe that this
shortcoming can be circumvented upon replacing $F(w^{k,N},\cdot)$ by
$F(u^{k,N},\cdot)v^{k,N}$ in \eqref{eq:exp_loss}, where $u^{1:N}$ and
$v^{1:N}$ are the weights of the hidden and the second layer
respectively. However, this raises new theoretical challenges which
are left for future work.


Throughout this paper, we consider the following assumptions.
\begin{assumption}
	\label{assum:all}
	There exist measurable functions $\Phi : \msx \to \coint{1,\plusinfty}$ and $\Psi :\msy \to \coint{1,\plusinfty}$ such that the following conditions hold.
	\begin{enumerate}[wide, labelwidth=!, labelindent=0pt,label=(\alph*)]
        \item \label{item:assuml} $\ell: \ \bR\times\bR \to \bR_+$
          is such that for any $y \in \msy$,
          $(\tilde{\rmy} \mapsto \ell(\tilde{\rmy}, y))$ is three-times
          differentiable and for any
        $\mathrm{y} \in \rset$ and $y \in \msy$ we have
        \begin{equation}
\abs{\partial_1 \ell(0,y) } \leq \Psi(y) \eqsp , \qquad \abs{\partial_1^2 \ell(\rmy,y)} + \abs{\partial_1^3 \ell(\rmy,y)} \leq \Psi(y) \eqsp,
        \end{equation}
where for any $i \in \{1,2,3\}$, $\partial_1^i \ell(\rmy,y)$ is the $i$-th derivative of $(\tilde{\rmy} \mapsto \ell(\tilde{\rmy}, y))$ at $\rmy$.
\item \label{item:assumf} $F: \ \rset^p \times \msx \to \rset$
  is such that for any $x \in \msx$, $(\tw \mapsto F(\tw, x))$ is
  three-times differentiable and for any
        $w \in \rset^p$ and $x \in \msx$
        \begin{equation}
          \norm{F(w, x)} + \norm{\rmD_w^1 F(w, x)}+
          \norm{\rmD_w^2 F(w, x)} + \norm{\rmD_w^3 F(w, x)} \leq \Phi(x)\eqsp,
        \end{equation}
        where for any $i \in \{1, 2, 3\}$, $\rmD_w^i F(w, x)$ is
        the $i$-th differential of
        $(\tilde{w} \mapsto F(\tilde{w}, x))$ at $w$.
      \item \label{item:assumV} $V \in \rmC^3(\rset^p , \rset)$ satisfies
        $ \sup_{w \in \rset^p} \{ \normLigne{\rmD^2V(w)} +
        \normLigne{\rmD^3V(w)}\}< \plusinfty$.
  	\item \label{item:compact} The data distribution $\datap$ satisfies
  $
		\int_{\msx\times\msy} \defEnsLigne{\Phi^{10}(x)+\Psi^4(y)} \dd\datap(x,y) < \infty\eqsp .
  $
	\end{enumerate}
\end{assumption}

%
Note that \Cref{assum:all}-\ref{item:compact} is immediately satisfied
in the case where $\datap$ is compactly supported, $\msx$ and $\msy$
are subsets of $\rset^d$ and $\rset$ respectively and $\Psi$ and $\Phi$
are bounded on the support of $\datap$. For any $N \in \nsets$, under
\Cref{assum:all}, by the Lebesgue dominated convergence theorem,
$\risk^N$ given by \eqref{eq:exp_loss} is well-defined, continuously
differentiable with gradient given for any $w^{1:N} \in (\rset^{p})^N$ by
\begin{equation}
  \label{eq:2}
  \begin{aligned}
 \nabla \risk^N(w^{1:N})&=
\int_{\msx\times \msy} \nabla_w \hrisk^N(w^{1:N}, x, y) \rmd \pi(x,y) \eqsp, \\
  \hrisk^N(w^{1:N}, x, y) &= \textstyle{ \ell\parenthese{\frac{1}{N}\sum_{k=1}^N F(\wkN,x), y}
  + \frac{1}{N} \sum_{k=1}^N V(\wkN)}  \eqsp,\\
N \nabla_w \hrisk^N(w^{1:N}, x, y)        &=\textstyle{\partial_1 \ell
    \parenthese{\frac{1}{N} \sum_{k=1}^N F(w^{k,N},x) ,y}
    \nablaw F^{1:N}(w^{1:N},x)}  + \nabla V^{1:N}(w^{1:N})
\eqsp,
\end{aligned}
\label{eq:h_risk_def}
\end{equation}
setting $\nablaw F^{1:N}(w^{1:N},x) =  \{\nablaw F(w^{k,N},x)\}_{k=1}^N$,  and   $\nabla V^{1:N}(w^{1:N}) =\{\nabla V(w^{k,N})\}_{k=1}^N$.



Let $(W_0^k)_{ k\in\nsets}$ be \iid~$p$ dimensional random variables with distribution $\mu_0$.
Consider the sequence $(W_n^{1:N})_{n \in \nset}$ associated with SGD, starting from $W_0^{1:N}$ and defined by the following recursion:
 for any $n \in \nset$ denoting the iteration index
\begin{equation}
	\label{eq:sgdtotal}
	W^{1:N}\nn=W^{1:N}_n-\gamma N^{\beta}(n+\gua^{-1})^{-\alpha} \nabla \hrisk^N(W^{1:N}_n, X_n, Y_n) \eqsp ,
\end{equation}
where $(X_n, Y_n)_{n \in \nset}$ is a sequence of \iid~input/label
samples distributed according to $\pi$, and
$(\gamma N^{\beta}(n+\gua^{-1})^{-\alpha})_{n \in \nset}$ as a whole
denotes a sequence of stepsizes: here, $\beta \in \ccint{0,1}$,
$\alpha \in \coint{0,1}$, and
$\gua = \gamma^{1/(1-\alpha)}N^{(\beta-1)/(1-\alpha)}$.  Note that in
the constant stepsize setting $\alpha =0$, the recursion
\eqref{eq:sgdtotal} consists in using $\gamma N^{\beta}$ as a stepsize
In addition, it also encompasses the case of decreasing stepsizes (as
soon as $\alpha>0$). The term $\gua^{-1}$ in \eqref{eq:sgdtotal} is a
scaling parameter which appears naturally in the corresponding
continuous-time dynamics, see \eqref{eq:sde} below. We stress that
contrary to previous approaches such as
\cite{chizat2018global,mei2018mean,sirignano2018mean}, the stepsize
appearing in \eqref{eq:sgdtotal} depends on the number of neurons
$N$. Our main contribution is to establish that different mean-field
limiting behaviors of a continuous counterpart of SGD arise depending
on $\beta$.

We will show that the quantity $\gua$ plays the role of a
discretization stepsize in the McKean-Vlasov approximation of SGD. The
case where $\alpha=0$ and $\beta = 0$, \ie, the setting considered by
\cite{chizat2018global,mei2018mean,sirignano2018mean}, corresponds to
choosing the stepsize as $\gamma/N$, which decreases with increasing
$N$. In the new setting $\alpha=0$, $\beta =1$, this corresponds to
take a fixed stepsize $\gamma$. This observation further motivates the scaling
and the parameter we introduced in \eqref{eq:sgdtotal}.

Before stating our result, we present and give an informal derivation of the continuous particle system
dynamics we consider to model \eqref{eq:sgdtotal}.
We first show that \eqref{eq:sgdtotal} can be rewritten as a recursion
corresponding to the discretization of a continuous particle system,
\ie, a stochastic differential equation (SDE) with coefficients
depending on the empirical measure of the particles.  Let us denote by
$\Pens(\mse)$ the set of probability measures on a measurable space
$(\mse,\mce)$.  Remark that for each particle dynamics
  $(W_n^{k,N})_{n \in \nset}$ the SGD update \eqref{eq:sgdtotal} is a
  function of the current position and the empirical measure of the
  weights. To show this, define the mean-field
  $h : \rset^p \times \Pens(\rset^d)\to \rset^p$ and the noise field 
  $\xi : \rset^p \times \Pens(\rset^d) \times \msx\times \msy \to
  \rset^p$, for any $\mu \in \MRp$, $w \in \rset^p$,
$(x,y) \in \msx \times \msy$ by 
\begin{align}
  \label{eq:h}
&  h(w, \mu)=-\int_{\msx \times \msy}  \partial_1
\ell \parenthese{\mu[F(\cdot,x)] ,y} \nablaw F(w,x) \, \dd \datap(x,y) - \nabla V(w) \eqsp, \\
  \label{eq:def_xi_n}
 &  \xi(w,\mu,x,y)=-h(w, \mu) - \partial_1 \ell(\mu[F(\cdot,x)] ,y) \nablaw F(w,x)  - \nabla V(w)\eqsp .
\end{align}
Note that with this notation,
$h(\wkN, \nunn) = -N \partial_{\wkN} \risk^N(w^{1:N})$ and
$ \xi(\wkN, \nunn,X_n,Y_n) = N \defEnsLigne{-\partial_{\wkN}
  \hrisk^N(w^{1:N}, X_n, Y_n) + \partial_{\wkN} \risk^N(w^{1:N})}$,
for any $N \in \nset$, $k \in \{1, \dots, N\}$ and $n \in \nset$,
where $\nun$ is the empirical measure of the discrete particle system
corresponding to SGD defined by
$\nunn=N^{-1}\sum_{k=1}^N \updelta_{\WkN_n}$. Then,
the recursion \eqref{eq:sgdtotal} can be rewritten as follows:
\begin{equation}
	\label{eq:sgd}
	\WkN\nn = \WkN_n + \gamma N^{\beta -1}(n+\gua\pinv)^{-\alpha} \defEns{h(\WkN_n,\nunn) + \xi(\WkN_n,\nunn,X_n,Y_n)} \eqsp .
\end{equation}
We now present the continuous model associated with this discrete
process. For large $N$ or small $\gamma$ these two processes
  can be arbitrarily close. For $N \in\nsets$, consider the particle system diffusion
$(\bfW^{1:N}_t)_{t \geq 0} = (\{\bfW^{k,N}_t\}_{k=1}^N)_{t \geq 0}$
starting from $\bfW^{1:N}_0=W_0^{1:N}$ defined for any $k \in \{1, \dots, N\}$
by
\begin{equation}
	\label{eq:sde}
	\dd \bfwkN_t = (t+1)^{-\alpha} \defEns{ h(\bfwkN_t,\bfnu^N_t)\dd t+\gua^{1/2}\Sigma^{1/2}(\bfwkN_t,\bfnu^N_t) \dd \bfB^{k}_t} \eqsp ,
\end{equation}
where
$\ensembleLigne{(\bfB_t^{k})_{t \geq 0}}{k \in \nsets}$ is a family of independent $p$-dimensional Brownian
motions and $\bfnun_t$ is the empirical probability distribution of the particles defined for any $t \geq 0$ by  $\bfnun_t=N^{-1}\sum_{k=1}^N \updelta_{\bfwkN_t}$.
In addition in \eqref{eq:sde}, $	\Sigma$ is the $p\times p$ matrix given by
\begin{equation}
	\label{eq:Sigma}
\textstyle{	\Sigma(w,\mu)=\int_{\msx\times\msy} \xi(w, \mu,x,y)\xi(w,\mu,x,y)\transpose  \dd \datap(x,y) \eqsp,} \qquad \text{ for any $w \in \rset^p$ and $\mu \in \MRp$} \eqsp,
      \end{equation}
which is well-defined under \Cref{assum:all}. In the supplementary material we show that under \Cref{assum:all}, \eqref{eq:sde}
admits a  unique strong  solution. We now give an informal discussion to justify why
\eqref{eq:sde} can be seen as the continuous-time counterpart of
\eqref{eq:sgd}. For any $N \in \nsets$,
define $(\tbfW^{1:N}_t)_{t \geq 0}$  for any $t \geq 0$ by
$\tbfW^{1:N}_t = W_{n_t}^{1:N}$ with $n_t = \floor{t/\gua}$ and denote
$\bfnutN_t$ the empirical measure associated with
$\tbfW^{1:N}_t$. In this case, by defining the interval $I_{n, \alpha, \beta}^N = \ccint{n\gua, (n+1)\gua}$ and using \eqref{eq:sgd} and
$\gua^{1-\alpha}=\gamma N^{\beta-1}$, we obtain the following approximation for
any $n \in \nset$
\begin{align}
  &\bfwkNt_{(n+1)\gua}-\bfwkNt_{n\gua}\\
  &=\gamma N^{\beta -1} (n+\gua\pinv)^{-\alpha} \defEns{ h(\bfwkNt_{n\gua},\nunn) + \xi(\bfwkNt_{n\gua},\nunn,X_n,Y_n)} \\
  &\approx \gua (n\gua+1)^{-\alpha} \defEns{h(\bfwkNt_{n\gua},\bfnutN_{n\gua}) + \Sigma^{1/2}(\bfwkNt_{n\gua},\bfnutN_{n\gua})G} \\
  &\approx \underbrace{\int_{I_{n, \alpha, \beta}^N} (s+1)^{-\alpha} h(\bfwt^{k,N}_s,\bfnutN_s)\dd s}_{(A)} + \underbrace{\int_{I_{n, \alpha, \beta}^N} \guaD{1/2} (s+1)^{-\alpha} \Sigma^{1/2}(\bfwt^{k,N}_s,\bfnutN_s) \dd \bfB_s^{k}}_{(B)} \eqsp, 	\label{eq:ansatz}
\end{align}
where $G$ is a $p$-dimensional Gaussian random variable with zero mean
and identity covariance matrix.  Note that the second line
  corresponds to \eqref{eq:sgd} and the last to \eqref{eq:sde}.  To
obtain such proxy, we first remark that for any
  $w \in \rset^p$ and $\mu \in \Pens(\rset^p)$, $\xi(w,\mu,X_n,Y_n)$
  has zero mean and covariance matrix $\Sigma(w, \mu)$ and assume
that the noise term is roughly Gaussian.  Second, we use that the
covariance of $(B)$ in \eqref{eq:ansatz} is equal to
$\int_{I_{n, \alpha, \beta}^N} \gua (s+1)^{-2\alpha}
\Sigma(\bfwt^{k,N}_s,\bfnutN_s) \rmd s$. To obtain the last line, we
use some first-order Taylor expansion of this term and $(A)$ as
$\gua \to 0$. Then, \eqref{eq:ansatz} corresponds to \eqref{eq:sde} on
$I_{n,\alpha,\beta}^N$. As a result, \eqref{eq:sde} is the continuous
counterpart to \eqref{eq:sgd} and $n$ iterations in \eqref{eq:sgd}
correspond to the horizon time $n \gua$ in \eqref{eq:sde}.  In the
next section, we show that a strong quantitative propagation of chaos holds
for \eqref{eq:sde} \ie, we show that for $N \to +\infty$ the
  particles become indenpendent and have the same distribution
  associated with a McKean-Vlasov diffusion. The extension of these
results to discrete SGD \eqref{eq:sgd} and the rigorous derivation of
\eqref{eq:ansatz} can be established using strong functional
approximations following \cite[Proposition
1]{fontaine2020continuous}. Due to space constraints, we
leave it as future work.

Finally, note that until now we only considered the case where the batch size in SGD is equal to one. For a batch size $M \in\nsets$, this limitation can be lifted replacing $\pi$ and  $\hrisk^N$ in \eqref{eq:sgdtotal}
by $\pi^{\otimes M}$ and
\begin{equation}
\textstyle{  \hrisk^{N, M}(w^{1:N}, x, y) = \dfrac{1}{M}\sum_{i=1}^M \ell\parenthese{\dfrac{1}{N}\sum_{k=1}^N F(\wkN,x_i), y_i} \eqsp,}
\end{equation}
defined for any $w^{1:N} \in (\rset^p)^N$,
$x = (x_i)_{i \in \{1, \dots, M\}} \in \msx^M$ and
$y = (y_i)_{i \in \{1, \dots, M\}} \in \msy^M$.
In this case, we obtain that the continuous-time counterpart of
\eqref{eq:sgd} is given by \eqref{eq:sde} upon replacing $\Sigma^{1/2}$ by $\Sigma^{1/2}/M^{1/2}$.
This leads to the particle system diffusion $(\bfW^{1:N}_t)_{t \geq 0} = (\{\bfW^{k,N}_t\}_{k=1}^N)_{t \geq 0}$  starting from $\bfW^{1:N}_0$ defined for any
$k \in \{1, \dots, N\}$ by
\begin{equation}
	\label{eq:sde_batch}
	\dd \bfwkN_t = (t+1)^{-\alpha} \defEns{ h(\bfwkN_t,\bfnu^N_t)\dd t+(\gua/M)^{1/2}\Sigma^{1/2}(\bfwkN_t,\bfnu^N_t) \dd \bfB^k_t} \eqsp .
      \end{equation}
      In the supplement \Cref{sec:stoch-grad-lang}, we also present
      the case of a modified Stochastic Gradient Langevin Dynamics
      (mSGLD) algorithm \cite{welling2011bayesian} which was
      considered in \cite{mei2018mean} in the specific case
        $\beta = 0$. We extend our propagation of chaos results to
        this setting.



\section{Mean-Field Approximation and Propagation of Chaos}
\label{sec:mean-field-appr}

In this section we identify the mean-field limit of the diffusion
\eqref{eq:sde_batch}. More precisely, we show that there exist two regimes
depending on how the stepsize scale with the number of hidden units.

Our results are based on the propagation of chaos theory
\cite{sznitman1991topics,meleard1988systemes,gottlieb2000markov} and
extend the recent works of
\cite{chizat2019sparse,chizat2018global,sirignano2020mean,sirignano2018mean,mei2019mean,mei2018mean}.
In what follows, we denote
$\Pens_2(\rset^p) = \{ \mu \in \Pens(\rset^p) \, : \, \int_{\rset^p}
\norm[2]{\tilde{w}} \rmd \mu(\tilde{w}) < \plusinfty\}$ and 
$\rmc(\rset_+, \rset^p)$ the set of continuous functions from
$\rset_+$ to $\rset^p$. We also consider the usual metric $\dist$ on
$\rmc(\rset_+, \rset^p)$ defined for any
$u_1, u_2 \in \rmc(\rset_+, \rset^p)$ by
$\dist(u_1, u_2) = \sum_{n \in \nsets} 2^{-n}
\norm{u_1-u_2}_{\infty,n} / \defEnsLigne{1 + \norm{u_1-u_2}_{\infty,n}
}$, where
$\norm{u_1-u_2}_{\infty,n} = \sup_{t \in \ccint{0,n}}
\normLigne{u_1(t) - u_2(t)}$.  It is well-known that
$(\scrC, \dist) = (\rmc(\rset_+, \rset^p), \dist)$ is a complete
separable space.  For any metric space $(\msf, \dist_{\msf})$, with
Borel $\sigma$-field $\mcb{\msf}$, we define the extended Wasserstein
distance of order $2$, denoted $\wassersteinD[2]: \ \Pens(\msf) \times \Pens(\msf) \to
  \ccint{0,+\infty}$ for any $\mu_1, \mu_2 \in \Pens(\msf)$ by
$ \wassersteinD[2]^2(\mu_1, \mu_2) = \inf_{\Lambda \in \Gamma(\mu_1,
  \mu_2)}\int_{\msf \times \msf} \dist_{\msf}^2(v_1,v_2) \rmd
\Lambda(v_1,v_2)$, where $\Gamma(\mu_1, \mu_2)$ is the set of
transference plans between $\mu_1$ and $\mu_2$, \ie,
$\Lambda \in \Gamma(\mu_1, \mu_2)$ if for any $\msa \in \mcb{\msf}$,
$\Lambda(\msa \times \msf) = \mu_1(\msa)$ and
$\Lambda(\msf \times \msa) = \mu_2(\msa)$.

We start by stating our results in the case where $\beta \in \coint{0,1}$ for which a deterministic mean-field limit is obtained.
Consider the mean-field ODE starting from a random variable $\bfw_0^{\star}$ given by
  \begin{equation}
    \label{eq:mean_field_beta_small}
    \rmd \bfw_t^{\star} = (t+1)^{-\alpha} h(\bfw_t^{ \star}, \bflambda_t^{ \star}) \rmd t \eqsp , \qquad \text{with $\bflambda^{\star}_t$ the distribution of $\bfw_t^{ \star}$} \eqsp.
  \end{equation}
  We show in the supplement that this ODE admits a solution on $\rset_+$.  This mean-field equation
  \eqref{eq:mean_field_beta_small} is deterministic conditionally to
  its initialization.


\begin{theorem}
  \label{thm:empi_conv_cont}
  Assume \rref{assum:all}. Let $(\bfw_0^k)_{k \in \nset}$ be a
  sequence of \iid~$\rset^p$-valued random variables with distribution
  $\mu_0 \in \Pens_2(\rset^p)$ and set for any $N \in \nsets$,
  $\bfw_0^{1:N} = \{\bfw_0^{k}\}_{k=1}^N$. Then, for any
  $\nbparticlem \in \nsets$ and $T \geq 0$, there exists $C_{\nbparticlem,T} \geq 0$
  such that for any $\alpha \in \coint{0,1}$, $\beta \in \coint{0,1}$,
  $M \in \nsets$ and $N \in \nsets$ with $N \geq m$
  \begin{equation}
   \textstyle{ \expe{\sup_{t \in \ccint{0,T}} \normLigne{\bfw_t^{1:\nbparticlem,N} - \bfw_t^{1:\nbparticlem, \star}}^2} \leq C _{\nbparticlem, T} \defEns{N^{-(1-\beta)/(1-\alpha)}M^{-1} + N^{-1}} \eqsp ,}
  \end{equation}
  with
  $(\bfw_t^{1:\nbparticlem,N}, \bfw_t^{1:\nbparticlem, \star}) =
  \{(\bfw_t^{k, N}, \bfw_t^{k, \star})\}_{k=1}^{\nbparticlem}$,
  $(\bfw_t^{1:N})$ the solution of \eqref{eq:sde_batch} starting
  from $\bfw_0^{1:N}$, and for any $k \in \nsets$,
  $\bfw_t^{k, \star}$ the solution of
  \eqref{eq:mean_field_beta_small} starting from $\bfw_0^k$.
\end{theorem}

In \Cref{thm:empi_conv_cont}, $m$ is a fixed number of particles.
Note that
$\ensembleLigne{(\bfw_t^{k, \star})_{t \geq 0}}{k \in \nsets}$ is
\iid~with distribution $\bflambda^{\star}$ which is the pushfoward measure of
$\mu_0$ by the function $(w_0\mapsto (w_t)_{t \geq 0})$ which from an
initial point $w_0$ gives $(w_t)_{t \geq 0} \in \scrC$ the solution of
\eqref{eq:mean_field_beta_small} on
$\rset_+$. \Cref{thm:empi_conv_cont} shows that the dynamics of the 
particles become deterministic and independent when
$N \to \plusinfty$.    The proofs of \Cref{thm:empi_conv_cont} and the following result, \Cref{thm:empi_conv_cont_one}, are postponed to
\Cref{sec:proofs-crefthm:-cref}.

We now consider the case $\beta = 1$ and derive a similar quantitative
theorem as \Cref{thm:empi_conv_cont} but with a different dynamics
than \eqref{eq:mean_field_beta_small}. Consider the mean-field SDE
starting from variable $\bfw_0^{\star}$ given by
  \begin{equation}
    \label{eq:mean_field_beta_one}
    \rmd \bfw_t^{\star} = (t+1)^{-\alpha} \defEns{h(\bfw_t^{ \star}, \bflambda_t^{ \star}) \rmd t
    + (\gamma^{1/(1-\alpha)}\Sigma(\bfw_t^{ \star}, \bflambda_t^{\star})/M)^{1/2} \rmd \bfB_t} \eqsp ,
  \end{equation}
  where $\bflambda^{\star}_t$ is the distribution of $\bfw_t^{ \star}$
  and $(\bfB_t)_{t \geq 0}$ is a $p$ dimensional Brownian motion.
  Note that taking the limit $\gamma \to 0$ or $M \to +\infty$ in
  \eqref{eq:mean_field_beta_one} we recover
  \eqref{eq:mean_field_beta_small}. We show in the supplement that this SDE admits a solution on
  $\rset_+$.  The following theorem is similar to
  \Cref{thm:empi_conv_cont} in the case $\beta = 1$.

\begin{theorem}
  \label{thm:empi_conv_cont_one}
  Let $\beta = 1$.  Assume \rref{assum:all}. Let
  $(\bfw_0^k)_{k \in \nset}$ be a sequence of $\rset^p$-valued random
  variables with distribution $\mu_0 \in \Pens_2(\rset^p)$ and assume
  that for any $N \in \nsets$,
  $\bfw_0^{1:N} = \{\bfw_0^{k}\}_{k=1}^N$. Then, for any
  $\nbparticlem \in \nsets$ and $T \geq 0$, there exists $C_{\nbparticlem,T} \geq 0$
  such that for any $\alpha \in \coint{0,1}$, $M \in \nsets$ and
  $N \in \nsets$ with $N \geq m$ 
  \begin{equation}
   \textstyle{ \expe{\sup_{t \in \ccint{0,T}} \normLigne{\bfw_t^{1:\nbparticlem,N} - \bfw_t^{1:\nbparticlem, \star}}^2} \leq C _{\nbparticlem, T} N^{-1} \eqsp , }
 \end{equation}
 with
 $(\bfw_t^{1:\nbparticlem,N}, \bfw_t^{1:\nbparticlem, \star}) =
 \{(\bfw_t^{k, N}, \bfw_t^{k, \star})\}_{k=1}^{\nbparticlem}$,
 $(\bfw_t^{1:N})$ the solution of \eqref{eq:sde_batch} starting
 from $\bfw_0^{1:N}$, and for any $k \in \nsets$,
 $\bfw_t^{k, \star}$ the solution of \eqref{eq:mean_field_beta_one}
 starting from $\bfw_0^k$ and Brownian motion $(\bfB^{k}_{t})_{t \geq 0}$.
\end{theorem}


The main difference between \eqref{eq:mean_field_beta_small} and
\eqref{eq:mean_field_beta_one} is that now this mean-field limit is
now longer deterministic up to its initialization but is a SDE driven
by a Brownian motion. The stochastic nature of SGD is preserved in
this second regime. \eqref{eq:mean_field_beta_small} corresponds to
some implicit regularization of \eqref{eq:mean_field_beta_one}.  In
the case where for any $w \in \rset^p$ and $\mu \in \Pens(\rset^p)$,
$\Sigma(w, \mu) = \upsigma^2 \Id$ with $\upsigma >0$, it can shown
that $(\bflambda_t^{\star})_{t \geq 0}$ is a gradient flow for an
entropic-regularized functional. This relation between our approach
and the gradient flow perspective is investigated in the supplement
\Cref{sec:gradient_flows}.

Denote for any $N \in \nsets$ and $\nbparticlem \in \{1, \dots, N\}$,
$\bflambda^{1:\nbparticlem,N}$ the distribution on $\scrC$ of
$\{(\bfw_t^{k,N})_{t \geq 0}\}_{k=1}^{\nbparticlem}$.  Recall that
$\{\bfw_0^{k,N}\}_{k=1}^N$ are $N$ i.i.d. $\rset^p$-valued random
variables with distribution $\mu_0 \in \Pens_2(\rset^p)$.  As an
immediate consequence of \Cref{thm:empi_conv_cont},
\Cref{thm:empi_conv_cont_one} and the definition of $\wassersteinD[2]$
for the distance $\dist$ on $\scrC$, we have the following propagation
of chaos result.

\begin{corollary}
  \label{coro:cv_propa_chaos}
  Assume \rref{assum:all}. Then for any $\beta \in \ccint{0,1}$,
  $\alpha \in \coint{0,1}$, $M \in \nsets$ and
  $\nbparticlem \in \nset$ we have
  $ \lim_{N \to +\infty} \wassersteinD[2](\bflambda^{1:\nbparticlem,
    N}, (\bflambda^{\star})^{\otimes \nbparticlem}) = 0$ where
  $\bflambda^{\star}$ is the distribution of
  $(\bfw^{\star}_t)_{t \geq 0}$ solution of
  \eqref{eq:mean_field_beta_small} if $\beta \in \ocint{0,1}$ and
  \eqref{eq:mean_field_beta_one} if $\beta =1$ with $\bfw_0^{\star}$
  distributed according to $\mu_0$.
\end{corollary}

\Cref{coro:cv_propa_chaos} has two main
consequences: when the number of hidden units is large
\begin{enumerate*}[label=(\roman*)]
\item all the units have the same distribution $\bflambda^{\star}$, and
\item the units are independent.
\end{enumerate*}
Note also that this corollary is valid for the whole trajectory and
not only for a fixed time horizon.

Finally, we derive similar results to
  \Cref{coro:cv_propa_chaos} for the sequence of the empirical
  measures.  Let $(\bfnu^N)_{N \in \nsets}$ be the sequence of
empirical measures associated with \eqref{eq:sde_batch} and given by
$\bfnu^N = N^{-1} \sum_{k=1}^N \updelta_{(\bfw_t^{k,N})_{t \geq
    0}}$. Note that for any $N \in \nset$, $\bfnu^N$ is a random
probability measure on $\Pens(\scrC)$. Denote for any $N \in \nsets$,
$\bfUpsilon^N$ its distribution which then belongs to
$ \Pens(\Pens(\scrC))$.  
Since the convergence with respect to the $\wassersteinD[2]$ distance
implies the weak convergence, using \Cref{coro:cv_propa_chaos} and the
Tanaka-Sznitman theorem \cite[Proposition 2.2]{sznitman1991topics}, we
get that $(\bfUpsilon^N)_{N \in \nsets}$ weakly converges towards
$\updelta_{\bflambda^{\star}}$. In fact, we prove the following
stronger proposition whose proof is postponed to \Cref{sec:main_res}.

\begin{proposition}
  \label{prop:cv_w2_empi}
  Assume \rref{assum:all}. Then, for any $\beta \in \coint{0,1}$,
  $\alpha \in \coint{0,1}$ and $M \in \nsets$ we have
  $ \lim_{N \to +\infty} \wassersteinD[2](\bfUpsilon^N,
  \updelta_{\bflambda^{\star}}) = 0 $, where
  $\bflambda^{\star}$ is the distribution of
  $(\bfw^{\star}_t)_{t \geq 0}$ solution of
  \eqref{eq:mean_field_beta_small} if $\beta \in \ocint{0,1}$ and
  \eqref{eq:mean_field_beta_one} if $\beta =1$ with $\bfw_0^{\star}$
  distributed according to $\mu_0$.
\end{proposition}

    \begin{proof}[Proof of \Cref{prop:cv_w2_empi}]
  We consider only the case $\beta =1$, the proof for
  $\beta \in \coint{0,1}$  following the same lines.  Let
 $M \in
  \nsets$. We have for any $N \in \nsets$ using \Cref{prop:fm_bound},
  \begin{equation}
    \wassersteinD[2](\bfUpsilon^N, \updelta_{\bflambda^{\star}})^2 \leq \expeLigne{\wassersteinD[2](\bfnu^N, \bflambda^{\star})^2} 
    \leq N^{-1}\sum_{k=1}^N \expeLigne{\dist^2((\bfw_t^{k,N})_{t \geq 0}, (\bfw_t^{k,\star})_{t \geq 0})} \eqsp . \label{eq:majo_wass}                                                                     
  \end{equation}
  Let $\vareps >0$ and $n_0\in \nsets$ such that
  $\sum_{n =n_0+1}^{+\infty} 2^{-n} \leq \vareps$.  Combining
  \eqref{eq:majo_wass}, \Cref{thm:empi_conv_cont_one} and the
  Cauchy-Schwarz inequality we get that for any $N \in \nsets$
  \begin{equation}
    \wassersteinD[2](\bfUpsilon^N, \updelta_{\bflambda^{\star}})^2 \leq 2 \vareps^2 + \frac{2n_0}{N} \sum_{k=1}^N \sum_{n=1}^{n_0}\expe{\sup_{t \in \ccint{0,n}}
      \normLigne{\bfw_t^{k,N} - \bfw_t^{k, \star}}^2}
  \leq 2 \vareps^2 + 2 n_0N^{-1} \sum_{n=0}^{n_0} C_{1,n} \eqsp .
  \end{equation}
  Therefore, for any $\vareps > 0$ there exists $N_0 \in \nsets$ such
  that for any $N \geq N_0$,
  $\wassersteinD[2](\bfUpsilon^N, \updelta_{\bflambda^{\star}}) \leq
  \vareps$.
\end{proof}

\paragraph{Relation to existing results.}
To the authors knowledge, only the case $\beta=0$ has been considered
in the current literature. More precisely, \Cref{thm:empi_conv_cont}
is a functional and quantitative extension of the results established
in \cite{sirignano2018mean, mei2018mean, chizat2018global,
  rotskoff2018trainability, sirignano2020mean}. First, in
\cite[Theorem~1.6]{sirignano2018mean}, it is shown that
$(\bflambda^{1:\nbparticlem, N})_{N \in \nsets, N \geq \ell}$ weakly
converges towards $(\bflambda^{\star})^{\otimes
  \nbparticlem}$. \cite[Theorem 3]{mei2018mean} shows weak convergence
of SGD to \eqref{eq:mean_field_beta_small} with high probability in
the case $V=0$ and the quadratic loss
$\ell(\rmy_1,\rmy_2) = (\rmy_1-\rmy_2)^2$.
\cite[Theorem~1.5]{sirignano2020mean} establishes a central limit
theorem for $(\bfnu^N)_{N \in \nsets}$ with rate $N^{-1/2}$ which is
in accordance with the convergence rate identified in
\Cref{thm:empi_conv_cont}. Finally, \cite[Theorem
2.6]{chizat2018global} and \cite[Proposition
3.2]{rotskoff2018trainability} imply the convergence of $\bfnu^N$
almost surely under the setting $\Sigma = 0$ in \eqref{eq:sde_batch}
which corresponds to the continuous gradient flow dynamics associated
with $\hrisk^{N,M}$. We conclude this part by mentioning that similar
results are derived for mSGLD in the supplement
\Cref{sec:stoch-grad-lang} which extend the ones obtained in
\cite{mei2019mean}.

Having established the convergence of \eqref{eq:sde_batch} to
\eqref{eq:mean_field_beta_one}, we are interested in the long-time
behaviour of $(\bfw^{\star}_{t})_{t \geq 0}$ in the case
$\alpha=0$. To address this problem, the first step is to show that
this SDE admits at least one stationary distribution, \ie \ a
probability measure $\mu^{\star}$ such that if $\bfw_0^{\star}$ has
distribution $\mu^{\star}$, then for any $t \geq 0$, $\bfw_t^{\star}$
has distribution $\mu^{\star}$.  If $V$ is strongly convex, we are
able to answer positively to this question in the case $p=1$. The
proof of this result is postponed to \Cref{sec:exist-invar-meas}.

  \begin{proposition}
    \label{prop:existence_inv}
    Assume \rref{assum:all}, $\alpha=0$  and  $p=1$. In addition, assume that there
    exist $\eta, \bar{\upsigma} > 0$ such that for any $w \in \rset$
    and $\mu \in \Pens(\rset)$, $\Sigma(w, \mu) \geq \bar{\upsigma}^2$
    and $V$ is $\eta$-strongly convex. Let
    $H: \ \Pens_2(\rset) \to \Pens_2(\rset)$ defined for any
    $\mu \in \Pens_2(\rset)$ and $w \in \rset$ by
  \begin{equation}
\textstyle{    (\rmd H(\mu) / \rmd \Leb)(w) \propto \bar{\Sigma}^{-1}(w, \mu) \exp \parentheseDeux{-2 \int_0^{w} h(\tilde{w}, \mu) / \bar{\Sigma}(\tilde{w}, \mu) \rmd \tilde{w}}  \eqsp ,}
  \end{equation}
  where
  $\bar{\Sigma}(w, \mu) = \gamma^{1/(1-\alpha)} \Sigma(w, \mu) / M$ and
  $\Leb$ is the Lebesgue measure on $\rset$. Then
  $\mss = \ensembleLigne{\mu \in \Pens_2(\rset)}{H(\mu) = \mu} \neq
  \emptyset$ and for any $\mu \in \mss$, $\mu$ is invariant for
  \eqref{eq:mean_field_beta_one}.
  \end{proposition}
\section{Experiments}

We now empirically illustrate the results derived in the previous
section. More precisely, we focus on the classification task for two
datasets: MNIST \cite{mnist} and CIFAR-10
\cite{krizhevsky2009learning}. In all of our experiments we consider a
fully-connected neural network with one hidden layer and ReLU
activation function. We consider the cross-entropy loss in order to
train the neural network using SGD as described in
\Cref{sec:setting}. All along this section we fix a time horizon
$T \geq 0$ and sample $W_{n_T}^{1:N}$ defined by \eqref{eq:sgd} with
$n_T = \floor{T / \gua}$ for
$\gua = \gamma^{1/(1-\alpha)}N^{(\beta-1)/(1-\alpha)}$ and taking a batch of
size $M\in\nsets$. We aim at illustrating the results of
\Cref{sec:mean-field-appr} taking $N \to \plusinfty$ and different
sets of values for the parameters $\alpha,\beta,M,\gamma$ in
\eqref{eq:sde_batch}. Indeed, recall that as observed in
\eqref{eq:ansatz}, $W_{n_T}^{1:N}$ is an approximation of
$\bfW^{1:N}_{T}$.  See \Cref{sec:addit-exper} for a detailed
description of our experimental setting. If not specified, we set
$\alpha =0$, $M=100$, $T = 100$, $\gamma =1$.

\paragraph{Convergence of the empirical measure.}
\begin{figure}[H]
 \begin{minipage}[b]{.45\linewidth}
  \centering
  \includegraphics[width=1.1\linewidth]{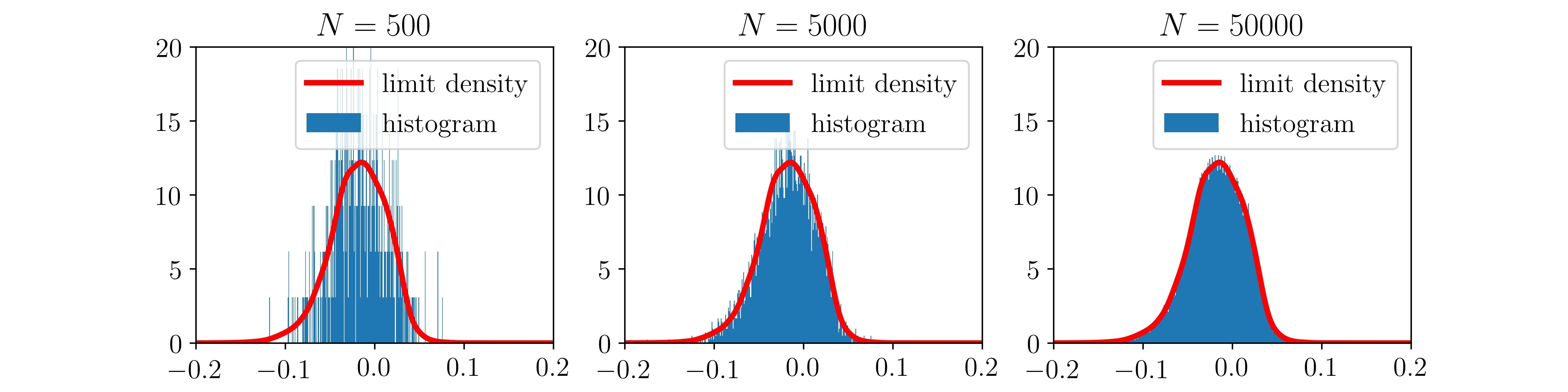} \\
  \includegraphics[width=1.1\linewidth]{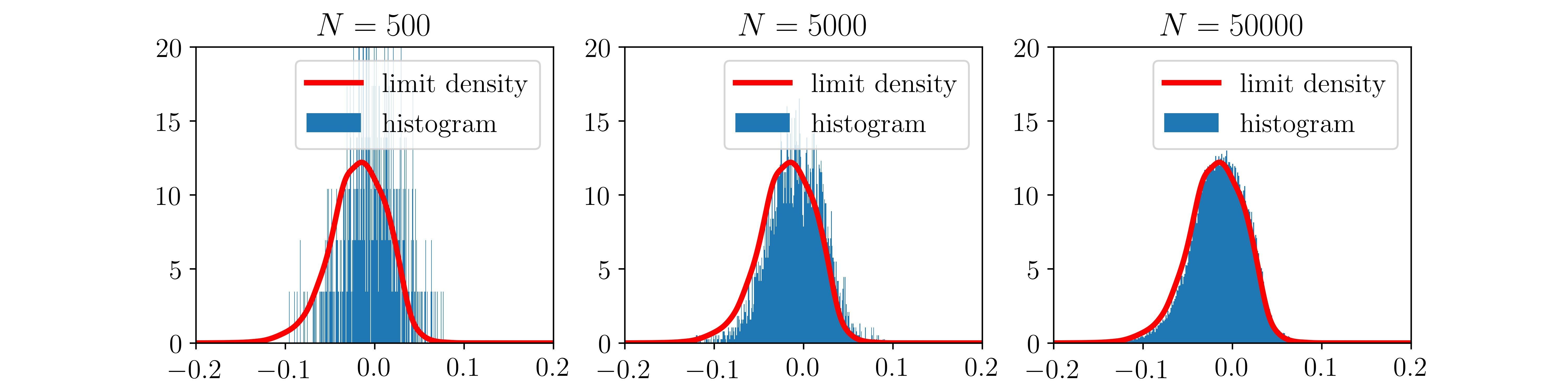} \\
  \includegraphics[width=1.1\linewidth]{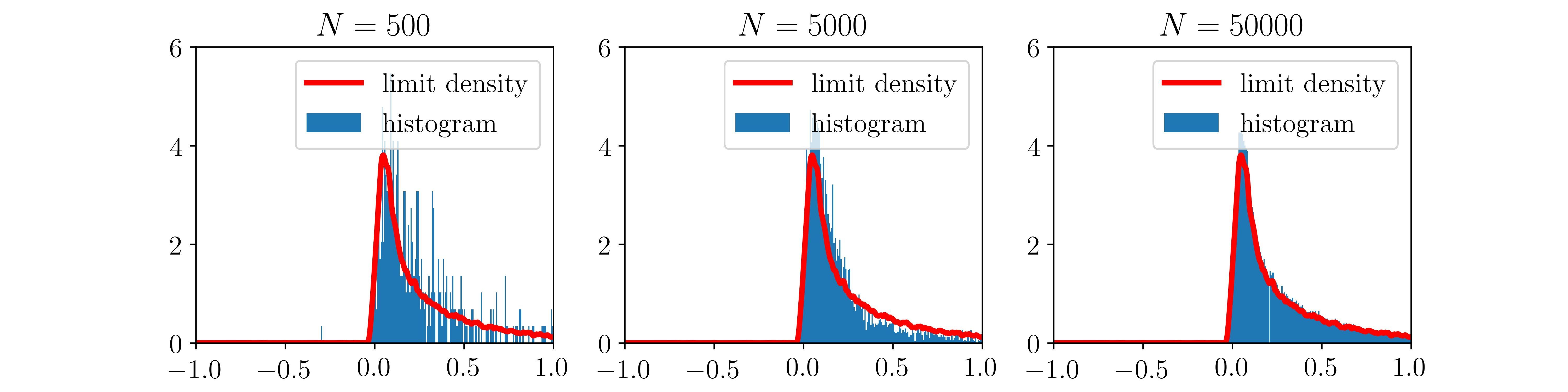} \\
  \subcaption{Case $\alpha=0$ \label{fig:cv_alpha0}}
 \end{minipage} \hfill
 \begin{minipage}[b]{.45\linewidth}
  \centering
  \includegraphics[width=1.1\linewidth]{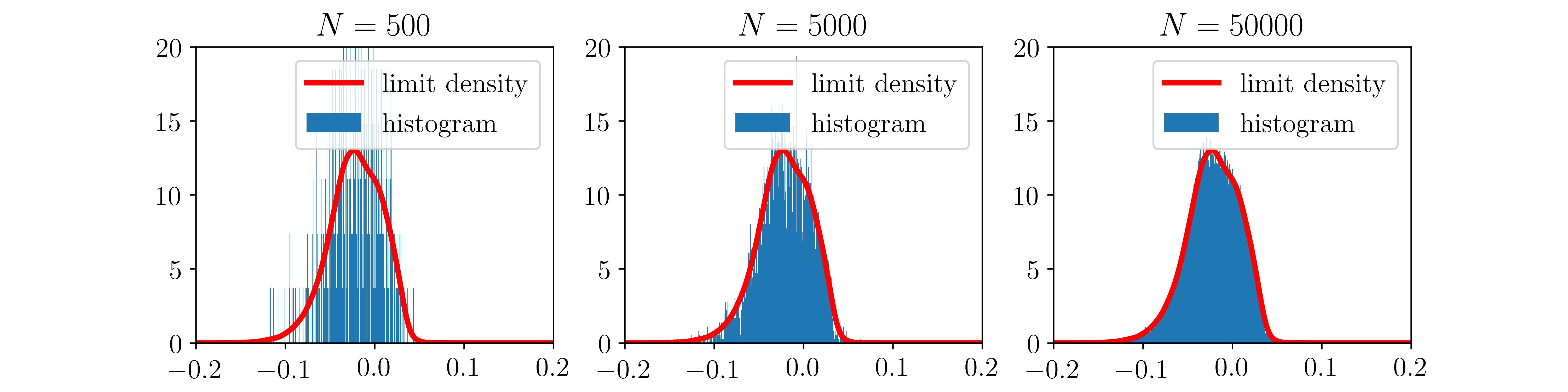}
	\includegraphics[width=1.1\linewidth]{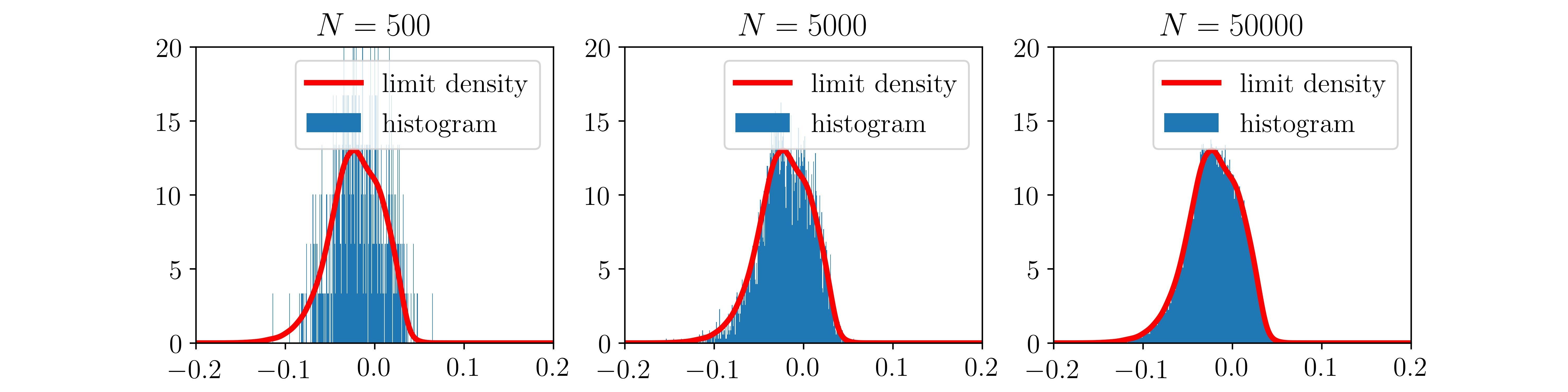}
	\includegraphics[width=1.1\linewidth]{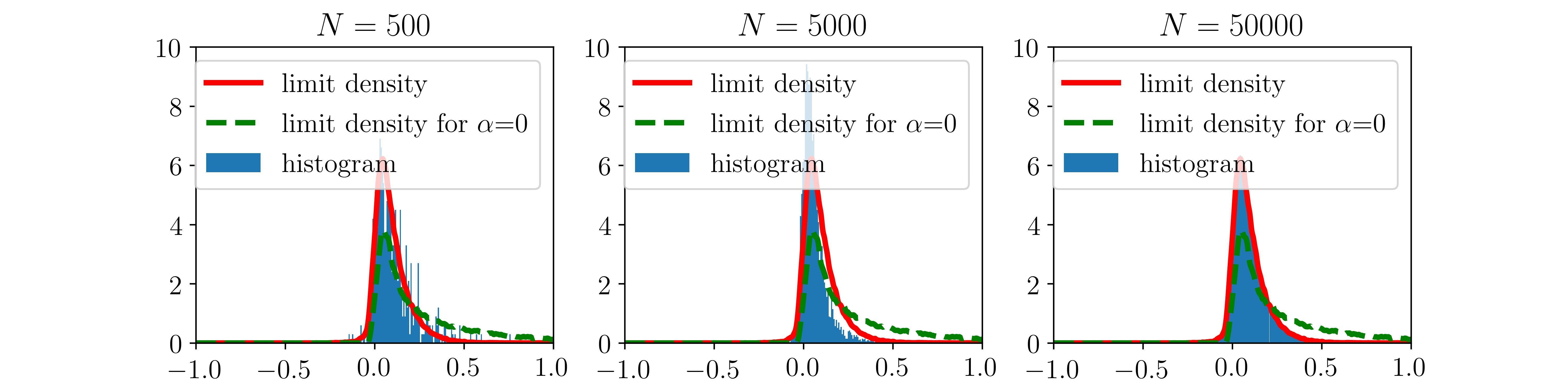}
	\subcaption{Case $\alpha=0.25$ \label{fig:cv_alpha25}}
      \end{minipage}
      \vspace{0.2cm}
 \caption{Convergence of the empirical distribution of the
weights as $N \to \plusinfty$.\label{fig:cv_empirical_measures}}
\end{figure}

First we assess the convergence of the empirical distribution of the
weights of the hidden layer to a limit distribution when
$N \to \plusinfty$. We focus on the MNIST classification task. Note that in
this case $p = 28 \times 28 =784$. In
\Cref{fig:cv_empirical_measures}, we observe the behavior of the
histograms of the weights $W_{n_T}^{1:N}$ of the hidden layer along
the coordinate $(1,1)$ as $N\to +\infty$. We
experimentally observe the existence of two different regimes, one for
$\beta<1$ and the other one for $\beta=1$.  In
\Cref{fig:cv_empirical_measures}, the first line corresponds to the
evolution of the histogram in the case where $\beta = 0.5$. The second
and the third lines correspond to the same experiment with
$\beta = 0.75$ and $\beta = 1$, respectively. Note that in both cases
the histograms converge to a limit. This limit histogram exhibits two
regimes depending if $\beta < 1$ or $\beta = 1$.

\vspace{-0.35cm}
\paragraph{Existence of two regimes.}

\begin{wrapfigure}{R}{0.5\textwidth}
\begin{tikzpicture}[spy using outlines={rectangle, black,magnification=5, connect spies}]
  \node {\pgfimage[interpolate=true,width=1.1\linewidth]{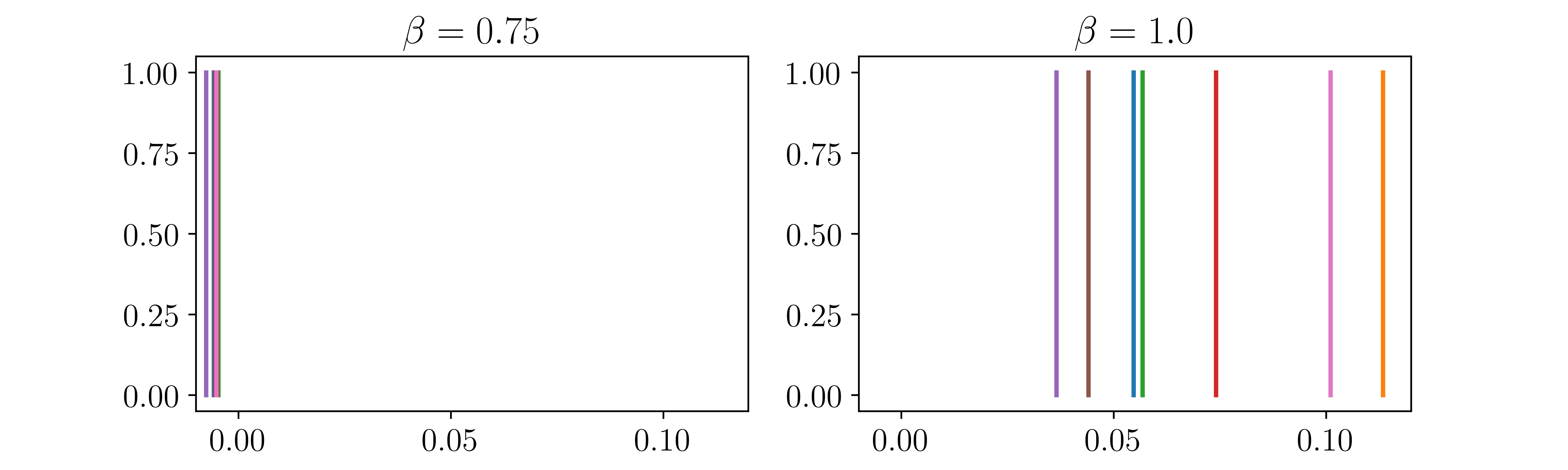}};
	\coordinate (spypoint) at (-2.77, -0.6);
	\coordinate (spyviewer) at (-1.5,-0.25);
 \spy[width=1.cm,height=1.cm] on (spypoint) in node [fill=white] at (spyviewer);
\end{tikzpicture}
\caption{Deterministic versus stochastic behavior depending on the value $\beta$.}
\label{fig:regimes}
\end{wrapfigure}

Now we assess the stochastic nature of the second regime we obtain in
the case $\beta =1$ in contrast to the regime for $\beta <1$ which is
deterministic. In order to highlight this situation, all the weights
of the neural network are initialized with a fixed value, \ie, for any
$N \in \nsets$ and $k \in \{1,\ldots,N\}$,
$W^{k,N}_0 = w_0 \in \rset^p$. Then, the neural network is trained on
the MNIST dataset for $N=10^6$ and $\beta = 0.75$ or $\beta =
1$. \Cref{fig:regimes} represents $7$ samples of the first component
of $W_{n_T}^{1:N}$ obtained with independent runs of SGD.  We can
observe that for $\beta=0.75$ all the samples converge to the same
value which agrees with \eqref{eq:mean_field_beta_small} while in the
case where $\beta=1$ they exhibit different values, which is in
accordance with \eqref{eq:mean_field_beta_one}.

\paragraph{From stochastic to deterministic.}

\begin{wrapfigure}{R}{0.62\textwidth}
	\includegraphics[width=1.1\linewidth]{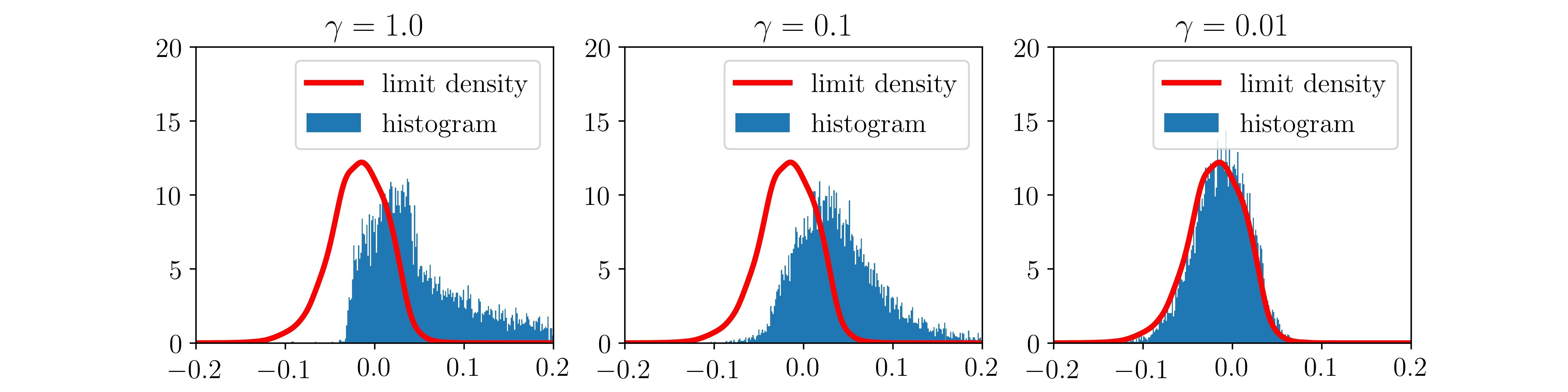}
\caption{Convergence to the deterministic regime as $\gamma \to 0$. \label{fig:cv_gammas}}
\end{wrapfigure}

We illustrate that when $\gamma \to 0$ the dynamics identified in
\eqref{eq:mean_field_beta_one} tends to the one identified in
\eqref{eq:mean_field_beta_small}. We fix $\beta = 1$ and $N=10000$ and
focus on the MNIST classification task. In \Cref{fig:cv_gammas} we
show the histogram of the weights $W_{n_T}^{1:N}$ along the coordinate
$(1,1)$ for different values of $\gamma$. As expected, see
\eqref{eq:mean_field_beta_one} and the following remark, when
$\gamma \to 0$ we recover the limit histogram with $\beta <
1$.
In \Cref{fig:CV_BS} we also study the convergence of the empirical measure when $M \to +\infty$ in the case where $\beta<1$.





\paragraph{Long-time behavior.}

Finally, we illustrate the interest of taking $\beta=1$ in our setting
by considering the more challenging classification task on the
CIFAR-10 dataset. We consider the following set of parameters
$\alpha = 0$, $M=100$, $T=10000$, $\gamma=0.1$. We emphasize that this
experiment aims at comparing the performance of the setting
$\beta <1$ and the one with $\beta=1$ and that we are not trying to
reach state-of-the-art
results. 
In \Cref{table:cifar} we present training and test accuracies for the
classification task at hand. To build the classification estimator we
average the weights along their trajectory, \ie, we perform averaging and consider the average
estimator
$\bar{W}_{n_T}^{1:N} = (n_T-n_0+1)^{-1} \sum_{n=n_0}^{n_T} W_{n}^{1:N}$, where
$n_0=1000$.
Using $\beta=1$ roughly increases
the test accuracy by $1\%$, while the training accuracy is not
$100\%$. This empirically seems to demonstrate that using a
smaller value of $\beta$ tends to overfit the data, whereas using
$\beta=1$ has a regularizing
effect. 
\vspace{-0.3cm}
\begin{table}[H]
  \centering
   \centerfloat
	\caption{Training and Test accuracies for different settings on the CIFAR-10 dataset, with $\alpha=0$, $M=100$ and $\gamma=0.1$ for $T=10000$}
    \vskip 0.15in
  \begin{tabular}{c||c|c?c|c?c|c}
    \hline
    Values & $N=5000$ &  $N=5000$ & $N=10000$ &  $N=10000$ & $N=50000$ &  $N=50000$ \\
    of $N$ and $\beta$ & $\beta=0.75$ & $\beta=1.0$ & $\beta=0.75$ & $\beta=1.0$ &  $\beta=0.75$ & $\beta=1.0$ \\ \hline
    Train acc. & $\textbf{100\%}$ & $97.2\%$ & $\textbf{100\%}$ & $97.2\%$ & $\textbf{100\%}$ & $99 \%$ \\ \hline
    Test acc. & $55.5\%$ & $\textbf{56.5\%}$ & $56.0\%$ & $\textbf{56.5\%}$ & $56.7\%$ & $\textbf{57.7\%}$ \\ \hline
  \end{tabular}
  \label{table:cifar}
\end{table}


\section{Conclusion}
\label{sec:conclusion}

We show in this paper that taking a stepsize in SGD depending on the
number of hidden units leads to particle systems with two possible
mean-field behaviours.  The first was already identified in
\cite{sirignano2018mean,mei2018mean,chizat2018global} and corresponds
to a deterministic mean-field ODE. The second is new and corresponds
to a McKean-Vlasov diffusion. Our numerical experiments on two real
datasets support our findings. In a future work, we intend to follow
the same approach for deep neural networks, \ie, with a growing number
of hidden layers. 


\nocite{sirignano2018mean}
\bibliography{main.bbl}

\begin{thebibliography}{10}

\bibitem{goodfellow:bengio:courville:2016}
I.~Goodfellow, Y.~Bengio, and A.~Courville, {\em Deep Learning}.
\newblock MIT Press, 2016.

\bibitem{krizhevsky2012imagenet}
A.~Krizhevsky, I.~Sutskever, and G.~Hinton, ``Imagenet classification with deep
  convolutional neural networks,'' in {\em Advances in neural information
  processing systems}, pp.~1097--1105, 2012.

\bibitem{manning1999foundations}
C.~Manning and H.~Sch{\"u}tze, {\em Foundations of statistical natural language
  processing}.
\newblock MIT press, 1999.

\bibitem{zhang2016understanding}
C.~Zhang, S.~Bengio, M.~Hardt, B.~Recht, and O.~Vinyals, ``Understanding deep
  learning requires rethinking generalization,'' {\em arXiv preprint
  arXiv:1611.03530}, 2016.

\bibitem{shalev2014understanding}
S.~Shalev-Shwartz and S.~Ben-David, {\em Understanding machine learning: From
  theory to algorithms}.
\newblock Cambridge university press, 2014.

\bibitem{li:et:al:2018:visualizing}
H.~Li, Z.~Xu, G.~Taylor, C.~Studer, and T.~Goldstein, ``Visualizing the loss
  landscape of neural nets,'' in {\em Advances in Neural Information Processing
  Systems}, pp.~6389--6399, 2018.

\bibitem{ballard:2017:energy}
A.~Ballard, R.~Das, S.~Martiniani, D.~Mehta, L.~Sagun, J.~Stevenson, and
  D.~Wales, ``Energy landscapes for machine learning,'' {\em Physical Chemistry
  Chemical Physics}, vol.~19, no.~20, pp.~12585--12603, 2017.

\bibitem{soltanolkotabi2019theoretical}
M.~Soltanolkotabi, A.~Javanmard, and J.~Lee, ``Theoretical insights into the
  optimization landscape of over-parameterized shallow neural networks,'' {\em
  {IEEE} Trans. Information Theory}, vol.~65, no.~2, pp.~742--769, 2019.

\bibitem{fukumizu2000local}
K.~Fukumizu and S.~Amari, ``Local minima and plateaus in hierarchical
  structures of multilayer perceptrons,'' {\em Neural networks}, vol.~13,
  no.~3, pp.~317--327, 2000.

\bibitem{bray2007statistics}
A.~J. Bray and D.~Dean, ``Statistics of critical points of gaussian fields on
  large-dimensional spaces,'' {\em Physical review letters}, vol.~98, no.~15,
  p.~150201, 2007.

\bibitem{pascanu2014saddle}
R.~Pascanu, Y.~N. Dauphin, S.~Ganguli, and Y.~Bengio, ``On the saddle point
  problem for non-convex optimization,'' {\em arXiv preprint arXiv:1405.4604},
  2014.

\bibitem{pennington2017geometry}
J.~Pennington and Y.~Bahri, ``Geometry of neural network loss surfaces via
  random matrix theory,'' in {\em Proceedings of the 34th International
  Conference on Machine Learning-Volume 70}, pp.~2798--2806, JMLR. org, 2017.

\bibitem{kawaguchi2016deep}
K.~Kawaguchi, ``Deep learning without poor local minima,'' in {\em Advances in
  neural information processing systems}, pp.~586--594, 2016.

\bibitem{freeman2016topology}
D.~Freeman and J.~Bruna, ``Topology and geometry of half-rectified network
  optimization,'' {\em arXiv preprint arXiv:1611.01540}, 2016.

\bibitem{venturi2018neural}
L.~Venturi, A.~Bandeira, and J.~Bruna, ``Neural networks with finite intrinsic
  dimension have no spurious valleys,'' {\em CoRR}, vol.~abs/1802.06384, 2018.

\bibitem{bach2017breaking}
F.~Bach, ``Breaking the curse of dimensionality with convex neural networks,''
  {\em The Journal of Machine Learning Research}, vol.~18, no.~1, pp.~629--681,
  2017.

\bibitem{choromanska:et:al:2015}
A.~Choromanska, M.~Henaff, M.~Mathieu, G.~Ben~Arous, and Y.~LeCun, ``The loss
  surfaces of multilayer networks,'' in {\em Artificial intelligence and
  statistics}, pp.~192--204, 2015.

\bibitem{venturi:bandeira:bruna:2019}
L.~Venturi, A.~S. Bandeira, and J.~Bruna, ``Spurious valleys in
  one-hidden-layer neural network optimization landscapes,'' {\em Journal of
  Machine Learning Research}, vol.~20, no.~133, pp.~1--34, 2019.

\bibitem{kuditipudi:2019:explaining_landscape}
R.~Kuditipudi, X.~Wang, H.~Lee, Y.~Zhang, Z.~Li, W.~Hu, R.~Ge, and S.~Arora,
  ``Explaining landscape connectivity of low-cost solutions for multilayer
  nets,'' in {\em Advances in Neural Information Processing Systems 32}
  (H.~Wallach, H.~Larochelle, A.~Beygelzimer, F.~d'Alch\'{e} Buc, E.~Fox, and
  R.~Garnett, eds.), pp.~14601--14610, Curran Associates, Inc., 2019.

\bibitem{allen:zho:li:song:2019}
Z.~Allen-Zhu, Y.~Li, and Z.~Song, ``On the convergence rate of training
  recurrent neural networks,'' in {\em Advances in Neural Information
  Processing Systems 32} (H.~Wallach, H.~Larochelle, A.~Beygelzimer,
  F.~d'Alch\'{e} Buc, E.~Fox, and R.~Garnett, eds.), pp.~6676--6688, Curran
  Associates, Inc., 2019.

\bibitem{sarao:manelli:biroli:et:al:2019}
S.~S.~Mannelli, G.~Biroli, C.~Cammarota, F.~Krzakala, and L.~Zdeborov\'{a},
  ``Who is afraid of big bad minima? analysis of gradient-flow in spiked
  matrix-tensor models,'' in {\em Advances in Neural Information Processing
  Systems 32} (H.~Wallach, H.~Larochelle, A.~Beygelzimer, F.~d'Alch\'{e} Buc,
  E.~Fox, and R.~Garnett, eds.), pp.~8679--8689, Curran Associates, Inc., 2019.

\bibitem{sirignano2018mean}
J.~Sirignano and K.~Spiliopoulos, ``Mean field analysis of neural networks,''
  {\em arXiv preprint arXiv:1805.01053}, 2018.

\bibitem{mei2018mean}
S.~Mei, A.~Montanari, and P.~Nguyen, ``A mean field view of the landscape of
  two-layer neural networks,'' {\em Proceedings of the National Academy of
  Sciences}, vol.~115, no.~33, pp.~E7665--E7671, 2018.

\bibitem{rotskoff2018trainability}
G.~M. Rotskoff and E.~Vanden-Eijnden, ``Trainability and accuracy of neural
  networks: An interacting particle system approach,'' 2018.

\bibitem{mei2019mean}
S.~Mei, T.~Misiakiewicz, and A.~Montanari, ``Mean-field theory of two-layers
  neural networks: dimension-free bounds and kernel limit,'' {\em arXiv
  preprint arXiv:1902.06015}, 2019.

\bibitem{javanmard2019analysis}
A.~Javanmard, M.~Mondelli, and A.~Montanari, ``Analysis of a two-layer neural
  network via displacement convexity,'' {\em arXiv preprint arXiv:1901.01375},
  2019.

\bibitem{chizat2019sparse}
L.~Chizat, ``Sparse optimization on measures with over-parameterized gradient
  descent,'' {\em arXiv preprint arXiv:1907.10300}, 2019.

\bibitem{chizat2018global}
L.~Chizat and F.~Bach, ``On the global convergence of gradient descent for
  over-parameterized models using optimal transport,'' in {\em Advances in
  neural information processing systems}, pp.~3036--3046, 2018.

\bibitem{jabir2019mean}
J.~Jabir, D.~{\v{S}}i{\v{s}}ka, and L.~Szpruch, ``Mean-field neural odes via
  relaxed optimal control,'' {\em arXiv preprint arXiv:1912.05475}, 2019.

\bibitem{ambrosio2008gradient}
L.~Ambrosio, N.~Gigli, and G.~Savar\'{e}, {\em Gradient flows in metric spaces
  and in the space of probability measures}.
\newblock Lectures in Mathematics ETH Z\"{u}rich, Birkh\"{a}user Verlag, Basel,
  second~ed., 2008.

\bibitem{erbar2010heat}
M.~Erbar, ``The heat equation on manifolds as a gradient flow in the
  wasserstein space,'' in {\em Annales de l'institut Henri Poincar{\'e} (B)},
  vol.~46, pp.~1--23, 2010.

\bibitem{ambrosio:savare:zambotti:2009}
L.~Ambrosio, G.~Savar{\'e}, and L.~Zambotti, ``Existence and stability for
  {F}okker--{P}lanck equations with log-concave reference measure,'' {\em
  Probability Theory and Related Fields}, vol.~145, no.~3, pp.~517--564, 2009.

\bibitem{sznitman1991topics}
A.~Sznitman, ``Topics in propagation of chaos,'' in {\em Ecole d'{\'e}t{\'e} de
  probabilit{\'e}s de Saint-Flour XIX—1989}, pp.~165--251, Springer, 1991.

\bibitem{gottlieb2000markov}
A.~Gottlieb, ``Markov transitions and the propagation of chaos,'' {\em arXiv
  preprint math/0001076}, 2000.

\bibitem{meleard1988systemes}
S.~M{\'e}l{\'e}ard and S.~Roelly-Coppoletta, ``Syst{\`e}mes de particules et
  mesures-martingales: un th{\'e}or{\`e}me de propagation du chaos,'' {\em
  S{\'e}minaire de probabilit{\'e}s de Strasbourg}, vol.~22, pp.~438--448,
  1988.

\bibitem{jordan1998variational}
R.~Jordan, D.~Kinderlehrer, and F.~Otto, ``The variational formulation of the
  fokker--planck equation,'' {\em SIAM journal on mathematical analysis},
  vol.~29, no.~1, pp.~1--17, 1998.

\bibitem{mckean:1967}
H.~P. McKean, Jr., ``Propagation of chaos for a class of non-linear parabolic
  equations,'' in {\em Stochastic {D}ifferential {E}quations ({L}ecture
  {S}eries in {D}ifferential {E}quations, {S}ession 7, {C}atholic {U}niv.,
  1967)}, pp.~41--57, Air Force Office Sci. Res., Arlington, Va., 1967.

\bibitem{fontaine2020continuous}
X.~Fontaine, V.~D. Bortoli, and A.~Durmus, ``Continuous and discrete-time
  analysis of stochastic gradient descent for convex and non-convex
  functions,'' 2020.

\bibitem{welling2011bayesian}
M.~Welling and Y.~Teh, ``Bayesian learning via stochastic gradient langevin
  dynamics,'' in {\em Proceedings of the 28th international conference on
  machine learning (ICML-11)}, pp.~681--688, 2011.

\bibitem{sirignano2020mean}
J.~Sirignano and K.~Spiliopoulos, ``Mean field analysis of neural networks: A
  central limit theorem,'' {\em Stochastic Processes and their Applications},
  vol.~130, no.~3, pp.~1820--1852, 2020.

\bibitem{mnist}
Y.~LeCun and C.~Cortes, ``{MNIST} handwritten digit database,'' 2010.

\bibitem{krizhevsky2009learning}
A.~Krizhevsky, G.~Hinton, {\em et~al.}, ``Learning multiple layers of features
  from tiny images,'' 2009.

\bibitem{villani2009optimal}
C.~Villani, {\em Optimal transport}, vol.~338 of {\em Grundlehren der
  Mathematischen Wissenschaften [Fundamental Principles of Mathematical
  Sciences]}.
\newblock Springer-Verlag, Berlin, 2009.
\newblock Old and new.

\bibitem{kechris:2012}
A.~Kechris, {\em Classical Descriptive Set Theory}.
\newblock Graduate Texts in Mathematics, Springer New York, 2012.

\bibitem{stroock2006multi}
D.~W. Stroock and S.~R.~S. Varadhan, {\em Multidimensional diffusion
  processes}.
\newblock Classics in Mathematics, Springer-Verlag, Berlin, 2006.
\newblock Reprint of the 1997 edition.

\bibitem{karatzas1991brownian}
I.~Karatzas and S.~E. Shreve, {\em Brownian motion and stochastic calculus},
  vol.~113 of {\em Graduate Texts in Mathematics}.
\newblock Springer-Verlag, New York, second~ed., 1991.

\bibitem{rogers2000diffusionsII}
L.~C.~G. Rogers and D.~Williams, {\em Diffusions, {M}arkov processes, and
  martingales. {V}ol. 2}.
\newblock Cambridge Mathematical Library, Cambridge University Press,
  Cambridge, 2000.
\newblock It\^{o} calculus, Reprint of the second (1994) edition.

\bibitem{nesterov2004introductory}
Y.~Nesterov, {\em Introductory lectures on convex optimization}, vol.~87 of
  {\em Applied Optimization}.
\newblock Kluwer Academic Publishers, Boston, MA, 2004.
\newblock A basic course.

\bibitem{ambrosio2013user}
L.~Ambrosio and N.~Gigli, ``A user's guide to optimal transport,'' in {\em
  Modelling and optimisation of flows on networks}, vol.~2062 of {\em Lecture
  Notes in Math.}, pp.~1--155, Springer, Heidelberg, 2013.

\bibitem{bonsall1962lectures}
F.~F. Bonsall and K.~Vedak, {\em Lectures on some fixed point theorems of
  functional analysis}.
\newblock No.~26, Tata Institute of Fundamental Research Bombay, 1962.

\bibitem{kent:1978}
J.~Kent, ``Time-reversible diffusions,'' {\em Advances in Applied Probability},
  vol.~10, pp.~819--835, 12 1978.

\end{thebibliography}
\bibliographystyle{ieeetr}
\section{Preliminaries}
\subsection{Notation}
Let $(\mse, d_E)$ and $(\msf, d_F)$ be two metric spaces.
$\rmc(\mse, \msf)$ stands for the set of continuous $\msf$-valued
functions. If $\msf = \rset$, then we simply note
$\rmc(\mse)$.

We say that $f: \ \mse \to \rset^p$ is $L$-Lipschitz
if there exists $L \geq 0$ such that for any $x,y \in \mse$,
$\normLigne{f(x) - f(y)} \leq L d_E(x,y)$. 
Let $\rmcb(\mse, \rset^p)$ (respectively $\rmcc(\mse, \rset^p)$) be the set of
bounded continuous functions from $\mse$ to $\rset^p$ (respectively the set of
compactly supported functions from $\mse$ to $\rset^p$). If $p=1$, we
simply note $\rmcb(\mse)$ (respectively $\rmcc(\mse)$).

For $\msu$ an open set of $\rset^d$, $n \in \nsets$ and define
$\rmC^n(\msu, \rset^p)$ the set of the $n$-differentiable
$\rset^p$-valued functions over $\msu$. If $p=1$ then we
simply note $\rmC^n(\msu)$. Let $f \in \rmC^1(\msu)$ we denote by
$\nabla f$ its gradient.  More generally, if
$f \in \rmC^n(\msu, \rset^p)$ with $n, p \in \nsets$, we denote by
$\rmD^k f(x)$ the $k$-th differential of $f$. We also denote for any
$i \in \{1, \dots, d\}$ and $\ell \in \{1, \dots, k\}$,
$\partial_i^{\ell} f$ the $i$-th partial derivative of $f$ of order $\ell$.  If
$f \in \rmc^2(\rset^{\dim}, \rset)$, we denote by $\Delta f$ its
Laplacian. $\rmcc^n(\msu, \rset^p)$ is the subset of
$\rmc^n(\msu, \rset^p)$ such that for any
$f \in \rmcc^n(\msu, \rset^p)$ and $\ell \in \{0, \dots, n\}$,
$\rmD^{\ell} f$ has compact support.


Consider $(\msf, d)$ a metric space. Let $\Pens(\msf)$ be the space of
probability measures over $\msf$ equipped with its Borel
$\sigma$-field $\mcb{\msf}$.  For any $\mu \in \Pens(\msf)$ and
$f: \ \msf \to \rset$, we say that $f$ is $\mu$-integrable if
$\int_{\msf} \absLigne{f(x)} \rmd \mu(x) < +\infty$. In this case, we
set $\mu[f] = \int_{\msf} f(x) \rmd \mu(x)$. Let
$\mu_0 \in \Pens(\msf)$. For any $r \geq 1$, define
$\Pens_r(\msf) = \ensembleLigne{\mu \in \Pens(\msf)}{\int_{\rset^p}
  d(\mu_0, \mu)^r \rmd \mu(x) < +\infty}$.  If not specified, we
consider a filtered probability space
$(\Omega,\mcf,\PP,(\mcf_t)_{t \geq 0})$ satisfying the usual
conditions and any random variables is defined on this probability
space.  Let $f: \ (\mse, \mce) \to (\msg, \mcg)$ be a measurable
function. Then for any measure $\mu$ on $\mce$ we define its
pushforward measure by $f$, $f_{\#} \mu$, for any $\msa \in \mcg$ by
$f_{\#}\mu(\msa) = \mu(f^{-1}(\msa))$.

The set of $m \times n$ real matrices is denoted by
$\rset^{m \times n}$. The set of symmetric real matrices of size
$p$ is denoted $\mathbb{S}_p(\rset)$.

\subsection{Wasserstein distances}
\label{sec:preliminaries}

Let $(\msf, d)$ be a metric space.  Let $\mu_1,\mu_2 \in \Pens(\msf)$,
where $\msf$ is equipped with its Borel $\sigma$-field $\mcb{\msf}$.
A probability measure $\zeta$ over $\mcb{\msf}^{\otimes 2}$ is said
to be a transference plan between $\mu_1$ and $\mu_2$ if for any
$\msa \in \mcb{\msf}$,
$\zeta(\msa \times \msf) = \mu_1(\msa)$ and
$\zeta(\msf \times \msa) = \mu_2(\msa)$. We denote by
$\Lambda(\mu_1, \mu_2)$ the set of all transference plans
between $\mu_1$ and $\mu_2$.  If $\mu_1,\mu_2 \in \Pens_r(\rset^p)$,
we define the Wasserstein distance $\wassersteinD[r](\mu_1, \mu_2)$ of
order $r$ between $\mu_1$ and $\mu_2$ by
\begin{equation}
  \label{eq:def_distance_wasser}
  \wassersteinD[r]^r(\mu_1, \mu_2) = \inf_{\zeta \in \Lambda(\mu_1, \mu_2)} \defEns{ \int_{\msf \times \msf}  d(x,y)^r \rmd \zeta (x,y)} \eqsp.
\end{equation}
Note that $\wassersteinD[r]$ is a distance on $\Pens_r(\msf)$ by
\cite[Theorem 6.18]{villani2009optimal}. In addition
$(\Pens_r(\rset^p),\wassersteinD[r])$ is a complete separable metric
space.  For any $\mu_1, \mu_2 \in \Pens_p(\msf)$ we say that a
couple of random variables $(X,Y)$ is an optimal coupling of
$(\mu_1, \mu_2)$ for $\wassersteinD[p]$ if it has distribution $\xi$
where $\xi$ is an optimal transference plan between $\mu_1$ and
$\mu_2$.  
For any $T\geq 0$, the space
$\scrC_{2,T}^p = \rmC(\ccint{0,T},\MRpdeux)$ is 
a complete separable metric space \cite[Theorem 4.19]{kechris:2012} with the
metric $\dinfwass$ given for any $(\nu_t)_{t \in \ccint{0,T}}$ and
$(\mu_t)_{t \in \ccint{0,T}}$ by
\begin{equation}
  \dinfwass((\nu_t)_{t \in \ccint{0,T}}, (\mu_t)_{t \in \ccint{0,T}}) = \sup_{t \in \ccint{0,T}} \wassersteinD[2](\nu_t, \mu_t) \eqsp .
\end{equation}


In the case where the measures we consider can be written as sums of
Dirac we have the following proposition.
\begin{proposition}
  \label{prop:fm_bound}
  Let $r \geq 1$, $N \in \nsets$, $\{\alpha_k\}_{k=1}^N \in \ccint{0,1}^N$ with
  $\sum_{k=1}^N \alpha_k = 1$, $\{\mu_{k,a}\}_{k=1}^N \in \Pens(\msf)^N$ and
  $\{\mu_{k,b}\}_{k=1}^N \in \Pens(\msf)^N$. Then, setting
  $\nu_i = \sum_{k=1}^N \alpha_k \mu_{k,i}$ with $i\in\{a,b\}$,
  we have
  \begin{equation}
    \wassersteinD[r]^r(\nu_a,\nu_b)w  \leq \sum_{k=1}^N \wassersteinD[r]^r(\mu_{k,a},\mu_{k,b}) \eqsp .
  \end{equation}
\end{proposition}

\begin{proof}
  Consider
  $\zeta = \sum_{k=1}^N \alpha_k \zeta_k \in
  \Lambda(\nu_a,\nu_b)$ with $\zeta_k$ the optimal transference plan
  between $\mu_{k,a}$ and $\mu_{k,b}$. Then, we have
  \begin{equation}
    \wasserstein[r]{\nu_a, \nu_b}[r] \leq \int_{\rset^p \times \rset^p} d(x,y)^r \rmd \zeta (x,y)  \leq N^{-1} \sum_{k=1}^N \wassersteinD[r]^r(\mu_{k,a},\mu_{k,b}) \eqsp .
  \end{equation}
\end{proof}

As a special case of \Cref{prop:fm_bound}, we obtain that for any
$r \geq 1$, $\{w_{k,a}\}_{k=1}^N \in \msf^N$ and
$\{w_{k,a}\}_{k=1}^N \in \msf^N$,
\begin{equation}
  \wassersteinD[r](N^{-1} \sum_{k=1}^N \updelta_{w_{k,a}}, N^{-1} \sum_{k=1}^N \updelta_{w_{k,b}}) \leq N^{-1} \sum_{k=1}^N d(w_{k,a}, w_{k,b})^r \eqsp .
\end{equation}
As another special case of \Cref{prop:fm_bound}, we obtain that for
any $\mu \in \Pens_r(\msf)$ and $\{w_{k}\}_{k=1}^N \in \msf^N$
\begin{equation}
  \wassersteinD[r](N^{-1} \sum_{k=1}^N \updelta_{w_{k}}, N^{-1}, \mu) \leq N^{-1} \sum_{k=1}^N \wassersteinD[r](w_{k},\mu)^r   \eqsp .
\end{equation}




\section{A mean-field modification of  Stochastic Gradient Langevin Dynamics}
\label{sec:stoch-grad-lang}

\subsection{Presentation of the modified SGLD and its continuous counterpart}

We start by introducing a modified Stochastic Gradient Langevin Dynamics
(mSGLD) \cite{welling2011bayesian}.  In the mean-field regime, this
setting was studied in the case $\beta = 0$ in \cite{mei2018mean}.  We
recall that the mean-field
$h : \rset^p \times \Pens(\rset^d)\to \rset^p$ and
$\xi : \rset^p \times \Pens(\rset^d) \times \msx\times \msy \to
\rset^p$ are given for any $\mu \in \MRp$, $w \in \rset^p$,
$(x,y) \in \msx \times \msy$ by
\begin{align}
  \label{eq:h}
&  h(w, \mu)=-\int_{\msx \times \msy}  \partial_1
\ell \parenthese{\mu[F(\cdot,x)] ,y} \nablaw F(w,x) \, \dd \datap(x,y) - \nabla V(w) \eqsp, \\
  \label{eq:def_xi_n}
 &  \xi(w,\mu,x,y)=-h(w, \mu) - \partial_1 \ell(\mu[F(\cdot,x)] ,y) \nablaw F(w,x)  - \nabla V(w)\eqsp .
\end{align}
 Let
$(W_0^k)_{ k\in\nsets}$ be \iid~$p$ dimensional random variables with
distribution $\mu_0$ and       $\ensembleLigne{Z_k^n}{k, n \in \nsets}$ be \iid~$p$
      dimensional independent Gaussian random variables with zero mean
      and identity covariance matrix.  Consider the sequence
$(W_n^{1:N})_{n \in \nset}$ associated with mSGLD starting from
$W_0^{1:N}$ and defined by the following recursion: for any
$n \in \nset$, $k\in \{1,\ldots,N\}$,
\begin{multline}
	\label{eq:sgld}
	\WkN\nn = \WkN_n + \gamma N^{\beta - 1}(n+\gua\pinv)^{-\alpha} \defEns{h(\WkN_n,\nunn) + \xi(\WkN_n,\nunn,X_n,Y_n)} \\ + \parentheseDeux{2 \eta \gamma N^{\beta - 1} (n+\gua^{-1})^{-\alpha}}^{1/2} Z_{k,n}  \eqsp ,
      \end{multline}
      where $\eta \geq 0$, $\beta \in \ccint{0,1}$,
      $\alpha \in \coint{0,1}$,  $\gamma>0$, $(X_n, Y_n)_{n \in \nset}$ is a sequence of
      \iid~input/label samples distributed according to $\pi$ and
      $\gua = \gamma^{1/(1-\alpha)} N^{(\beta - 1)/(1 - \alpha)}$. Note that in the cas $\eta = 0$,
      we obtain \eqref{eq:sgd}. In addition, \eqref{eq:sgld} does not exactly correspond to the usual implementation of mSGLD as introduced in \cite{welling2011bayesian}. Indeed, to recover this algorithm, we should replace $\parentheseDeuxLigne{2 \eta \gamma N^{\beta - 1} (n+\gua^{-1})^{-\alpha}}^{1/2} Z_{k,n}$ by $\parentheseDeuxLigne{2 \eta \gamma N^{\beta} (n+\gua^{-1})^{-\alpha}}^{1/2} Z_{k,n}$ in \eqref{eq:sgld}.  The scheme presented in \eqref{eq:sgld} amounts to consider a temperature which scales as $\gamma N^{\beta - 1}$ with the number of particles. As emphasized before, this scheme was also considered in \cite{mei2018mean}.

      We now present the continuous model associated with this discrete
process in the limit $\gamma \to 0$ or $N \to \plusinfty$. For $N \in\nsets$, consider the particle system diffusion $(\bfW^{1:N}_t)_{t \geq 0} = (\{\bfW^{k,N}_t\}_{k=1}^N)_{t \geq 0}$  starting from $\bfW^{1:N}_0$ defined for any
$k \in \{1, \dots, N\}$ by
\begin{equation}
	\label{eq:sde_diff}
	\dd \bfwkN_t = (t+1)^{-\alpha} \defEns{ h(\bfwkN_t,\bfnu^N_t)\dd t+\gua^{1/2}\Sigma^{1/2}(\bfwkN_t,\bfnu^N_t) \dd \bfB^{k}_t + \sqrt{2 \eta} \rmd \tbfB_t^{k}}\eqsp ,
\end{equation}
where $\ensembleLigne{(\bfB_t^{k})_{t \geq 0}}{k \in \nsets}$ and
$\ensembleLigne{(\tbfB_t^{k})_{t \geq 0}}{k \in \nsets}$ are two
independent families of independent $p$ dimensional Brownian motions
and $\bfnun_t$ is the empirical probability distribution of the
particles defined for any $t \geq 0$ by
$\bfnun_t=N^{-1}\sum_{k=1}^N \updelta_{\bfwkN_t}$.
Similarly to \Cref{sec:setting}, \eqref{eq:sde_diff} is the continuous
counterpart of \eqref{eq:sgld}. Let $M \in \nsets$. Similarly to \eqref{eq:sde_batch}, we consider the
following particle system diffusion
$(\bfW^{1:N}_t)_{t \geq 0} = (\{\bfW^{k,N}_t\}_{k=1}^N)_{t \geq 0}$
starting from $\bfW^{1:N}_0$ defined for any $k \in \{1, \dots, N\}$
by
\begin{equation}
  \label{eq:sde_diff_batch}
	\dd \bfwkN_t = (t+1)^{-\alpha} \defEns{ h(\bfwkN_t,\bfnu^N_t)\dd t+(\gua/M)^{1/2}\Sigma^{1/2}(\bfwkN_t,\bfnu^N_t) \dd \bfB^k_t + \sqrt{2 \eta} \rmd \tbfB_t^{k}} \eqsp .
      \end{equation}

\subsection{Mean field approximation and propagation of chaos for mSGLD}

The following theorems are the extensions of \Cref{thm:empi_conv_cont}
and \Cref{thm:empi_conv_cont_one} to \eqref{eq:sde_diff} for any
$\eta \geq 0$. Note that in the case $\eta = 0$,
\Cref{thm:empi_conv_cont_sgld} boils down to \Cref{thm:empi_conv_cont}
and \Cref{thm:empi_conv_cont_one_sgld} to
\Cref{thm:empi_conv_cont_one}.

We start by stating our results in the case $\beta \in \coint{0,1}$.
Consider the mean-field SDE starting from a random variable
$\bfw_0^{\star}$ given by
  \begin{equation}
    \label{eq:mean_field_beta_small_sgld}
    \rmd \bfw_t^{\star} = (t+1)^{-\alpha} \defEns{h(\bfw_t^{ \star}, \bflambda_t^{ \star}) \rmd t + \sqrt{2 \eta} \tbfB_t} \eqsp , \qquad \text{with $\bflambda^{\star}_t$ the distribution of $\bfw_t^{ \star}$} \eqsp.
  \end{equation}
\begin{theorem}
  \label{thm:empi_conv_cont_sgld}
    Assume \rref{assum:all}. Let $(\bfw_0^k)_{k \in \nset}$ be a
  sequence of \iid~$\rset^p$-valued random variables with distribution
  $\mu_0 \in \Pens_2(\rset^p)$ and set for any $N \in \nsets$,
  $\bfw_0^{1:N} = \{\bfw_0^{k}\}_{k=1}^N$. Then, for any
  $\nbparticlem \in \nsets$ and $T \geq 0$, there exists $C_{\nbparticlem,T} \geq 0$
  such that for any $\alpha \in \coint{0,1}$, $\beta \in \coint{0,1}$,
  $M \in \nsets$ and $N \in \nsets$
  \begin{equation}
   \textstyle{ \expe{\sup_{t \in \ccint{0,T}} \normLigne{\bfw_t^{1:\nbparticlem,N} - \bfw_t^{1:\nbparticlem, \star}}^2} \leq C _{\nbparticlem, T} \defEns{N^{-(1-\beta)/(1-\alpha)}M^{-1} + N^{-1}} \eqsp ,}
  \end{equation}
with
$(\bfw_t^{1:\nbparticlem,N}, \bfw_t^{1:\nbparticlem, \star}) = \{(\bfw_t^{k, N}, \bfw_t^{k, \star})\}_{k=1}^{\nbparticlem}$, $(\bfw_t^{1:N})$ is the solution of \eqref{eq:sde_diff_batch} starting from $\bfw_0^{1:N}$, and for any $k \in \{1,\ldots,N\}$, $\bfw_t^{k, \star}$ is the solution of \eqref{eq:mean_field_beta_small_sgld} starting from $\bfw_0^k$ and Brownian motion $(\tbfB^{k}_{t})_{t \geq 0}$.
\end{theorem}

\begin{proof}
  The proof is postponed to \Cref{sec:proofs-crefthm:-cref}
\end{proof}

Consider now the mean-field SDE starting from a random variable
$\bfw_0^{\star}$ given by
  \begin{equation}
    \label{eq:mean_field_beta_one_sgld}
    \rmd \bfw_t^{\star} = (t+1)^{-\alpha} \defEns{h(\bfw_t^{ \star}, \bflambda_t^{ \star}) \rmd t
    + (\gamma^{1/(1-\alpha)}\Sigma(\bfw_t^{ \star}, \bflambda_t^{\star})/M)^{1/2} \rmd \bfB_t + \sqrt{2 \eta} \rmd \tbfB_t} \eqsp ,
  \end{equation}
  where $\bflambda^{\star}_t$ is the distribution of $\bfw_t^{ \star}$
  and $(\bfB_t)_{t \geq 0}$ and $(\tbfB_t)_{t \geq 0}$ are independent
  $p$ dimensional Brownian motions.

\begin{theorem}
  \label{thm:empi_conv_cont_one_sgld}
  Let $\beta = 1$.  Assume \rref{assum:all}. Let
  $(\bfw_0^k)_{k \in \nset}$ be a sequence of $\rset^p$-valued random
  variables with distribution $\mu_0 \in \Pens_2(\rset^p)$ and assume
  that for any $N \in \nsets$,
  $\bfw_0^{1:N} = \{\bfw_0^{k}\}_{k=1}^N$. Then, for any
  $\nbparticlem \in \nsets$ and $T \geq 0$, there exists $C_{\nbparticlem,T} \geq 0$
  such that for any $\alpha \in \coint{0,1}$, $M \in \nsets$ and
  $N \in \nsets$ we have
  \begin{equation}
   \textstyle{ \expe{\sup_{t \in \ccint{0,T}} \normLigne{\bfw_t^{1:\nbparticlem,N} - \bfw_t^{1:\nbparticlem, \star}}^2} \leq C _{\nbparticlem, T} N^{-1} \eqsp , }
 \end{equation}
 with
 $(\bfw_t^{1:\nbparticlem,N}, \bfw_t^{1:\nbparticlem, \star}) =
 \{(\bfw_t^{k, N}, \bfw_t^{k, \star})\}_{k=1}^{\nbparticlem}$,
 $(\bfw_t^{1:N})$ is the solution of \eqref{eq:sde_diff_batch}
 starting from $\bfw_0^{1:N}$, and for any $k \in \{1,\ldots,N\}$,
 $\bfw_t^{k, \star}$ is the solution of
 \eqref{eq:mean_field_beta_one_sgld} starting from $\bfw_0^k$ and
 Brownian motions $(\bfB^{k}_{t})_{t \geq 0}$ and
 $(\tbfB^{k}_{t})_{t \geq 0}$.
\end{theorem}

\begin{proof}
  The proof is postponed to \Cref{sec:proofs-crefthm:-cref}
\end{proof}



\section{Technical results}
\label{sec:mean-field-prop}


In this section, we derive technical results needed to establish \Cref{thm:empi_conv_cont}, \Cref{thm:empi_conv_cont_one}, \Cref{thm:empi_conv_cont_sgld} and \Cref{thm:empi_conv_cont_one_sgld}. In particular, we are interested in the regularity properties of
the mean field $h$ and the diffusion matrix $\Sigma$ under
\rref{assum:all}. We recall that in this setting, for any
$w \in \rset^p$, $\mu \in \MRp$, $(x,y) \in \msx \times \msy$, we have
\begin{align}
  h(w, \mu)&= \tilde{h}(w,\mu) - \nabla V(w) \eqsp ,\\
  \text{ with  }\quad   \tilde{h}(w,\mu)& = -\int_{\msx \times \msy}  \partial_1 \ell \parenthese{\int_{\bR^p} F(\zeta,x) \,\dd \mu(\zeta) ,y}
                                          \nablaw F(w,x) \, \dd \datap(x,y) \eqsp, \\
  \xi(w,\mu,x,y)&=-\tilde{h}(w,\mu)- \partial_1 \ell \parenthese{ \int_{\bR^p}
                  F(\zeta,x) \,\dd \mu(\zeta) ,y} \nablaw F(w,x) \eqsp ,\\
  \Sigma(w,\mu)&=\int_{\msx\times\msy} \{\xi\xi\transpose\}(w,\mu,x,y)  \dd \datap(x,y)\eqsp, \qquad
                 \rmS(w,\mu) = \Sigma^{1/2}(w,\mu) \eqsp. \label{eq:h_supp_vrai}
\end{align}
Note that by \Cref{assum:all}-\ref{item:assuml}, we obtain the
following estimate used in the proof of the results of this Section:
for any $\rmy, y \in \rset$
  \begin{equation}
    \label{eq:majo_partial_ell}
    \abs{\partial_1 \ell(\rmy, y)} \leq \abs{\partial_1 \ell(0, y)} + \Psi(y) \abs{\rmy} \leq 2 \Psi(y) \max(1, \abs{\rmy})  \eqsp .
  \end{equation}
  In addition, note that under \Cref{assum:all}-\ref{item:assumV},
  there exists $\Ktt \geq 0$ such that for any $w \in \rset^p$
  \begin{equation}
    \label{eq:V_ineq}
\norm{\nabla^2 V(w)} + \norm{\rmD^3 V(w)} \leq \Ktt \eqsp , \qquad \norm{\nabla V(w)} \leq \Ktt (1 + \norm{w}) \eqsp .
\end{equation}

Let $G: \ \rset^p \times \msx \times \msy \to \rset$ given for any $(x,y) \in \msx \times \msy$ and $w \in \rset^p$ by
\begin{equation}
  \label{eq:G_def}
    G(w, x,y) = \defEnsLigne{\Phi^4(x) + \Psi^2(y)} F(w, x)  \eqsp .
  \end{equation}

  We now state our main regularity/boundedness proposition.
\begin{proposition}
  \label{prop:h_reg}
  Assume \rref{assum:all}.
Then, there exists $\Lip \geq 0$
  such that the following hold.
  \begin{enumerate}[wide, labelwidth=!, labelindent=0pt,label=(\alph*)]
  \item For any $\mu_1, \mu_2 \in \MRp$ and $w_1, w_2 \in \rset^p$ we have
    \begin{multline}
    \norm{h(w_1,\mu_1) - h(w_2,\mu_2)} \\ \leq \Lip \defEns{\norm{w_1 - w_2} + \parenthese{\int_{\msx\times\msy}\norm{\mu_1[G(\cdot, x, y)] - \mu_2[G(\cdot, x, y)]}^2 \rmd \datap(x,y)}^{1/2}} \eqsp .       \label{eq:lip_h}
  \end{multline}
  In addition, we have for any $\mu \in \Pens(\rset^p)$ and $w \in \rset^p$, $\normLigne{h(w,\mu)} \leq \Lip(1+\normLigne{w})$ and $\normLigne{\bar{h}(w, \mu)} \leq \Lip$.
  \item For any $\mu_1, \mu_2 \in \MRp$, $w_1, w_2 \in \rset^p$ and $i,j \in \{1, \dots, p \}$ we have
    \begin{multline}
    \abs{\rmS_{i,j}(w_1, \mu_1) - \rmS_{i,j}(w_2, \mu_2)} \\ \leq \Lip \defEns{\norm{w_1 - w_2} + \parenthese{\int_{\msx \times \msy }\norm{\mu_1[G(\cdot, x, y)] - \mu_2[G(\cdot, x, y)]}^2 \rmd \datap(x,y)}^{1/2}} \eqsp . \label{eq:lip_sig}
  \end{multline}
  In addition, we have for any $\mu \in \Pens(\rset^p)$,
  $w \in \rset^p$ and $i, j \in \{1, \dots, p\}$,
  $\absLigne{\rmS_{i,j}(w, \mu)} \leq \Lip$.
\item For any $\mu \in \MRp$ and $w \in \rset^p$,
  $\int_{\msx \times \msy} \norm{\xi(w, \mu, x, y)}^2 \rmd \datap(x,y)
  \leq p^2\Lip^2$.
  \end{enumerate}
\end{proposition}

\begin{proof}
  \begin{enumerate}[label=(\alph*),wide, labelwidth=!, labelindent=0pt]
  \item
    \label{item:bound_h}
    First, we show that \eqref{eq:lip_h} holds. Note that by
    the triangle inequality and \eqref{eq:h_supp_vrai}, we only need
    to consider $h \leftarrow \tilde{h}$ and $h \leftarrow V$. The
    case $h \leftarrow V$ is straightforward using \eqref{eq:V_ineq}.
    We now deal with the first case. For any $w_1, w_2 \in \rset^p$
    and $\mu_1, \mu_2 \in \MRp$, consider the decomposition,
    \begin{equation}
      \label{eq:decomp_lip_h}
      \normLigne{\tilde{h}(w_1, \mu_1) - \tilde{h}(w_2, \mu_2)} \leq  \normLigne{\tilde{h}(w_1, \mu_1)
        - \tilde{h}(w_2, \mu_1)} + \normLigne{\tilde{h}(w_2, \mu_1) - \tilde{h}(w_2, \mu_2)} \eqsp.
    \end{equation}
    In what follows, we bound separately the two terms in the
    right-hand side.  Using \rref{assum:all}-\ref{item:assuml},
    \rref{assum:all}-\ref{item:assumf}, \eqref{eq:h_supp_vrai} and \eqref{eq:majo_partial_ell} we have for any
    $w_1, w_2 \in \rset^p$ and $\mu_1 \in \MRp$
  \begin{align}
    \normLigne{\tilde{h}(w_1, \mu_1) - \tilde{h}(w_2, \mu_1)} &\leq \left\| \int_{\msx \times \msy}  \partial_1 \ell(\mu_1[F( \cdot , x)] ,y) \nablaw F(w_1,x) \, \dd \datap(x,y) \right. \\  & \qquad   \left. - \int_{\msx \times \msy}  \partial_1 \ell(\mu_1[F( \cdot , x)] ,y) \nablaw F(w_2,x) \, \dd \datap(x,y) \right\|  \\
    &  \leq \int_{\msx \times \msy}  \abs{\partial_1 \ell(\mu_1[F(\cdot, x)], y)} \Phi(x) \rmd \datap(x,y) \norm{w_1 - w_2} \\
    &    \leq   \int_{\msx \times \msy} \Psi(y)\Phi(x) \parenthese{1 + \abs{\mu_1[F(\cdot, x)]}}\dd\datap(x,y)\norm{w_1 - w_2}  \\
    &   \leq  2 \int_{\msx \times \msy} \Psi(y)\Phi^2(x)\dd\datap(x,y)\norm{w_1 - w_2}  \eqsp .     \label{eq:lip_weights}
  \end{align}
  Using \rref{assum:all}-\ref{item:assuml},
  \rref{assum:all}-\ref{item:assumf}, \eqref{eq:h_supp_vrai} and the
  Cauchy-Schwarz inequality, we also have for any $w_1 \in \rset^p$
  and $\mu_1, \mu_2 \in \MRp$
  \begin{align}
    &\normLigne{\tilde{h}(\mu_1, w_1) - \tilde{h}(\mu_2, w_1)} \\
&  \qquad  \leq \left\| \int_{\msx \times \msy}  \{\partial_1 \ell(\mu_1[F( \cdot , x)] ,y) \nablaw F(w_1,x) -\partial_1 \ell(\mu_2[F( \cdot , x)] ,y) \nablaw F(w_1,x) \}\, \dd \datap(x,y) \right\| \\
    & \qquad \leq \int_{\msx \times \msy} \abs{\partial_1 \ell(\mu_1[F(\cdot, x)], y) - \partial_1\ell(\mu_2[F(\cdot, x)],y)} \norm{\nabla_w F(w_1, x)} \rmd \datap(x,y) \\
    & \qquad \leq \int_{\msx \times \msy} \Psi(y) \norm{\mu_1[F(\cdot,x)] - \mu_2[F(\cdot,x)]} \Phi(x) \rmd \datap(x,y) \\
    & \qquad \leq \parenthese{\int_{\msx \times \msy} \Psi^2(y) \Phi^2(x) \rmd \datap(x,y)}^{1/2} \parenthese{\int_{\msx} \norm{\mu_1[F(\cdot, x)] - \mu_2[F(\cdot, x)]}^2 \rmd \datap(x)}^{1/2}  \eqsp.    \label{eq:lip_emp}
  \end{align}
  Combining \eqref{eq:G_def}, \eqref{eq:lip_weights}, \eqref{eq:lip_emp}, the fact that
  for any $a,b \geq 0$, $2ab \leq a^2 + b^2$ and
  \Cref{assum:all}-\ref{item:compact}, we obtain that there exists
  $\Lip_1 \geq 0$ such that for any $\mu_1, \mu_2 \in \MRp$ and
  $w_1, w_2 \in \rset^p$ we have
  \begin{multline}
    \normLigne{\tilde{h}(w_1,\mu_1) - \tilde{h}(w_2, \mu_2)} \\ \leq
    \Lip_1 \defEns{\norm{w_1 - w_2} +
      \parenthese{\int_{\msx\times\msy}\norm{\mu_1[G(\cdot, x,y)] -
          \mu_2[G(\cdot, x,y)]}^2 \rmd \datap(x,y)}^{1/2}} \eqsp .
  \end{multline}
  In addition, using \Cref{assum:all}-\ref{item:assumf} and
  \eqref{eq:majo_partial_ell}, we have for any $w \in \rset^p$,
  $\mu \in \MRp$, $x \in \msx$ and $y \in \msy$
  \begin{equation}
    \abs{\partial_1 \ell(\mu[F(\cdot, x)], y)} \norm{\nabla_w F(w, x)} \leq   \Psi(y)\Phi(x)(1 + \Phi(x)) \leq 2 \Psi(y)\Phi^2(x) \eqsp .
                                                                             \label{eq:majo_xy}
  \end{equation}
  Therefore, combining this result and \eqref{eq:h_supp_vrai}, we get
  that for any $w \in \rset^p$ and $\mu \in \MRp$
  \begin{equation}
  	\label{eq:boundh}
    \normLigne{\tilde{h}(w,\mu)} \leq  \int_{\msx \times \msy} 2 \Psi(y)\Phi^2(x) \dd \datap(x,y)\eqsp .
  \end{equation}
  Using the fact that for any $a,b \geq 0$, $2ab \leq a^2 + b^2$ and
  \Cref{assum:all}-\ref{item:compact}, there exists $\Lip_2 \geq 0$ such that
  for any $w \in \rset^p$ and $\mu \in \MRp$,
  \begin{equation}
    \label{eq:tilde_h_bound}
   \normLigne{\tilde{h}(w, \mu)} \leq \Lip_2
  \end{equation}
\item Second, we first show that there exists $\Lip_3 \geq 0$ such
  that for any $\mu \in \Pens(\rset^p)$, $w \in \rset^p$ and
  $i, j \in \{1, \dots, p\}$,
  $\absLigne{\rmS_{i,j}(w, \mu)} \leq \Lip$. Let
  $i,j \in \{1, \dots, p\}$. We have for any $w \in \rset^p$ and
  $\mu \in \MRp$
  \begin{equation}
    \label{eq:majo_beta_uno}
    \abs{\rmS_{i,j}(w, \mu)} \leq \| \rmS(w, \mu) \| \leq \trace^{1/2}\parenthese{\Sigma(w, \mu)}\eqsp .
  \end{equation}
  Similarly to \eqref{eq:majo_xy}, using \eqref{eq:h_supp_vrai}, \eqref{eq:tilde_h_bound}, the
  fact that for any $a,b \geq 0$, $(a+b)^2 \leq 2 (a^2 + b^2)$ and the
  Cauchy-Schwarz inequality, we get for any $w \in \rset^p$ and
  $\mu \in \MRp$
  \begin{equation}
    \label{eq:majo_beta_duo}
    \trace\parenthese{\Sigma(w, \mu)} \leq  \int_{\msx \times \msy}  \norm{\xi(w,\mu, x, y)}^2 \rmd \datap(x,y) \leq 2 \int_{\msx \times \msy} \defEnsLigne{\Lip_2^2 + 2 \Psi^2(y)\Phi^4(x)} \rmd \datap(x,y) \eqsp .
  \end{equation}
  Combining \eqref{eq:majo_beta_uno}, \eqref{eq:majo_beta_duo} and
  \Cref{assum:all}-\ref{item:compact}, there exists $\Lip_3 \geq 0$
  such that for any $w \in \rset^p$ and $\mu \in \MRp$, $\max_{1\leq i,j\leq p}\absLigne{\rmS_{i,j}(w, \mu)} \leq \Lip_3$.

  We now show that \eqref{eq:lip_sig} holds.  For any
  $w_1, w_2 \in \rset^p$, $\mu_1, \mu_2 \in \MRp$ define
  $\funsig: \ccint{0,1} \to \mathbb{S}_p(\rset)$ for any
  $t \in \ccint{0,1}$ by
  \begin{equation}
    \label{eq:g_def}
    \funsig(t) = \Sigma(t w_1 + (1-t) w_2, t \mu_1 + (1-t) \mu_2) \eqsp .
  \end{equation}
  For ease of notation, the dependency of $\funsig$ with respect to
  $w_1, w_2 \in \rset^p$ and $\mu_1, \mu_2 \in \Pens(\rset^p)$ is
  omitted.  In what follows, we show that for any
  $w_1, w_2 \in \rset^p$, $\mu_1, \mu_2 \in \MRp$,
  $\funsig \in \rmc^2(\ccint{0,1}, \mathbb{S}_p(\rset))$ and that there
  exists $\Lip_4 \geq 0$ such that for any $t \in \ccint{0,1}$
  \begin{equation}
    \label{eq:g_seconde_bound}
    \norm{\funsig''(t)} \leq \Lip_4 \defEns{\norm{w_1 - w_2} + \parenthese{\int_{\msx \times \msy}\norm{\mu_1[G(\cdot, x,y)] - \mu_2[G(\cdot, x,y)]}^2 \rmd \datap(x,y)}^{1/2}}^2 \eqsp ,
  \end{equation}
  which will conclude the proof of \eqref{eq:lip_sig} upon using a
  straightforward adaptation of \cite[Lemma 3.2.3, Theorem
  5.2.3]{stroock2006multi}. We conclude the proof of \Cref{prop:h_reg}
  upon letting $\Lip = \max(\Lip_1, \Lip_2, \Lip_3, \Lip_4)$.

  For any $t \in \ccint{0,1}$, let
  $\mu_t = \mu_1 + t(\mu_2 - \mu_1) \in \MRp$ and
  $w_t = w_1 + t(w_2 - w_1) \in \rset^p$ and for any
  $(x,y) \in \msx \times \msy$ define
  \begin{equation}
    \label{eq:f_definition}
    \begin{aligned}
    &\rmf (t,x,y) = \partial_1 \ell(\mu_t[F(\cdot, x)], y) \nabla_w F(w_t, x) \eqsp , \\
    &\tilde{\rmf}(t,x,y) = \xi(w_t, \mu_t, x, y) = \int_{\msx \times \msy} \rmf(t,x,y) \rmd \datap(x,y) - \rmf(t,x,y)   \eqsp .
  \end{aligned}
  \end{equation}
  The rest of the proof consists in showing that $\funsig$ is twice
  differentiable with dominated derivatives using the Lebesgue
  convergence theorem.


  By \eqref{eq:h_supp_vrai}, \eqref{eq:majo_xy} and
  \eqref{eq:tilde_h_bound}, we get that for any $w_1,w_2 \in \rset^p$,
  $\mu_1, \mu_2 \in \Pens(\rset^p)$, $(x,y) \in \msx \times \msy$
  and $t \in \ccint{0, 1}$
\begin{equation}
  \label{eq:first_bound_g}
  \norm{\rmf(t,x,y)} \leq 2\Psi(y) \Phi^2(x) \eqsp , \qquad \normLigne{\tilde{\rmf}(t,x,y)} \leq \Lip_2 + 2 \Psi(y) \Phi^2(x) \eqsp .
\end{equation}

Using \eqref{eq:f_definition}, \Cref{assum:all}-\ref{item:assuml} and
\Cref{assum:all}-\ref{item:assumf}, we have that for any
$(x,y) \in \msx \times \msy$,
$\rmf(\cdot, x, y) \in \rmc^1(\ccint{0,1}, \rset^p)$ and for any
$w_1,w_2 \in \rset^p$, $\mu_1, \mu_2 \in \Pens(\rset^p)$,
$(x,y) \in \msx \times \msy$ and $t \in \ccint{0, 1}$
\begin{multline}
  \label{eq:def_f_prime}
    \partial_1 \rmf(t,x,y) = \partial_1^2 \ell(\mu_t[F(\cdot, x)], y) \nabla_w F(w_t, x)  \parenthese{\mu_2[F(\cdot, x)] - \mu_1[F(\cdot, x)]} \\ +  \partial_{1} \ell(\mu_t[F(\cdot, x)], y) \nabla_w^2 F(w_t, x) (w_2 - w_1) \eqsp .
  \end{multline}
  Using \Cref{assum:all}-\ref{item:assuml},
  \Cref{assum:all}-\ref{item:assumf}, \eqref{eq:G_def} and \eqref{eq:majo_partial_ell}, we
  get that for any $(x,y) \in \msx \times \msy$ and
  $t \in \ccint{0, 1}$
  \begin{equation}
    \label{eq:bound_g_prime}
    \norm{\partial_1 \rmf(t,x,y)} \leq 3 \Psi(y)\Phi^2(x) \parenthese{\norm{w_2 - w_1} +\norm{\mu_1[F(\cdot, x)] - \mu_2[F(\cdot, x)]} } \eqsp ,
  \end{equation}
  Similarly, using \eqref{eq:def_f_prime},
  \Cref{assum:all}-\ref{item:assuml} and
  \Cref{assum:all}-\ref{item:assumf}, we have that for any
  $(x,y) \in \msx \times \msy$,
  $\rmf(\cdot, x, y) \in \rmc^2(\ccint{0,1}, \rset^p)$ and for any
  $w_1, w_2 \in \rset^p$, $\mu_1, \mu_2 \in \Pens(\rset^p)$,
  $(x,y) \in \msx \times \msy$ and $t \in \ccint{0,1}$
  \begin{multline}
    \partial_1^2\rmf(t,x,y) = \partial_1^3 \ell(\mu_t[F(\cdot, x)], y) \nabla_w F(w_t, x)  \parenthese{\mu_2[F(\cdot, x)] - \mu_1[F(\cdot, x)]}^2 \\ +  2\partial_{1}^2 \ell(\mu_t[F(\cdot, x)], y) \nabla_w^2 F(w_t, x) (w_2 - w_1) \parenthese{\mu_2[F(\cdot, x)] - \mu_1[F(\cdot, x)]} \\
    + \partial_{1} \ell(\mu_t[F(\cdot, x)], y) \rmD_w^3 F(w_t, x) (w_2 - w_1)^{\otimes 2} \eqsp .
  \end{multline}
  Using \Cref{assum:all}-\ref{item:assuml},
  \Cref{assum:all}-\ref{item:assumf} and \eqref{eq:majo_partial_ell}
  and that for any $a,b \geq 0$, $2ab \leq a^2 + b^2$, we get that for
  any $(x,y) \in \msx \times \msy$ and $t \in \ccint{0, 1}$
  \begin{align}
    \norm{\partial_1^2 \rmf (t,x,y)} \leq 5 \Psi(y) \Phi^2(x) \parenthese{\norm{w_2 - w_1}^2 + \norm{\mu_1[F(\cdot, x)] - \mu_2[F(\cdot, x)]}^2} \eqsp .     \label{eq:bound_g_secon}
  \end{align}
  Combining \eqref{eq:f_definition}, \eqref{eq:bound_g_prime},
  \eqref{eq:bound_g_secon}, \Cref{assum:all}-\ref{item:compact} and
  the dominated convergence theorem, we get that for any
  $(x,y) \in \msx \times \msy$,
  $\tilde{\rmf}(\cdot, x, y) \in \rmc^2(\ccint{0,1}, \rset^p)$.  In
  addition, using \eqref{eq:f_definition}, \eqref{eq:first_bound_g},
  \eqref{eq:bound_g_prime}, \eqref{eq:bound_g_secon}, the
  Cauchy-Schwarz inequality and the fact that for any $a,b \geq 0$,
  $2ab \leq a^2 + b^2$, there exists $C \geq 0$, such that for any
  $w_1, w_2 \in \rset^p$, $\mu_1, \mu_2 \in \MRp$,
  $(x,y) \in \msx \times \msy$ and $t \in \ccint{0,1}$
  \begin{align}
    \normLigne{\tilde{\rmf}(t, x, y)} &\leq C \parenthese{\Phi^4(x)+\Psi^2(y)} \eqsp ,\\
    \normLigne{\partial_1 \tilde{\rmf}(t, x, y)} &\leq C \parenthese{\Phi^4(x)+\Psi^2(y)} \chi(w_1, w_2, \mu_1, \mu_2,x) \eqsp , \\
    \normLigne{\partial_1^2 \tilde{\rmf}(t, x, y)} &\leq C \parenthese{\Phi^4(x)+\Psi^2(y)} \chi^2(w_1, w_2, \mu_1, \mu_2,x) \eqsp ,     \label{eq:h_bound}
  \end{align}
  where
  \begin{multline}
    \label{eq:chi_def}
    \chi(w_1, w_2, \mu_1, \mu_2,x) = \norm{w_1 - w_2} \\ + \norm{\mu_1[F(\cdot, x)] - \mu_2[F(\cdot, x)]} + \parenthese{\int_{\msx \times \msy} \norm{\mu_1[G(\cdot, \tilde{x}, \tilde{y})] - \mu_2[G(\cdot, \tilde{x}, \tilde{y})]}^2\rmd \pi(\tilde{x}, \tilde{y})}^{1/2} \eqsp .
  \end{multline}

  Using \eqref{eq:g_def} and \eqref{eq:h_supp_vrai}, we have that for
  any $w_1, w_2 \in \rset^p$, $\mu_1, \mu_2 \in \MRp$,
  $t \in \ccint{0, 1}$
  \begin{equation}
    \funsig(t) = \int_{\msx \times \msy} \tilde{\rmf}(t,x,y) \tilde{\rmf}(t,x,y)^{\top} \rmd \datap(x,y) \eqsp .
  \end{equation}

  Combining this result, \eqref{eq:h_bound} and
  \Cref{assum:all}-\ref{item:compact} we get that for any
  $w_1, w_2 \in \rset^p$ and $\mu_1, \mu_2 \in \MRp$,
  $\funsig \in \rmc^2(\ccint{0,1}, \mathbb{S}_p(\rset))$ and, using
  the Cauchy-Schwarz inequality, there exist $C_1, C_2 \geq 0$ such that for any
  $w_1, w_2 \in \rset^p$ and $\mu_1, \mu_2 \in \MRp$,
  $t \in \ccint{0, 1}$ and $\rmu \in \rset^p$ with
  $\normLigne{\rmu} = 1$, we have
  \begin{align}
    \langle \rmu, \funsig''(t)  \rmu \rangle &= \int_{\msx \times \msy} \partial_1^2 \parenthese{\langle \rmu, \tilde{\rmf}(t,x,y) \rangle^2} \rmd \datap(x,y) \\
                                       &\leq 2 \int_{\msx \times \msy} \normLigne{\partial_1 \tilde{\rmf}(t,x,y)}^2 \rmd \datap(x,y) + 2 \int_{\msx \times \msy} \normLigne{\partial_1^2 \tilde{\rmf}(t,x,y)} \normLigne{\tilde{\rmf}(t,x,y)} \rmd \datap(x,y) \\
                                       &\leq C_1 \int_{\msx \times \msy} \parenthese{\Phi^{8}(x)+\Psi^4(y)} \chi^2(w_1, w_2,x,y) \rmd \datap(x,y)  \\
                                       & \leq C_2 \defEns{\norm{w_1- w_2} + \parenthese{\int_{\msx} \norm{\mu_1[G(\cdot, x, y)] - \mu_2[G(\cdot, x, y)]}^2 \rmd \datap(x,y)}^{1/2}}^2 \eqsp , \label{eq:ineq_a}
  \end{align}
  Therefore, we get that for any
  $w_1, w_2 \in \rset^p$, $\mu_1, \mu_2 \in \MRp$,
  $t \in \ccint{0, 1}$
  \begin{align}
    \norm{\funsig''(t)} &= \sup_{\rmu \in \rset^p, \norm{\rmu} = 1}\langle \rmu, \funsig''(t) \rmu \rangle \\
                  &\leq C \defEns{\norm{w_1- w_2} + \parenthese{\int_{\msx} \norm{\mu_1[G(\cdot, x, y)] - \mu_2[G(\cdot, x, y)]}^2 \rmd \datap(x,y)}^{1/2}}^2 \eqsp .
  \end{align}
  Combining this result and a straightforward adaptation of \cite[Lemma 3.2.3, Theorem 5.2.3]{stroock2006multi} we
  obtain that for any $w_1, w_2 \in \rset^p$, $\mu_1, \mu_2 \in \MRp$
    \begin{equation}
      \abs{\rmS_{i,j}(w_1, \mu_1) - \rmS_{i,j}(w_2, \mu_2)} \leq \Lip_4  \defEns{\norm{w_1- w_2} + \parenthese{\int_{\msx} \norm{\mu_1[G(\cdot, x, y)] - \mu_2[G(\cdot, x, y)]}^2 \rmd \datap(x,y)}^{1/2}} ,
  \end{equation}
  with $\Lip_4 = \sqrt{2C} p$.
\item Using \eqref{eq:h_supp_vrai}, we have for any $w \in \rset^p$ and
  $\mu \in \MRp$
  \begin{equation}
    \hspace{-0.2cm} \int_{\msx \times \msy} \norm{\xi(w, \mu, x, y)}^2 \rmd \datap(x,y) = \int_{\msx \times \msy} \trace\parenthese{\xi \xi^{\top}(w, \mu, x, y)} \rmd \datap(x,y)  = \sum_{i,j=1}^p \abs{\rmS_{i,j}(w, \mu)}^2 \leq p^2 \Lip^2 \eqsp .
  \end{equation}
\end{enumerate}
\end{proof}

\section{Quantitative propagation of chaos}
\label{sec:propagation-chaos}

\subsection{Existence of strong solutions to the particle SDE}
\label{sec:quant_prop}

In this section, for two functions
$A,B : \ \bigcup_{N \in \nsets} \defEns{\{1, \dots, N\} \times \rset_+
  \times (\rset^{p})^2 \times (\Pens_2(\rset^p))^2} \to \rset$, the
notation
$A_N(k,t,w_1,w_2,\mu_1,\mu_2) \lesssim B_N(k,t,w_1,w_2,\mu_1,\mu_2)$
stands for the statement that there exists $C \geq 0$ such that for
any $N \in \nsets$, $k \in \{1, \dots, N\}$, $t \in \rset_+$,
$w_1,w_2 \in \rset^p$, $\mu_1,\mu_2 \in \Pens_2(\rset^p)$,
$A_N(k, t,w_1,w_2,\mu_1,\mu_2) \leq C B_N(k, t,w_1,w_2,\mu_1,\mu_2)$, where $A_N$ and $B_N$ are
the restrictions of $A$ and $B$ to
$\{1, \dots, N\} \times \rset_+ \times (\rset^p)^2 \times
(\Pens_2(\rset^p))^2$.

We consider for $N \in \nsets$, $p$ dimensional particle system $(\bfw^{1:N}_t)_{t \geq 0}$ associated with the SDE: for
any $k \in \{1, \dots, N\}$
\begin{equation}
  \label{eq:sde_supp}
  \rmd \bfwkN_t = b_N(t, \bfwkN_t, \bfnu_t^N) \rmd t + \upsigma_N(t, \bfwkN_t, \bfnu_t^N) \rmd \bfB_t^{k} \eqsp , \qquad \bfnu_t^N = (1/N) \sum_{k=1}^N \updelta_{\bfwkN_t} \eqsp ,
\end{equation}
where $(\bfB_t^{k})_{k \in \nsets}$ are
independent $r$-dimensional Brownian motions and where
$(b_N)_{N\in\nsets}$ and $(\upsigma_N)_{N\in \nsets}$ are family of
measurable functions such that for any $N \in \nsets$,
$b_N : \ \rset_+ \times \rset^{p} \times \Pens_2(\rset^p) \to \rset^p$
and
$\upsigma_N : \ \rset_+ \times \rset^{p} \times \Pens_2(\rset^p) \to
\rset^{p \times r}$. We make the following assumption ensuring the
existence and uniqueness of solutions of \eqref{eq:sde_supp} for any
$N \in \nsets$. Consider in the sequel a measurable space
$(\msz, \mcz)$ and a  probability measure $\pi_{\msz}$ on this space.

\begin{assumptionB}
  \label{assum:propagation}
  There exist a measurable function
  $\rmg: \ \rset^p \times \msz \to \rset$, $\Mip_1 \geq 0$
  and $\mu_0 \in \Pens_2(\rset^p)$ such that for any $N \in \nsets$,
  the following hold.
    \begin{enumerate}[label=(\alph*), wide, labelwidth=!, labelindent=0pt]
    \item For any $w_1, w_2 \in \rset^p$ and $z \in \msz$ we have
      \begin{equation}
        \norm{\rmg(w_1,z) - \rmg(w_2,z)} \leq \zeta(z)  \norm{w_1 - w_2} \eqsp , \quad \norm{\rmg(w_1,z)} \leq \zeta(z) \eqsp , \text{       with $\int_{\msz} \zeta^2(z) \rmd \pi_{\msz}(z) < +\infty$} \eqsp.
      \end{equation}
    \item
      $b_N \in \rmc( \rset_+ \times \rset^p \times \MRpdeux, \rset^p)$
      and
      $\upsigma_N \in \rmc( \rset_+ \times \rset^p \times \MRpdeux,
      \rset^{p \times r})$.
    \item For any
      $w_1,w_2 \in \rset^p$ and $ \mu_1,\mu_2 \in \Pens_2(\rset^p)$
  \begin{align}
      \label{eq:lip_ineq}
    &\textstyle{\sup_{t \geq 0}\{ \norm{b_N(t, w_1, \mu_1) - b_N(t, w_2, \mu_2)}} + \norm{\upsigma_N(t, w_1, \mu_1) - \upsigma_N(t, w_2, \mu_2)} \}\\
    \nonumber
    & \qquad \qquad \qquad  \leq \Mip_1 \defEns{\norm{w_1 - w_2} + \parenthese{\int_{\msz} \abs{\mu_1[\rmg(\cdot, z)] - \mu_2[\rmg(\cdot, z)]}^2 \rmd \pi_{\msz}(z)}^{1/2} } \eqsp,
    \\
    \nonumber
    &\textstyle{\sup_{t \geq 0}}  \defEns{\norm{b_N(t, 0, \mu_0)} + \norm{\upsigma_N(t, 0, \mu_0)}} \leq \Mip_1 \eqsp .
  \end{align}
  \end{enumerate}
\end{assumptionB}

\begin{assumptionB}
  \label{assum:limit_b_sig}
  There exist $\Mip_2 \geq0$, $\upkappa > 0$,
  $b \in \rmc( \rset_+ \times \rset^p \times \MRpdeux, \rset^p)$ and
  $\upsigma \in \rmc( \rset_+ \times \rset^p \times \MRpdeux, \rset^{p
    \times r})$ such that
  \begin{equation}
    \sup_{t\geq 0, w \in \rset^p, \mu \in \Pens_2(\rset^p)}\defEns{ \norm{b_N(t, w, \mu) - b(t, w, \mu)} + \norm{\upsigma_N(t, w, \mu) - \upsigma(t, w, \mu)}} \leq \Mip_2 N^{-\upkappa} \eqsp .
  \end{equation}
\end{assumptionB}

Note that under \Cref{assum:propagation}, we have the following
estimate which will be used in our next result,
\begin{equation}
  \label{eq:majoration_b_sig}
  \norm{b_N(t, w, \mu)} + \norm{\upsigma_N(t, w, \mu)} \lesssim  \parentheseDeux{1 + \norm{w} + \parenthese{\int_{\rset^p} (1 + \norm{\tilde{w}}^2) \rmd \mu(\tilde{w})}^{1/2}} \eqsp,
\end{equation}
\begin{multline}
\sup_{t \geq 0}\{ \norm{b_N(t, w_1, \mu_1) - b_N(t, w_2, \mu_2)} + \norm{\upsigma_N(t, w_1, \mu_1) - \upsigma_N(t, w_2, \mu_2)} \} \\  \lesssim
    \norm{w_1 - w_2} + \wassersteinD[2](\mu_1, \mu_2)  \eqsp.
\end{multline}

\begin{theorem}
  \label{thm:strong_sol_part}
  Assume \rref{assum:propagation}. Then for any $N \in \nsets$,
  \eqref{eq:sde_supp} admits a unique strong  solution. If in
  addition, there exists $m \geq 1$ such that
  $ \sup_{N \in \nsets} \sup_{k \in \{1, \dots, N\}}
  \expeLigne{\normLigne{\bfwkN_0}^{2m}} < +\infty$, then for any
  $T \geq 0$, there exists $C \geq 0$ such that
  \begin{equation}
     \sup_{N \in \nsets} \sup_{k \in \{1, \dots, N\}} \expe{\sup_{t \in \ccint{0,T}} \norm{\bfwkN_t}^{2m}} \leq C \eqsp .
  \end{equation}
\end{theorem}

\begin{proof}
  First, we show that for any $N \in \nsets$, \eqref{eq:sde_supp}
  admits a unique strong  solution.  Let
  $\tb_N: \ \rset_+ \times (\rset^p)^N \to (\rset^p)^N$ and
  $\tupsigma_N: \ \rset_+ \times (\rset^p)^N \to (\rset^{p \times r})^N$
  given, setting $\nu^{N,w} =     (1/N)\sum_{j=1}^N \updelta_{w^{j, N}}$ for any $t \geq 0$ and $w^{1:N} \in (\rset^p)^N$, by
  \begin{align}
    & \tb_N(t, w^{1:N}) = \parenthese{b_N\parenthese{t, w^{k,N}, \nu^{N,w}}}_{k \in \{1, \dots, N\}} \eqsp ,  \tupsigma_N(t, w^{1:N}) = \parenthese{\upsigma_N\parenthese{t, w^{k,N},\nu^{N,w}}}_{k \in \{1, \dots, N\}} \eqsp.
  \end{align}
  Let $w_1^{1:N}, w_2^{1:N} \in (\rset^p)^N$. Using
 \Cref{assum:propagation},   \Cref{prop:fm_bound} and that for any
  $a,b \geq 0$, $(a+b)^{1/2} \leq a^{1/2} + b^{1/2}$, we have
  \begin{align}
    &\normLigne{b_N\parentheseLigne{t, w_1^{k,N}, \nu^{N,w_1}}
      - b_N\parentheseLigne{t, w_2^{k,N}, \nu^{N,w_2}}}   \lesssim \normLigne{w_1^{k,N}- w_2^{k,N}} + \wassersteinD[2]\parentheseLigne{\nu^{N,w_1}, \nu^{N,w_2}} \\
    & \qquad \qquad \lesssim \normLigne{w_1^{k,N}- w_2^{k,N}} + \textstyle{\parentheseLigne{N^{-1}\sum_{j=1}^N \normLigne{w_1^{j, N} - w_2^{j, N}}^2}^{1/2}}  \lesssim \normLigne{w^{1:N} - w_2^{1:N}}\eqsp .
  \end{align}
  Similarly, we have
  $\normLigne{\upsigma_N\parentheseLigne{t, w_1^{k,N}, \nu^{N,w_1}} -
    \upsigma_N\parentheseLigne{t, w_2^{k,N}, \nu^{N,w_2}}} \lesssim \normLigne{w^{1:N} - w_2^{1:N}}$.  Therefore, we
  obtain that for any $N \in \nsets$, $\tb_N$ and $\tupsigma_N$ are
  Lipschitz-continuous and using \cite[Theorem
  2.9]{karatzas1991brownian}, we get that there exists a unique strong
   solution to \eqref{eq:sde_supp}.  Let $m \geq 1$ and assume
  that
  $ \sup_{N \in \nsets} \sup_{k \in \{1, \dots, N\}}
  \expeLigne{\normLigne{\bfwkN_0}^{2m}} < +\infty$, we now show that
  for any $T \geq 0$, there exists $C \geq 0$ such that
  \begin{equation}
    \sup_{t \in \ccint{0,T}} \sup_{N \in \nsets} \sup_{k \in \{1, \dots, N\}} \expe{\norm{\bfwkN_t}^{2m}} \leq C \eqsp .
  \end{equation}
  Let $V_m: \ \rset^{p} \to \rset_+$ given for any $w \in \rset^p$ by
  $V_m(w) = 1 + \norm{w}^{2m}$. For any $w \in \rset^p$ we have
  \begin{equation}
    \norm{\nabla V_m(w)} = 2 m \norm{w}^{2m-1} \eqsp , \qquad \norm{\nabla^2 V_m(w)} \leq 2 m (2 m - 1) \norm{w}^{2m - 2} \eqsp .
  \end{equation}
  Combining this result with \eqref{eq:majoration_b_sig}, the
  Cauchy-Schwarz inequality and the fact that for any $a,b \geq 0$ and
  $n_1,n_2 \in \nset$, $a^{n_1}b^{n_2} \leq a^{n_1+n_2}+b^{n_1+n_2}$, we get that
  \begin{align}
    &\abs{\langle \nabla V_m(w), b_N(t,w, \mu) \rangle} + \abs{\langle \nabla^2 V_m(w), \upsigma_N \upsigma_N^{\top}(t,w, \mu) \rangle}
    \\ & \qquad \lesssim \parentheseDeux{1 + \norm{w} + \parenthese{\int_{\rset^p} (1 + \norm{\tilde{w}}^2) \rmd \mu(\tilde{w})}^{1/2}}\norm{\nabla V_m(w)}
    \\ & \qquad \qquad + \parentheseDeux{1 + \norm{w} + \parenthese{\int_{\rset^p} (1 + \norm{\tilde{w}}^2) \rmd \mu(\tilde{w})}^{1/2}}^2\norm{\nabla^2 V_m(w)}
    \\ & \qquad \lesssim \parentheseDeux{1 + \norm{w} + \parenthese{\int_{\rset^p} (1 + \norm{\tilde{w}}^2) \rmd \mu(\tilde{w})}^{1/2}}\norm{w}^{2m-1}
    \\ & \qquad \qquad + \parentheseDeux{1 + \norm{w}^2 + \int_{\rset^p} (1 + \norm{\tilde{w}}^2) \rmd \mu(\tilde{w})}\norm{w}^{2m - 2}
    \\ & \qquad \lesssim 1 + \norm{w}^{2m} + \parenthese{\int_{\rset^p} (1 + \norm{\tilde{w}}^2) \rmd \mu(\tilde{w})}^m
     \lesssim 1 + \norm{w}^{2m} + \int_{\rset^p} (1 + \norm{\tilde{w}}^{2m}) \rmd \mu(\tilde{w}) \eqsp . \label{eq:majo_m}
  \end{align}
  Now let $\tau_n^{N} = \inf \ensembleLigne{t \geq 0}{\normLigne{\bfw_t^{k,N}} \geq n \text{ for some $k\in  \{1,\ldots,N\}$}}$.
  Using Itô's lemma, \eqref{eq:majo_m} and \eqref{eq:sde_supp}, we have
  \begin{align}
   \expe{V_m(\bfw_{t \wedge \tau_n^N}^{k,N})}& = \expe{V_m(\bfw_{0 \wedge \tau_n^N}^{k,N})} + \expe{\int_0^{t \wedge \tau_n^N}\left\langle \nabla V_m(\bfw_s^{k,N}), b_N\parenthese{s,\bfw_s^{k,N}, \nubf^{N}_s} \right\rangle  \rmd s }
                             \\ & \qquad + (1/2) \expe{\int_0^{t \wedge \tau_n^N}\left\langle \nabla^2 V_m(\bfw_s^{k,N}), \upsigma_N \upsigma_N^{\top}\parenthese{s,\bfw_s^{k,N}, \nubf_s^N} \right\rangle \rmd s } \\
    &\lesssim \expe{V_m(\bfw_{0 \wedge \tau_n^N}^{k,N})} + \expe{\int_0^{t \wedge \tau_n^N} \defEns{V_m(\bfw_s^{k,N}) + (1/N) \sum_{j=1}^N V_m(\bfw_s^{j,N})} \rmd s }
  \end{align}
  Using Fatou's lemma, since almost surely $\tau_n^N \to \plusinfty$
  as $n \to \plusinfty$, we get that
  \begin{multline}
    \expe{V_m(\bfw_t^{k,N}) + (1/N) \sum_{j=1}^N V_m(\bfw_t^{j,N})} \\ \lesssim \expe{V_m(\bfw_0^{k,N}) + (1/N) \sum_{j=1}^N V_m(\bfw_0^{j,N})} + \int_0^t \expe{V_m(\bfw_s^{k,N}) + (1/N) \sum_{j=1}^N V_m(\bfw_s^{j,N})} \rmd s \eqsp .
  \end{multline}
  Using Grönwall's lemma, we get that for any $T \geq 0$, there exists $C \geq 0$ such that
  \begin{equation}
    \sup_{t \in \ccint{0,T}} \sup_{N \in \nsets} \sup_{k \in \{1, \dots, N\}} \expe{\norm{\bfwkN_t}^{2m}} \leq C \eqsp .
  \end{equation}
  We now show that there exists $C \geq 0$ such that
  \begin{equation}
    \sup_{N \in \nsets} \sup_{k \in \{1, \dots, N\}} \expe{\sup_{t \in \ccint{0,T}} \norm{\bfwkN_t}^{2m}} \leq C \eqsp .
  \end{equation}
  Using Jensen's inequality, Burkholder-Davis-Gundy's inequality
  \cite[IV.42]{rogers2000diffusionsII}, \eqref{eq:majoration_b_sig}
  and the fact that for any $(a_j)_{j \in \{1, \dots, M\}}$ and $r \geq 1$ such
  that $a_j \geq 0$,
  $(\sum_{j=1}^Ma_j)^r \leq M^{r-1} \sum_{j=1}^Ma_j^r$ we get for any
  $m \in \nsets$
  \begin{align}
    &\expe{\sup_{t \in \ccint{0,T}} \norm{\bfW_t^{k,N}}^{2m}} \\
    &\lesssim  \expe{\sup_{t \in \ccint{0,T}} \norm{\int_0^tb_N(s, \bfW_s^{k,N}, \bfnu_s^N) \rmd s }^{2m}}  + \expe{\sup_{t \in \ccint{0,T}} \norm{\int_0^t \upsigma_N^{1/2}(s, \bfW_s^{k,N}, \bfnu_s^N) \rmd \bfB_s}^{2m}}  \\
    &\lesssim  \expe{\int_0^T \norm{b_N(s, \bfW_s^{k,N}, \bfnu_s^N)}^{2m} \rmd s } + \expe{\parenthese{\int_0^T\trace(\upsigma_N \upsigma_N^{\top}(s, \bfW_s^{k,N}, \bfnu_s^N)) \rmd s}^m} \\
    & \lesssim  \int_0^T \defEns{\expe{\norm{b_N(s, \bfW_s^{k,N}, \bfnu_s^N)}^{2m}} + \expe{\norm{\upsigma_N (s, \bfW_s^{k,N}, \bfnu_s^N)}^{2m}}} \rmd s \\
    & \lesssim  \int_0^T \defEns{1 + \expe{\norm{\bfW_s^{k,N}}^{2m}} + \expe{\int_{\rset^p} (1 + \norm{\tilde{w}}^{2m}) \rmd \bfnu_s^N(\tilde{w})} } \rmd s  \\
    & \lesssim  \int_0^T \defEns{1+ \expe{\norm{\bfW_s^{k,N}}^{2m}} + (1/N) \sum_{j=1}^N \expe{\norm{\bfW_s^{j,N}}^{2m}}} \rmd s  \\
    &\lesssim  1 + \sup_{N \in \nsets} \sup_{j \in \{1, \dots, N\}} \sup_{t \in \ccint{0,T}} \expe{\norm{\bfW_s^{j,N}}^{2m}} \eqsp ,
  \end{align}
which concludes the proof.
\end{proof}

\subsection{Existence of solutions to the mean-field SDE}

The following result is based on \cite[Theorem
1.1]{sznitman1991topics} showing, under \Cref{assum:propagation} and
\Cref{assum:limit_b_sig}, the existence of strong solutions and
pathwise uniqueness for non-homogeneous McKean-Vlasov SDE with
non-constant covariance matrix:
  \begin{equation}
    \label{eq:mean_field_sde}
    \rmd \bfw^{\star}_t = b(t, \bfw^{\star}_t, \bflambda_t^{\star}) \rmd t + \upsigma(t, \bfw^{\star}_t, \bflambda_t^{\star}) \rmd \bfB_t \eqsp ,
  \end{equation}
  where $b$ and $\upsigma$ are given in \Cref{assum:limit_b_sig} and
  where for any $t \geq 0$, $\bfw^{\star}_t$ has distribution
  $\bflambda_t^{\star} \in \Pens_2(\rset^p)$, $(\bfB_t)_{t \geq 0}$
  is a $r$ dimensional Brownian motion and $\bfW_0^{\star}$ has
  distribution $\mu_0 \in \Pens_2(\rset^p)$.

\begin{proposition}
  \label{prop:cauchy_lip}
  Assume \rref{assum:propagation} and \rref{assum:limit_b_sig}. Let
  $\mu_0 \in \MRpdeux$. Then, there exists a
  $(\mcf_t)_{t \geq 0}$-adapted continuous process
  $(\bfW^\star_t)_{t \geq 0}$ which is the unique strong solution of
  \eqref{eq:mean_field_sde} satisfying for any $T \geq 0$,
  $ \sup_{t \in \ccint{0,T}} \expeLigne{ \norm[2]{\bfW^{\star}_t}} <
  \plusinfty$.
\end{proposition}

\begin{proof}
  Let $\delta \geq 0$ and $\mu_0 \in \MRpdeux$.  Note that we only need to
  show that \eqref{eq:mean_field_sde} admits a strong solution up to
  $\delta >0$. 
  First,
  using \cite[Theorem 2.9]{karatzas1991brownian}, note that for any
  $(\bfmu_t)_{t \in \ccint{0,\delta}} \in \scrC_{2,\delta}^p$ the SDE,
  \begin{equation}
    \label{eq:sde_intermediaire}
      \rmd \bfw_t^{\bfmu} = b(t, \bfw_t^{\bfmu}, \bfmu_t) \rmd t + \upsigma(t, \bfw_t^{\bfmu}, \bfmu_t) \rmd \bfB_t \eqsp ,
    \end{equation}
    admits a unique strong solution, since for any $t \in \ccint{0,\delta}$ and $w_1, w_2 \in \rset^p$
    \begin{equation}
      \label{eq:bound_inter}
      \norm{b(t, w_1, \bfmu_t) - b(t, w_2, \bfmu_t)} + \norm{\upsigma(t, w_1, \bfmu_t)
        - \upsigma(t, w_2, \bfmu_t)} \leq \Mip_1 \norm{w_1 - w_2} \eqsp.
    \end{equation}
    In addition, $ \sup_{t \in \ccint{0,\delta}} \expeLigne{ \norm[2]{\bfW^{\bfmu}_t}} <
    \plusinfty$.

    In the rest of the proof, the strategy is to adapt the well-known
    Cauchy-Lipschitz approach using the Picard fixed point
    theorem. More precisely, we define below for $\delta > 0$ small
    enough, a contractive mapping
    $\Phibf_{\delta}: \ \scrC_{2,\delta}^p \to \scrC_{2,\delta}^p$
    such that the unique fixed point
    $(\bflambda_t^{\star})_{t \in \ccint{0, \delta}}$ is a weak solution
    of \eqref{eq:mean_field_sde}. Considering
    $(\bfw_t^{\bflambda^{\star}})_{t\in \ccint{0, \delta}}$, we obtain the
    unique strong solution of \eqref{eq:mean_field_sde} on
    $\ccint{0, \delta}$.

    Let $\delta >0$. Denote
    $(\bflambda^{\bfmu}_t)_{t \in \ccint{0, \delta}} \in
    \MRpdeux^{\ccint{0,\delta}}$ such that for any
    $t \in \ccint{0,\delta}$, $\bflambda_t^{\bfmu}$ is the
    distribution of $\bfw_t^{\bfmu}$ with initial condition $\bfW^{\star}_0$ with distribution
    $\bflambda_0^{\bfmu} = \mu_0$.  In addition, using
    \eqref{eq:def_distance_wasser}, \eqref{eq:majoration_b_sig}, \eqref{eq:bound_inter},
    \rref{assum:propagation}, \rref{assum:limit_b_sig}, the
    Cauchy-Schwarz inequality, the Itô isometry and the fact that for
    any $a,b \geq 0$, $2ab \leq a^2 + b^2$, there exists $C \geq 0$ such that for any
    $t, s \in \ccint{0, \delta}$ with $t \geq s$,
    \begin{align}
      &\wassersteinD[2](\bflambda_t^{\bfmu}, \bflambda_s^{\bfmu})^2
       \leq \expe{\norm{\bfw_t^{\bfmu} - \bfw_s^{\bfmu}}^2} \\
& \leq 2 \expe{\norm{\int_s^t b(u, \bfw_u^{\bfmu}, \bfmu_u) \rmd u }^2} + 2 \expe{\norm{\int_s^t \upsigma(u, \bfw_u^{\bfmu}, \bfmu_u) \rmd \bfB_u }^2} \\
      &   \leq 2 (t-s) \int_s^t \expe{\norm{b(u, \bfw_u^{\bfmu}, \bfmu_u)}^2} \rmd u  + 2 \int_s^t \expe{\trace(\upsigma \upsigma^{\top}(u, \bfw_u^{\bfmu}, \bfmu_u))} \rmd u  \\
      &   \leq 4 (t-s) \int_s^t \defEns{\norm{b(u, 0, \bfmu_u)}^2 + \Mip_1^2 \expe{\norm{\bfw_u^{\bfmu}}^2}} \rmd u \\
      & \qquad \qquad + 4 \int_s^t \defEns{\norm{\sigma(u, 0, \bfmu_u)}^2 + \Mip_1^2 \expe{\norm{\bfw_u^{\bfmu}}^2}} \rmd u \\
      &   \leq 4(1 + \delta)(t-s) \parentheseDeux{ \Mip_1^2 \sup_{t \in \ccint{0,\delta}} \expeLigne{ \norm[2]{\bfW^{\bfmu}_t}}  +  \sup_{t \in \ccint{0,\delta}} \defEns{\norm{b(t, 0, \bfmu_t)}^2 + \norm{\sigma(t, 0, \bfmu_t)}^2}}\\
      & \leq C(t-s) \{1+ \sup_{t \in \ccint{0,\delta}}\expeLigne{\norm[2]{\bfW^{\bfmu}_t}}\}\eqsp .
    \end{align}
    Therefore,
    $(\bflambda_t^{\bfmu})_{t \in \ccint{0,\delta}} \in
    \scrC_{2,\delta}^p$. Let
    $\Phibf_{\delta}: \ \scrC_{2,\delta}^p \to \scrC_{2,\delta}^p$
    given for any
    $(\bfmu_t)_{t \in \ccint{0,\delta}} \in \scrC_{2,\delta}^p$ by
    $\Phibf_{\delta}((\bfmu_t)_{t \in \ccint{0,\delta}}) =
    (\bflambda_t^{\bfmu})_{t \in \ccint{0,\delta}}$.  Let
    $(\bfmu_{1,t})_{t \in \ccint{0,\delta}}, (\bfmu_{2,t})_{t \in
      \ccint{0,\delta}} \in \scrC_{2,\delta}^p$, using
    \eqref{eq:def_distance_wasser}, \eqref{eq:bound_inter},
    \rref{assum:propagation}, \rref{assum:limit_b_sig}, the
    Cauchy-Schwarz inequality, the Itô isometry, the fact that for any
    $a,b \geq 0$, $2ab \leq a^2 + b^2$ and Grönwall's inequality we
    have for any $t \in \ccint{0,\delta}$
    \begin{align}
      \label{eq:gronwall_w2}
      &\expe{\norm{\bfw_t^{\bfmu_1} - \bfw_t^{\bfmu_2}}^2} \leq 2 \expe{\norm{\int_0^t \defEns{ b(s, \bfw_s^{\bfmu_1}, \bfmu_{1,s}) - b(s, \bfw_s^{\bfmu_2}, \bfmu_{2,s}) } \rmd s }^2} \\ & \qquad \qquad + 2 \expe{{\norm{\int_0^t \defEns{ \upsigma(s, \bfw_s^{\bfmu_1}, \bfmu_{1,s}) - \upsigma(s, \bfw_s^{\bfmu_2}, \bfmu_{2,s}) } \rmd \bfB_s }^2}} \\
      & \qquad  \leq 2 \delta \int_0^t \expe{\norm{b(s, \bfw_s^{\bfmu_1}, \bfmu_{1,s}) - b(s, \bfw_s^{\bfmu_2}, \bfmu_{2,s})}^2} \rmd s  \\ & \qquad \qquad + 2 \int_0^t \expe{\norm{\upsigma(s, \bfw_s^{\bfmu_1}, \bfmu_{1,s}) - \upsigma(s, \bfw_s^{\bfmu_2}, \bfmu_{2,s})}^2} \rmd s \\
      & \qquad \leq 4 \Mip_1^2 (1 + \delta) \int_0^t \defEns{\expe{\norm{\bfw_s^{\bfmu_1} - \bfw_s^{\bfmu_2}}^2} + \int_{\msz} \zeta^2(z) \rmd \pi_{\msz}(z) \wassersteinD[2]^2(\bfmu_{1,s}, \bfmu_{2,s})} \rmd s \\
      & \qquad \leq  4 \Mip_1^2 \delta (1 + \delta) \int_{\msz} \zeta^2(z) \rmd \pi_{\msz}(z)
        \dinfwass[\delta]^2(\bfmu_{1}, \bfmu_{2}) + 4 \Mip_1^2 (1 + \delta) \int_0^t \expe{\norm{\bfw_s^{\bfmu_1} - \bfw_s^{\bfmu_2}}^2} \rmd s  \\
      & \qquad \leq 4 \Mip_1^2 \delta (1 + \delta) \exp \parentheseDeux{4 \Mip_1^2 (1 + \delta) \delta \int_{\msz} \zeta^2(z) \rmd \pi_{\msz}(z)} \dinfwass[\delta]^2(\bfmu_{1}, \bfmu_{2}) \eqsp .
    \end{align}
    Using this result, we obtain that for any $(\bfmu_{1,t})_{t \in \ccint{0,\delta}}, (\bfmu_{2,t})_{t \in \ccint{0,\delta}}
    \in \rmc(\ccint{0,\delta}, \MRpdeux)$,
    \begin{align}
      \dinfwass[\delta]^2(\Phibf_\delta(\bfmu_1), \Phibf_\delta(\bfmu_2))
                    &\leq \sup_{t \in \ccint{0,\delta}} \expe{\norm{\bfw_t^{\bfmu_1} - \bfw_t^{\bfmu_2}}^2} \\
                    &\leq 4 \Mip_1^2  \delta(1 + \delta) \exp \parentheseDeux{4 \Mip_1^2 (1 + \delta) \delta \int_{\msz} \zeta^2(z) \rmd \pi_{\msz}(z)}  \dinfwass[\delta]^2(\bfmu_{1}, \bfmu_{2}) \eqsp .
    \end{align}
    Hence, for $\delta > 0$ small enough, $\Phibf_{\delta}$ is
    contractive and since $\rmc(\ccint{0,\delta}, \MRpdeux)$ is a
    complete metric space, we get, using Picard fixed point theorem,
    that there exists a unique
    $(\bflambda_t^{\star})_{t \in \ccint{0,\delta}} \in
    \rmc(\ccint{0,\delta}, \MRpdeux)$ such that,
    $\Phibf_{\delta}(\bflambda^{\star}) = \bflambda^{\star}$. For this
    $\bflambda^{\star}$, we have that
    $(\bfw_t^{\bflambda^{\star}})_{t \in \ccint{0,\delta}}$ is a
    strong solution to \eqref{eq:mean_field_sde}.  We have shown that
    \eqref{eq:mean_field_sde} admits a strong solution for any initial
    condition $\mu_0 \in \MRpdeux$.

    We now show that pathwise
    uniqueness holds for \eqref{eq:mean_field_sde}. Let
    $(\bfw_t^1)_{t \in \ccint{0,\delta}}$ and
    $(\bfw_t^2)_{t \in \ccint{0,\delta}}$ be two strong solutions of
    \eqref{eq:mean_field_sde} such that
    $\bfw_0^1 = \bfw_0^2 = w_0 \in \rset^p$. Let,
    $(\bfmu_{1,t})_{t \in \ccint{0,\delta}}$ and
    $(\bfmu_{2,t})_{t \in \ccint{0,\delta}}$ such that for any
    $t \in \ccint{0,\delta}$, $\bfmu_{1, t}$ is the distribution of
    $\bfw_t^1$ and $\bfmu_{2, t}$ the one of $\bfw_t^2$. Since
    $\Phibf_{\delta}$ admits a unique fixed point, we get that
    $\bfmu_1 = \bfmu_2$. Hence, $(\bfw_t^1)_{t \in \ccint{0,\delta}}$ and
    $(\bfw_t^2)_{t \in \ccint{0,\delta}}$ are strong solutions of
    \eqref{eq:bound_inter} with $\bfmu \leftarrow \bfmu_1 = \bfmu_2$ and
    since pathwise uniqueness holds for \eqref{eq:bound_inter}, we get
    that
    $(\bfw_t^1)_{t \in \ccint{0,\delta}} = (\bfw_t^1)_{t \in
      \ccint{0,\delta}}$.


 \end{proof}

 \subsection{Main result}
 \label{sec:main_res}

\begin{theorem}
  \label{thm:quantitative_prop_lip}
  Assume \rref{assum:propagation} and \rref{assum:limit_b_sig}. For
  any $N \in \nsets$, let $(\bfW^{1:N}_t)_{t \geq 0}$ be a strong
  solution of \eqref{eq:sde_supp} and for any $N \in \nsets$ and
  $k \in \{1, \dots, N\}$, let $(\bfw_t^{k, \star})_{t \geq 0}$ be
  a strong solution of \eqref{eq:mean_field_sde} with Brownian motion
  $(\bfB_t^{k})_{t \geq 0}$.  Assume that there exists
  $\mu_0 \in \MRpdeux$ such that for any $N \in \nsets$,
  $\bfW_0^{1:N}=\bfW_0^{\star, 1:N}$ has distribution
  $\mu_0^{\otimes N}$. Then for any $T \geq 0$, $N \in \nsets$ and
  $k \in \{1, \dots, N\}$
      \begin{multline}
        \expe{\sup_{t \in \ccint{0,T}} \norm{\bfw_t^{k,N} - \bfw_t^{k, \star}}^2} \leq 32(1+T)^2\parenthese{1 + \int_{\msz} \zeta^2(z) \rmd \pi_{\msz}(z)} \parenthese{\Mip_2^2 N^{-2\upkappa} + \Mip_1^2 N^{-1}} \\
        \times \exp \parentheseDeux{16(1+T)^2\parenthese{1 + \int_{\msz} \zeta^2(z) \rmd \pi_{\msz}(z)} \Mip_1^2} \eqsp .
        \end{multline}
    \end{theorem}

    \begin{proof}
      Let $T \geq 0$.  For any $N \in \nsets$, $t \geq 0$, let
      $\bfnu^{\star, N}_t = (1/N)\sum_{j=1}^N
      \updelta_{\bfw^{\star, j}_s}$.  Using
      \rref{assum:propagation}, \rref{assum:limit_b_sig}, Itô's
      isometry, Doob's inequality, Jensen's inequality and the fact
      that for any $a,b \geq 0$, $(a+b)^2 \leq 2(a^2 + b^2)$, we have
      for any $N \in \nsets$ and $k \in \{1, \dots, N\}$
      \begin{align}
        &\expe{\sup_{t \in \ccint{0,T}}\norm{\bfw_t^{k,N} - \bfw_t^{k, \star}}^2} \leq 2 \expe{\sup_{t \in \ccint{0,T}} \norm{\int_0^t \parenthese{b_N(s, \bfw_s^{k,N}, \bfnu_s^N) - b(s, \bfw_s^{k,\star}, \bflambda_s^{\star})} \rmd s }^2} \\ & \quad \quad + 2 \expe{\sup_{t \in \ccint{0,T}} \norm{\int_0^t \parenthese{\upsigma_N(s, \bfw_s^{k,N}, \bfnu_s^N) - \upsigma(s, \bfw_s^{k,\star}, \bflambda_s^{\star})} \rmd \bfB_s^{k,N}}^2} \\
        & \quad \leq 2 T \int_0^T \expe{ \norm{b_N(s, \bfw_s^{k,N}, \bfnu_s^N) - b(s, \bfw_s^{k,\star}, \bflambda_s^{\star})}^2} \rmd s  \\ & \quad \quad + 2 \expe{\norm{\int_0^T \parenthese{\upsigma_N(s, \bfw_s^{k,N}, \bfnu_s^N) - \upsigma(s, \bfw_s^{k,\star}, \bflambda_s^{\star})} \rmd \bfB_s^{k,N}}^2} \\
        &\quad \leq 2 (1+T) \int_0^T \left\lbrace \expe{\norm{b_N(s, \bfw_s^{k,N}, \bfnu_s^N) - b(s, \bfw_s^{k,\star}, \bflambda_s^{\star})}^2} \right.
        \\ & \quad \quad \left. + \expe{\norm{\upsigma_N(s, \bfw_s^{k,N}, \bfnu_s^N) - \upsigma(s, \bfw_s^{k,\star}, \bflambda_s^{\star})}^2} \right\rbrace \rmd s \\
        & \quad \leq 8\Mtt_2^2(1+T)^2N^{-2\upkappa} + 4 (1+T) \int_0^T \left\lbrace \expe{\norm{b(s, \bfw_s^{k,N}, \bfnu_s^N) - b(s, \bfw_s^{k,\star}, \bflambda_s^{\star})}^2} \right.
        \\ & \quad \quad \left. + \expe{\norm{\upsigma(s, \bfw_s^{k,N}, \bfnu_s^N) - \upsigma(s, \bfw_s^{k,\star}, \bflambda_s^{\star})}^2} \right\rbrace \rmd s \\
        & \quad \leq 8\Mtt_2^2(1+T)^2N^{-2\upkappa} + 8 \Mip_1^2 (1+T)
        \\ & \quad \quad  \times \int_0^T \left\lbrace \int_{\msz} \expe{\norm{\bfnu_s^N[\rmg(\cdot, z)] - \bflambda_s^{\star}[\rmg(\cdot, z)]}^2} \rmd \datap_{\msz}(z) + \expe{\norm{\bfw_s^{k,N} - \bfw_s^{k,\star}}^2}  \right\rbrace \rmd s \\
                & \quad \leq 8\Mtt_2^2(1+T)^2N^{-2\upkappa} + 16 \Mip_1^2 (1+T)
        \\ & \quad \quad  \times \int_0^T \left\lbrace \int_{\msz} \parenthese{\expe{\norm{\bfnu_s^N[\rmg(\cdot, z)] - \bfnu_s^{\star, N}[\rmg(\cdot, z)]}^2} + \expe{\norm{\bfnu_s^{\star, N}[\rmg(\cdot, z)] - \bflambda_s^{\star}[\rmg(\cdot, z)]}^2} } \rmd \datap_{\msz}(z) \right. \\
        & \qquad \qquad \qquad \left. + \expe{\norm{\bfw_s^{k,N} - \bfw_s^{k,\star}}^2}  \right\rbrace \rmd s \eqsp .
      \end{align}
      Then using the Cauchy-Schwarz's inequality, the fact that
      $\{(\bfw_t^{k,N})_{t \geq 0}\}_{k=1}^N$ are exchangeable, \ie \
      for any permutation
      $\uptau: \{1, \dots, N\} \to \{1, \dots, N\}$,
      $\{(\bfw_t^{k,N})_{t \geq 0}\}_{k=1}^N$ has the same
      distribution as $\{(\bfw_t^{\uptau(k),N})_{t \geq 0}\}_{k=1}^N$
      and $\{(\bfw_t^{k, \star})_{t \geq 0}\}_{k=1}^N$ are independent
      we have
      \begin{align}
        &\expe{\sup_{t \in \ccint{0,T}}\norm{\bfw_t^{k,N} - \bfw_t^{k, \star}}^2} \leq  8\Mtt_2^2(1+T)^2N^{-2\upkappa} + 16 \Mip_1^2 (1+T) \\
        & \quad \quad \times  \int_0^T \left\lbrace \dfrac{1}{N}\int_{\msz} \zeta^2(z) \rmd \pi_{\msz}(z) \sum_{j=1}^N\expe{\norm{\bfw_s^{j,N} - \bfw_s^{j, \star}}^2}  + \expe{\norm{\bfw_s^{k,N} - \bfw_s^{k,\star}}^2} \right.
        \\ & \quad \quad \left. + \int_{\msz} \textstyle{\expe{\norm{\dfrac{1}{N} \sum_{j=1}^N \rmg(\bfw_s^{j, \star}, z) - \int_{\rset^p}\rmg(\bar{w}, z) \rmd \bflambda_s^{\star}(\bar{w})}^2} \rmd \datap_{\msz}(z)} \right\rbrace \rmd s \\
        & \quad \leq  8\Mtt_2^2(1+T)^2N^{-2\upkappa} + 16 \Mip_1^2 (1+T) \parenthese{1 + \int_{\msz} \zeta^2(z) \rmd \pi_{\msz}(z)} \int_0^T \expe{\norm{\bfw_s^{k,N} - \bfw_s^{k,\star}}^2} \rmd s
        \\ & \quad \quad + 16 \Mip_1^2 (1+T) N^{-1} \int_0^T \int_{\msz} \textstyle{\expe{\norm{\rmg(\bfw_s^{k,\star}, z) - \int_{\rset^p}\rmg(\bar{w}, z) \rmd \bflambda_s^{\star}(\bar{w})}^2}} \rmd \datap_{\msz}(z) \rmd s \\
        & \quad \leq 8\Mtt_2^2(1+T)^2N^{-2\upkappa} + 16 \Mip_1^2 (1+T) \parenthese{1 + \int_{\msz} \zeta^2(z) \rmd \pi_{\msz}(z)}\int_0^T \expe{\norm{\bfw_s^{k,N} - \bfw_s^{k,\star} }^2} \rmd s \\
        & \quad \quad + 32 \Mip_1^2 (1+T)^2 N^{-1} \parenthese{1 + \int_{\msz} \zeta^2(z) \rmd \pi_{\msz}(z)} \eqsp .
      \end{align}
      We conclude the proof upon combining this result and Grönwall's inequality.
    \end{proof}

    \subsection{Proofs of  the main results}
    In this section we prove \Cref{thm:empi_conv_cont}, \Cref{thm:empi_conv_cont_one}, \Cref{thm:empi_conv_cont_sgld}, \Cref{thm:empi_conv_cont_one_sgld}. Note that we only need to show
    \Cref{thm:empi_conv_cont_sgld} and
    \Cref{thm:empi_conv_cont_one_sgld}, since in the case $\eta = 0$,
    \Cref{thm:empi_conv_cont_sgld} boils down to
    \Cref{thm:empi_conv_cont} and \Cref{thm:empi_conv_cont_one_sgld}
    to \Cref{thm:empi_conv_cont_one}.

    \label{sec:proofs-crefthm:-cref}
    \begin{proof}[Proof of \Cref{thm:empi_conv_cont_sgld}]
      Define for any $N \in \nsets$, $w \in \rset^p$,
      $\mu \in \Pens_2(\rset^p)$ and $t \geq 0$
      \begin{equation}
        \begin{aligned}
          &b_N(t, w, \mu) = (t+1)^{-\alpha} h(w, \mu) \eqsp , \upsigma_N(t, w, \mu) = (t+1)^{-\alpha} \parentheseLigne{(\gua/M)^{1/2}\Sigma^{1/2}(w, \mu), \sqrt{2} \Id} \eqsp , \\
          &b(t, w, \mu) = (t+1)^{-\alpha} h(w, \mu) \eqsp , \quad  \upsigma(t, w, \mu) = (t+1)^{-\alpha} \parentheseLigne{0, \sqrt{2} \Id } \eqsp ,
        \end{aligned}
      \end{equation}
      with $h$ and $\Sigma$ given in \eqref{eq:h_supp_vrai}.  Using
      \Cref{prop:h_reg}, we get that \Cref{assum:propagation} holds
      with $\Mip_1 \leftarrow \Lip$ and
      $\gua = \gamma^{1/(1-\alpha)} N^{(\beta-1)/(1-\alpha)}$.  In
      addition, using \Cref{prop:h_reg}, \Cref{assum:limit_b_sig}
      holds with $\Mtt_2 \leftarrow (\gamma^{1-\alpha}/M)^{1/2}p\Lip$
      and $2\upkappa = (1-\beta) / (1-\alpha)$. We conclude using
      \Cref{thm:quantitative_prop_lip}.
    \end{proof}

    \begin{proof}[Proof of \Cref{thm:empi_conv_cont_one_sgld}]
      Define for any $N \in \nsets$, $w \in \rset^p$,
      $\mu \in \Pens_2(\rset^p)$ and $t \geq 0$
            \begin{equation}
        \begin{aligned}
          &b_N(t, w, \mu) = (t+1)^{-\alpha} h(w, \mu) \eqsp , \ \upsigma_N(t, w, \mu) = (t+1)^{-\alpha} \parentheseLigne{(\gamma^{1/(1-\alpha)}/M)^{1/2}\Sigma^{1/2}(w, \mu), \sqrt{2} \Id} \eqsp , \\
        \end{aligned}
      \end{equation}
      with $h$ and $\Sigma$ given in \eqref{eq:h_supp_vrai}.  Using
      \Cref{prop:h_reg}, we get that \Cref{assum:propagation} holds
      with $\Mip_1 \leftarrow \Lip$. In  addition,
      \Cref{assum:limit_b_sig} holds with $b = b_N$,
      $\upsigma = \upsigma_N$, $\Mtt_2 \leftarrow 0$ and
      $\upkappa = 0$. We conclude using
      \Cref{thm:quantitative_prop_lip}.
    \end{proof}

    \begin{proof}[Proof of \Cref{prop:cv_w2_empi}]
  We consider only the case $\beta =1$, the proof for
  $\beta \in \coint{0,1}$  following the same lines.  Let
 $M \in
  \nsets$. We have for any $N \in \nsets$ using \Cref{prop:fm_bound},
  \begin{align}
    &\wassersteinD[2](\bfUpsilon^N, \updelta_{\bflambda^{\star}})^2 \leq \expe{\wassersteinD[2](\bfnu^N, \bflambda^{\star})^2}  \\
    & \quad \leq N^{-1} \sum_{k=1}^N \expe{\wassersteinD[2](\updelta_{(\bfw_t^{k,N})_{t \geq 0}}, \bflambda^{\star})^2}  \leq N^{-1} \sum_{k=1}^N \expe{\dist^2((\bfw_t^{k,N})_{t \geq 0}, (\bfw_t^{k,\star})_{t \geq 0})} \eqsp .
                                                                     \label{eq:majo_wass}
  \end{align}
  Let $\vareps >0$ and $n_0$ such that
  $\sum_{n =n_0+1}^{+\infty} 2^{-n} \leq \vareps$.  Combining
  \eqref{eq:majo_wass}, \Cref{thm:empi_conv_cont} and the
  Cauchy-Schwarz inequality we get that for any $N \in \nsets$
  \begin{equation}
    \wassersteinD[2](\bfUpsilon^N, \updelta_{\bflambda^{\star}})^2 \leq 2 \vareps^2 + \frac{2n_0}{N} \sum_{k=1}^N \sum_{n=1}^{n_0}\expe{\sup_{t \in \ccint{0,n}}
      \normLigne{\bfw_t^{k,N} - \bfw_t^{k, \star}}^2}
  \leq 2 \vareps^2 + 2 n_0N^{-1} \sum_{n=0}^{n_0} C_{1,n} \eqsp .
  \end{equation}
  Therefore, for any $\vareps > 0$ there exists $N_0 \in \nsets$ such
  that for any $N \in \nsets$ with $N \geq N_0$,
  $\wassersteinD[2](\bfUpsilon^N, \updelta_{\bflambda^{\star}}) \leq
  \vareps$, which concludes the proof.
\end{proof}






\section{Existence of invariant measure in the one-dimensional case}
\label{sec:exist-invar-meas}

In this section we prove \Cref{prop:existence_inv}.

\begin{proof}[Proof of \Cref{prop:existence_inv}]
  Since $V$ is $\eta$-strongly convex it admits a unique minimum at
  $w_0 \in \rset$. Using \Cref{assum:all}-\ref{item:assumV}, the fact that $V$ is
  $\eta$-strongly convex and \cite[Theorem 2.1.5, Theorem
  2.1.7]{nesterov2004introductory} there exists $\Mip \geq 0$ such that
  for any $w \in \rset$ we have
  \begin{equation}
    \label{eq:V_bounds}
    \eta (w - w_0)^2/2 \leq V(w) - V(w_0) \leq \Mip (w-w_0)^2/2 \eqsp .
  \end{equation}
  In addition, using \Cref{prop:h_reg}, we have for any
  $\mu \in \Pens_2(\rset)$ and $w \in \rset$,
  \begin{equation}
    \label{eq:Sigma_bounds}
    \bar{\upsigma}^2 \leq \Sigma(w, \mu) \leq \Lip^2 \eqsp .
  \end{equation}
  Recall that for any $\mu \in \Pens_2(\rset)$ and $w \in \rset$,
  $h(w, \mu) = \bar{h}(w, \mu) + V'(w)$, with $\bar{h}$ given in
  \eqref{eq:h_supp_vrai}. Note that for any
  $w \in \coint{w_0, +\infty}$, $V'(w) \geq 0$ and for any
  $w \in \ocint{-\infty, w_0}$, $V'(w) \leq 0$.  Combining this
  result, \Cref{prop:h_reg}, \eqref{eq:V_bounds} and
  \eqref{eq:Sigma_bounds}, there exists $\mtt_1 >0$ and
  $c_1 \in \rset$ such that for any $\mu \in \Pens_2(\rset)$ and
  $w \in \rset$, we have distinguishing the case $w \leq w_0$ and $w >w_0$,
  \begin{align}
    &\int_0^w \{h / \Sigma \}(\tilde{w}, \mu)\rmd w \geq - \bar{\upsigma}^{-2} \Lip^2 \abs{w} + \int_0^w V'(\tilde{w})/\Sigma(\tilde{w}, \mu) \rmd \tilde{w} \\
                                                               &\qquad \geq - \bar{\upsigma}^{-2} \Lip^2 \abs{w} - \bar{\upsigma}^{-2}\sup_{\tilde{w} \in \ccint{0,w_0}} \abs{V'(\tilde{w})}\abs{w_0} + \int_{w_0}^w V'(\tilde{w})/\Sigma(\tilde{w}, \mu) \rmd \tilde{w} \\
    &\qquad\geq - \bar{\upsigma}^{-2} \Lip^2 \abs{w} - \bar{\upsigma}^{-2}\sup_{\tilde{w} \in \ccint{0,w_0}} \abs{V'(\tilde{w})}\abs{w_0} + (V(w) - V(w_0))\Lip^{-2} \geq \mtt_1 w^2 + c_1 \eqsp . \label{eq:lower_bound_int}
  \end{align}
  Therefore, we obtain that for any $\mu \in \Pens_2(\rset)$, $\int_{\rset} \exp[\int_0^w h(\tilde{w}, \mu) / \Sigma(\tilde{w}, \mu) \rmd \tilde{w} ] \rmd w < +\infty$. Define $H: \ \Pens_2(\rset) \to \Pens_2(\rset)$ such that for any $\mu \in \Pens_2(\rset)$, $H(\mu)$ is the probability measure with density $\rho_{\mu}$ given for any $w \in \rset$ by
  \begin{equation}
    \label{eq:def_potential}
    \rho_{\mu}(w)   \propto \bar{\Sigma}^{-1}(w, \mu) \exp \parentheseDeux{- 2\int_0^{w} h(\tilde{w}, \mu) / \bar{\Sigma}(\tilde{w}, \mu) \rmd \tilde{w}} \eqsp ,
  \end{equation}
  where
  $\bar{\Sigma}(w, \mu) = \gamma^{1/(1-\alpha)} \Sigma(w, \mu) / M$.
  Similarly to \eqref{eq:lower_bound_int}, there exist $\mtt_2 >0$ and
  $c_2 \in \rset$ such that for any $\mu \in \Pens_2(\rset)$ and $w \in \rset$
  \begin{equation}
    \int_0^w h(\tilde{w}, \mu) / \Sigma(\tilde{w}, \mu) \rmd \tilde{w} \leq \mtt_2 w^2 + c_2 \eqsp . \label{eq:upper_bound_int}
  \end{equation}
  Combining \eqref{eq:Sigma_bounds}, \eqref{eq:lower_bound_int} and
  \eqref{eq:upper_bound_int}, there exists $\mtt >0$ and $c \in \rset$
  such that for any $\mu \in \Pens_2(\rset)$ and $w \in \rset$,
  $\rho_{\mu}(w) \leq c \rme^{-\mtt w^2}$.  Using this result, we get that
  $\sup_{\mu \in \Pens_2(\rset)} \int_{\rset} w^4 \rho_{\mu}(w) \rmd w
  < +\infty$.  Therefore, using \cite[Theorem 2.7]{ambrosio2013user}
  we obtain that $H(\Pens_2(\rset))$ is relatively compact in
  $(\Pens_2(\rset), \wassersteinD[2])$.

  We now show that $H \in \rmc(\Pens_2(\rset), \Pens_2(\rset))$. Let
  $\mu \in \Pens_2(\rset)$ and
  $(\mu_n)_{n \in \nset} \in \Pens_2(\rset)^{\nset}$ such that
  $\lim_{n \to +\infty} \mu_n = \mu$.  Using \Cref{prop:h_reg} and the
  Lebesgue dominated convergence theorem we obtain that for any
  $w \in \rset$,
  $\lim_{n \to +\infty} \rho_{\mu_n}(w) = \rho_{\mu}(w)$. Using
  Scheffé's lemma we get that
  $\lim_{n \to +\infty} \int_{\rset} \abs{\rho_{\mu_n}(w) -
    \rho_{\mu}(w)} \rmd w = 0$. Hence, $(H(\mu_n))_{n \in \nset}$
  weakly converges towards $H(\mu)$. 

  Let $(H(\mu_{n_k}))_{k \in \nset}$ be a converging sequence in
  $(\Pens_2(\rset), \wassersteinD[2])$.  Therefore,
  $(H(\mu_{n_k}))_{k \in \nset}$ also weakly converges and we obtain
  that
  $\lim_{k \to +\infty} \wassersteinD[2](H(\mu_{n_k}), H(\mu)) = 0$.
  Since $\ensembleLigne{H(\mu_n)}{n \in \nset}$ is relatively compact
  and admits a unique limit point we obtain that
  $\lim_{n \to +\infty} \wassersteinD[2](H(\mu_n), H(\mu)) = 0$.

  Hence $H \in \rmc(\Pens_2(\rset), \Pens_2(\rset))$.  Therefore,
  since $H \in \rmc(\Pens_2(\rset), \Pens_2(\rset))$ and
  $H(\Pens_2(\rset))$ is relatively compact in $\Pens_2(\rset)$
  Schauder's theorem \cite[Appendix]{bonsall1962lectures} implies that
  $H$ admits a fixed point.

  Let $\mu \in \Pens_2(\rset)$ be a fixed point of $H$. We now show
  that $\mu$ is an invariant probability distribution for
  \eqref{eq:mean_field_beta_one}. Let $(\bfw_t^{\mu})_{t \geq 0}$
  such that $\bfw_0^{\mu}$ has distribution $\mu$ and strong
  solution to the following SDE
  \begin{equation}
    \label{eq:langevin}
    \rmd \bfw_t^{\mu} = h(t, \mu) \rmd t + \gamma^{1/(1-\alpha)} \Sigma(\bfw_t^{\mu}, \mu) \rmd \bfB_t \eqsp .
  \end{equation}
  An invariant distribution for \eqref{eq:langevin} is given by
  $H(\mu)$, see \cite{kent:1978}.  Hence, since $\mu = H(\mu)$, for
  any $t \geq 0$, $\bfw_t^{\mu}$ has distribution $\mu$ and
  $(\bfw_t^{\mu})_{t \geq 0}$ is a strong solution to
  \eqref{eq:mean_field_beta_one}. Therefore, $\mu$ is an invariant
  probability measure for \eqref{eq:mean_field_beta_one} which
  concludes the proof.
\end{proof}


\section{Links with gradient flow approach}
\label{sec:gradient_flows}
\paragraph{Case $\beta \in\coint{0,1}$}
We now focus on the mean-field distribution $\bflambda^{\star}$.  Note
that the trajectories of $(\bfw_t^{k, \star})_{t \geq 0}$ for any
$k \in \nsets$ are deterministic conditionally to $\bfw_0^{k,
  \star}$. Using Itô's formula, we obtain that for any function
$f \in \rmc^2(\rset^p)$ with compact support  and $t \geq 0$
\begin{equation}
  \label{eq:evolution_equation_beta_min}
  \int_{\rset^p} f(\tilde{w}) \rmd \bflambda_t^{\star}(\tilde{w}) = \int_{\rset^p} f(\tilde{w}) \rmd \mu_0(\tilde{w}) + \int_0^t \int_{\rset^p}  (s+1)^{-\alpha} \langle h(\tilde{w}, \bflambda_s^{\star}), \nabla f(\tilde{w}) \rangle \rmd \bflambda_s^{\star}(\tilde{w}) \eqsp .
\end{equation}
Therefore, if for any $t \geq 0$, $\bflambda_t^{\star}$ admits a
density $\bfrho^{\star}_t$ such that
$(\bfrho^{\star}_t)_{t \geq 0} \in \rmc^1(\rset_+ \times \rset^p,
\rset)$ we obtain that $(\bfrho_t)_{t \geq 0}$ satisfies the following evolution equation for any $t > 0$ and $w \in \rset^p$
\begin{equation}
  \label{eq:evolution_equation}
  \partial_t \bfrho_t^{\star}(w) = - (t+1)^{-\alpha} \mathrm{div}(\bar{h}(\cdot, \bfrho_t^{\star}) \bfrho_t^{\star})(w)\eqsp ,
\end{equation}
with for any $w \in \rset^p$ and $\mu \in \Pens(\rset^p)$ with density
$\rho$, $h(w, \mu) = \bar{h}(w, \rho)$.  In the case $\alpha = 0$, it
is well-known, see
\cite{chizat2018global,mei2018mean,sirignano2018mean}, that
$(\bfrho_t^{\star})_{ t\geq0}$ is a Wasserstein gradient flow for
the functional $\risk^{\star}: \ \Pens_2^c(\rset^p) \to \rset$ given for any
$\rho \in \Pens_2^c(\rset^p)$
\begin{equation}
  \label{eq!U_star}
  \risk^{\star}(\rho) = \int_{\msx \times \msy} \ell \parenthese{\int_{\rset^p} F(\tilde{w}, x) \rho(\tilde{w}) \rmd \tilde{w} , y } \rmd \datap(x,y) \eqsp ,
\end{equation}
where $\Pens_2^c(\rset^p)$ is the set of probability density satisfying $\int_{\rset^p} \normLigne{\tilde{w}}^2 \rho(\tilde{w})  \rmd \tilde{w} < +\infty$.

\paragraph{Case $\beta =1$}

Focusing on $(\bflambda_t^{\star})_{t \geq 0}$, we no longer obtain
that $(\bflambda_t^{\star})_{t \geq 0}$ is a gradient flow for
\eqref{eq!U_star}. Indeed, using Itô's formula, we have the following
evolution equation for any $f \in \rmc_c^2(\rset^p)$ and $t \geq 0$
\begin{multline}
  \label{eq:evolution_equation_beta_one}
  \int_{\rset^p} f(\tilde{w}) \rmd \bflambda_t^{\star}(\tilde{w}) = \int_{\rset^p} f(\tilde{w}) \rmd \mu_0(\tilde{w}) + \int_0^t \int_{\rset^p}  (s+1)^{-\alpha} \langle h(\tilde{w}, \bflambda_s^{\star}), \nabla f(\tilde{w}) \rangle \rmd \bflambda_s^{\star}(\tilde{w}) \\ + \int_0^t \int_{\rset^p} (s+1)^{-\alpha} \trace(\Sigma(\tilde{w}, \bflambda_s^{\star}) \nabla^2f(\tilde{w})) \rmd \tilde{w} \eqsp .
\end{multline}
We higlight that the additional term in
\eqref{eq:evolution_equation_beta_one} from
\eqref{eq:evolution_equation_beta_min} corresponds to some entropic
regularization of the risk $\risk^{\star}$. Indeed, if for any
$w \in \rset^p$ and $\mu \in \Pens(\rset^p)$, $\Sigma = \beta \Id$
then, in the case $\alpha = 0$, we obtain that
$(\bfrho_t^{\star})_{t \geq 0}$ is a gradient flow for
$\rho \mapsto U^{\star}(\rho) + \beta \mathrm{Ent}(\rho)$, where
$\mathrm{Ent}: \ \msk_2 \to \rset$ is given for any $\rho \in \msk_2$ by
\begin{equation}
  \mathrm{Ent}(\rho) = - \int_{\rset^p} \rho(x) \log(\rho(x)) \rmd x \eqsp .
\end{equation}
This second regime emphasizes that large stepsizes act as an implicit
regularization procedure for SGD.

\section{Additional Experiments}
\label{sec:addit-exper}
In this section we present additional experiments illustrating the convergence results of the empirical measures. Contrary to the main document we illustrate our results with histograms of the weights of the first and second layers of the network, with a large number of different values of the parameters $\alpha$, $\beta$ and $N$.

\paragraph{Setting.}

In order to perform the following experiments we implemented a two-layer fully connected neural network on PyTorch. The input layer has the size of the input data, \ie, $N_{\text{input}}=28\times 28$ units in the case of the MNIST dataset~\cite{mnist} and $N_{\text{input}}=32\times 32\times 3$ in the case of the CIFAR-10 dataset~\cite{krizhevsky2009learning}. We use a varying number of $N$ units in the hidden layer and the output layer has $10$ units corresponding to the $10$ possible labels of the classification tasks. We use a ReLU activation function and the cross-entropy loss.

The linear layers' weights are initialized with PyTorch default initialization function which is a uniform initialization between $-1/N_{\text{input}}^{1/2}$ and $1/N_{\text{input}}^{1/2}$. In all our experiments, if not specified, we consider an
initialization $\bfw_0^{1:N}$ with distribution
$\mu_0^{\otimes N}$ where $\mu_0$ is the uniform distribution on
$\ccint{-0.04, 0.04}$.

In order to train the network we use SGD as described in \Cref{sec:setting} with an initial learning rate of $\gamma N^{\beta}$. In the case where $\alpha>0$ we decrease this stepsize at each iteration to have a learning rate of $\gamma N^{\beta} (n+\gua\pinv)^{-\alpha}$.
All experiments on the MNIST dataset are run for a finite time horizon $T=100$ and the ones on the CIFAR-10 dataset are run for $T=10000$. The average runtime of the experiments for $N=50000$ on the MNIST dataset is one day and the experiments on the CIFAR-10 dataset run during two days. The experiments were run on a cluster of 24 CPUs with 126Go of RAM.

All the histograms represented below correspond to the first coordinate of the weights' vector.

\paragraph{Experiments.} \Cref{fig:cv_M100_alpha0_layer1} shows that the empirical distributions of the weights converge as the number of hidden units $N$ goes to infinity. Those figures illustrate also the fact that we obtain two different limiting distributions one for $\beta<1$ (represented on the 3 first figures) and one for $\beta=1$ (on the last figure). The results presented on \Cref{fig:cv_M100_alpha0_layer2} illustrate the same fact, one the second layer. This means that the results we stated in \Cref{sec:mean-field-appr} are also true for the weights of the second layer, thanks to the procedure described for example in \cite{chizat2018global}.

\begin{figure}
	\centering
	\centerfloat
	\includegraphics[width=1.2\linewidth]{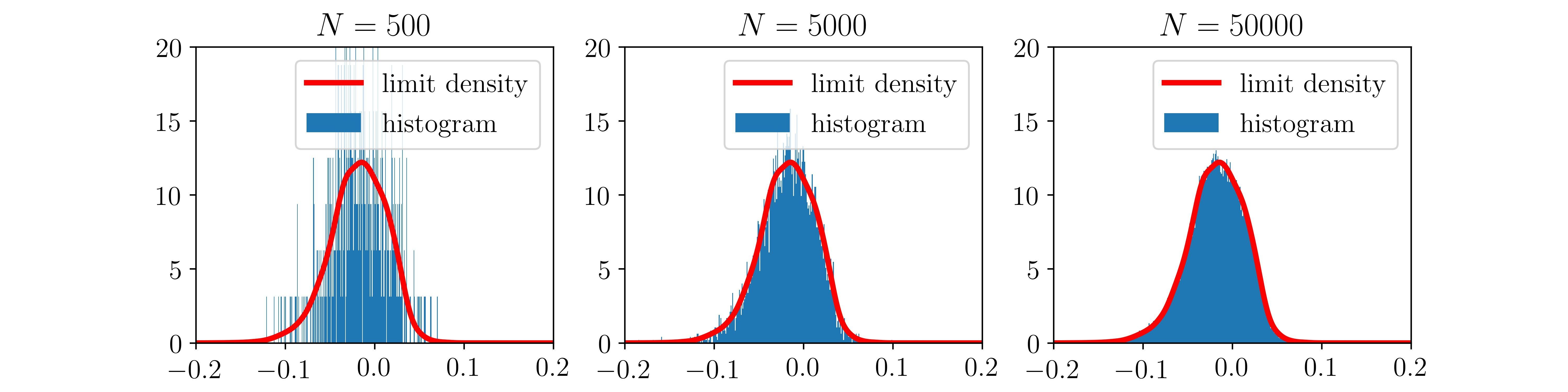}
	\includegraphics[width=1.2\linewidth]{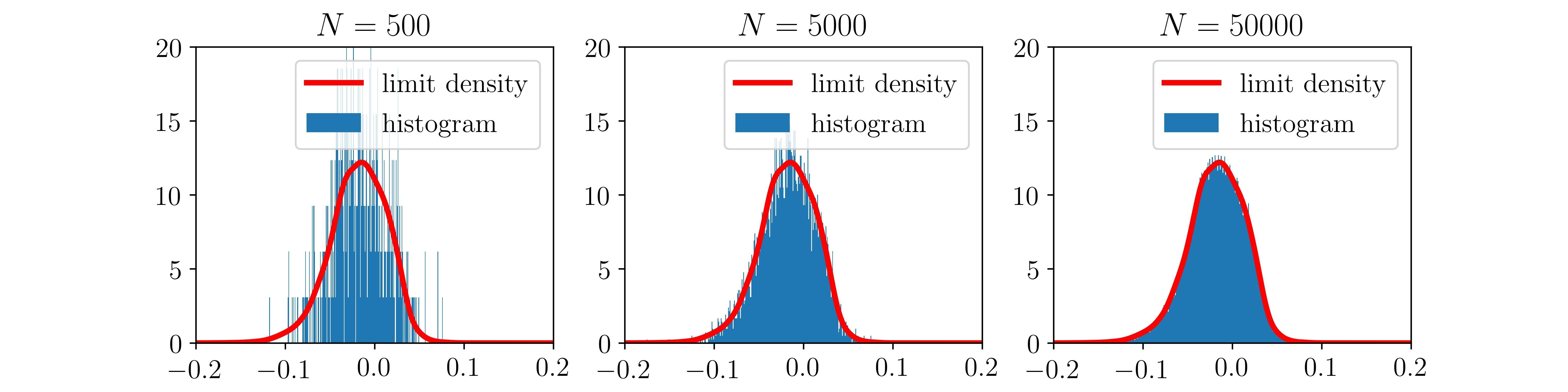}
	\includegraphics[width=1.2\linewidth]{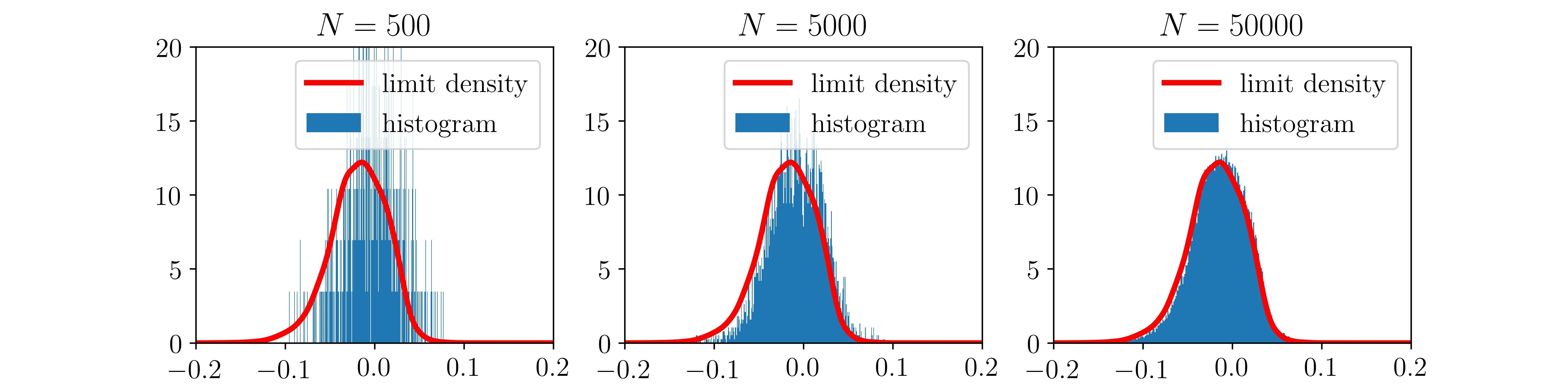}
	\includegraphics[width=1.2\linewidth]{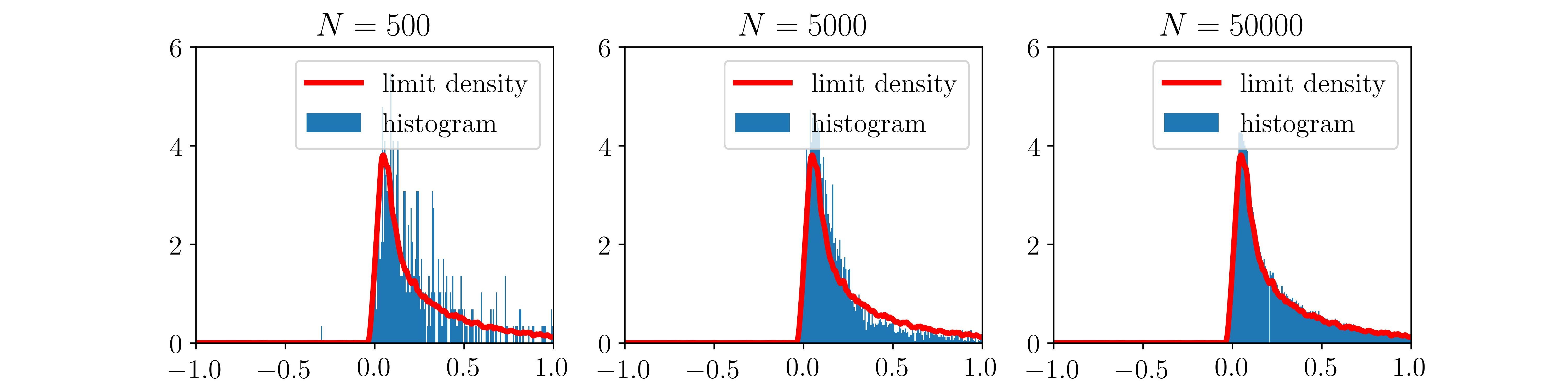}
	\caption{Convergence of the weights of the first layer as $N \to \plusinfty$ for $\alpha=0$ and $M=100$. The first line corresponds to $\beta=0.25$, the second to $\beta=0.5$, the third to $\beta=0.75$ and the last line to $\beta=1.0$.
	\label{fig:cv_M100_alpha0_layer1}}
\end{figure}

\begin{figure}
	\centering
	\centerfloat
	\includegraphics[width=1.2\linewidth]{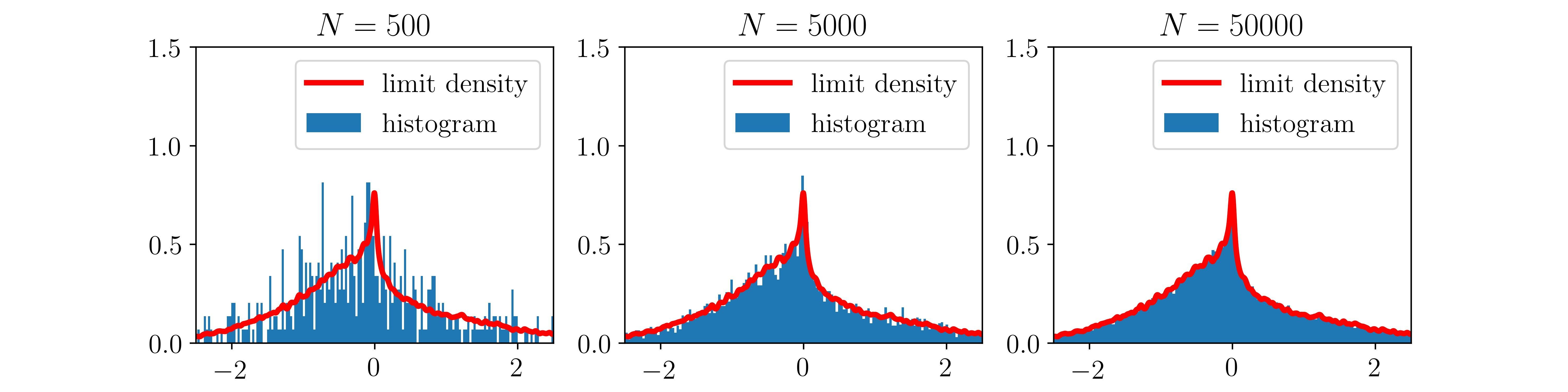}
	\includegraphics[width=1.2\linewidth]{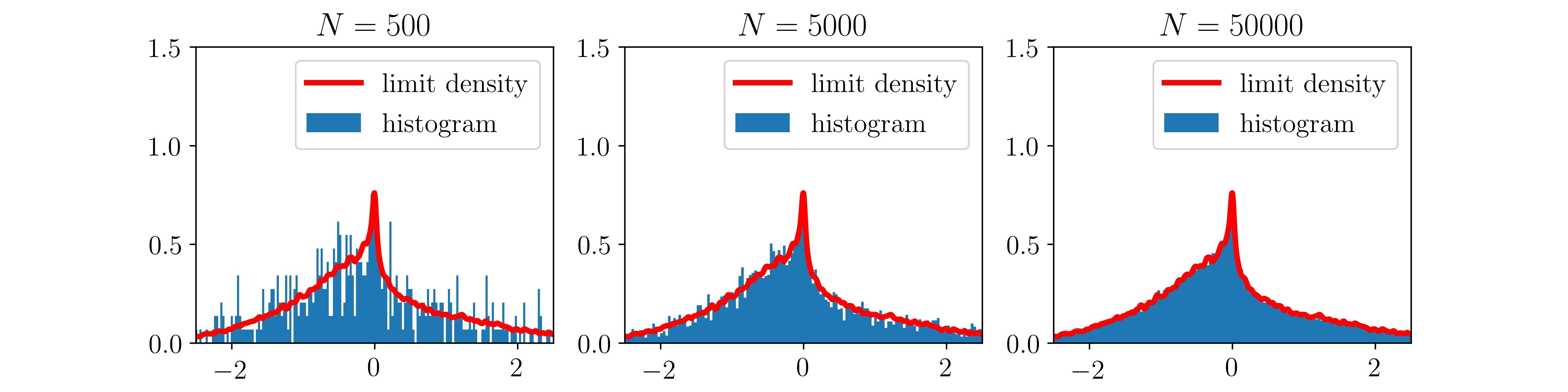}
	\includegraphics[width=1.2\linewidth]{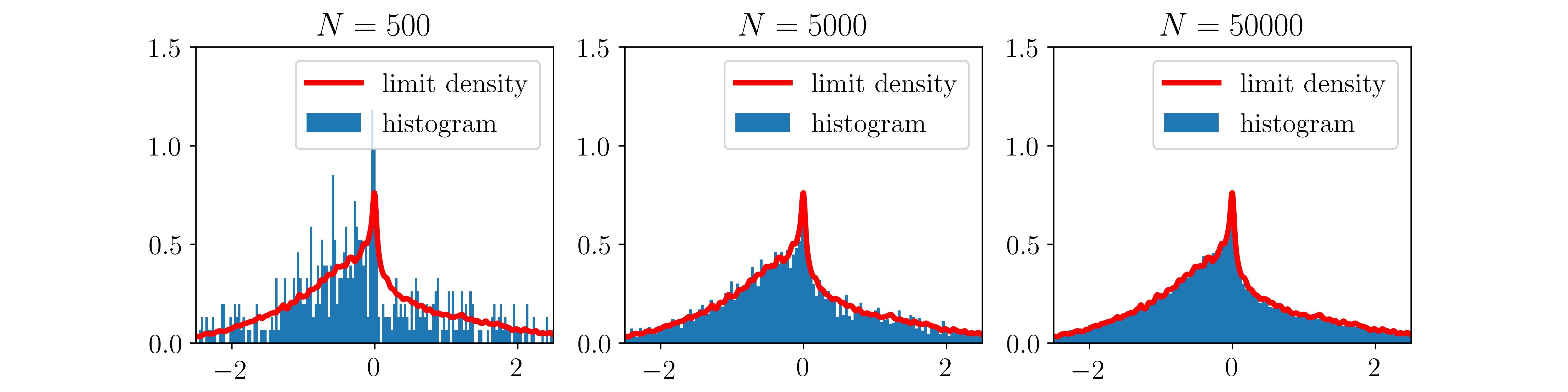}
	\includegraphics[width=1.2\linewidth]{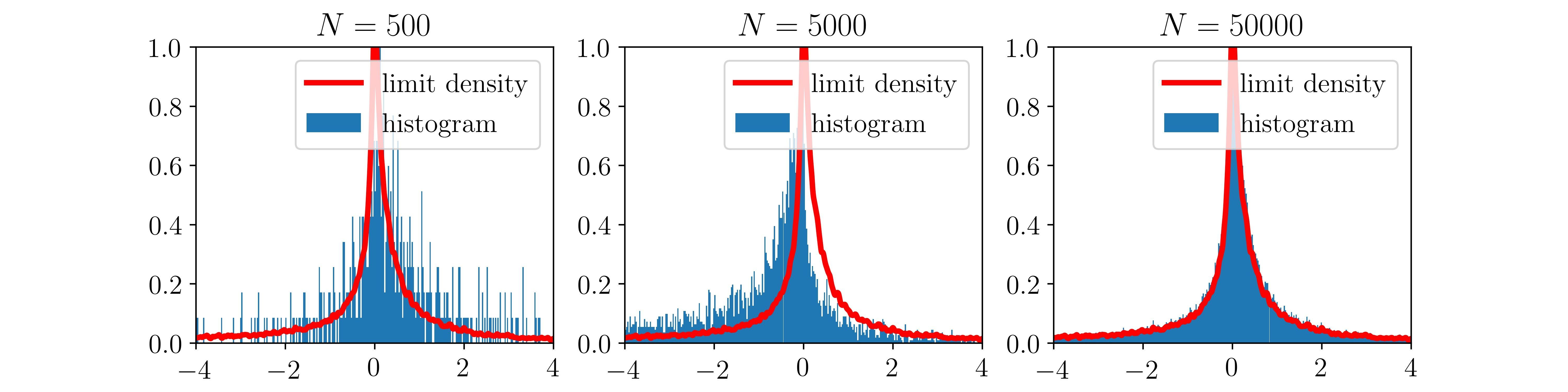}
	\caption{Convergence of the weights of the first layer as $N \to \plusinfty$ for $\alpha=0$ and $M=100$. The first line corresponds to $\beta=0.25$, the second to $\beta=0.5$, the third to $\beta=0.75$ and the last line to $\beta=1.0$.
	\label{fig:cv_M100_alpha0_layer2}}
\end{figure}

On \Cref{fig:cv_M100_alpha025_layer1} and \Cref{fig:cv_M100_alpha025_layer2} we show the results of the exact same experiments but this time using decreasing stepsizes and a parameter $\alpha=0.25$. Once again our experiments illustrate the convergence of the empirical distributions to some limiting distribution, and we can also identify two regimes. Note that the limiting distribution satisfying \eqref{eq:evolution_equation_beta_min} or \eqref{eq:evolution_equation_beta_one} (depending on the value of $\beta$), it depends on the parameter $\alpha$. Therefore the limiting distribution obtained in the case where $\alpha=0.25$ is different from the one obtained when $\alpha=0$. This is particularly visible in the case where $\beta=1$ (as shown in green on  \Cref{fig:cv_M100_alpha025_layer1} and \Cref{fig:cv_M100_alpha025_layer2}).

\begin{figure}
	\centering
	\centerfloat
	\includegraphics[width=1.2\linewidth]{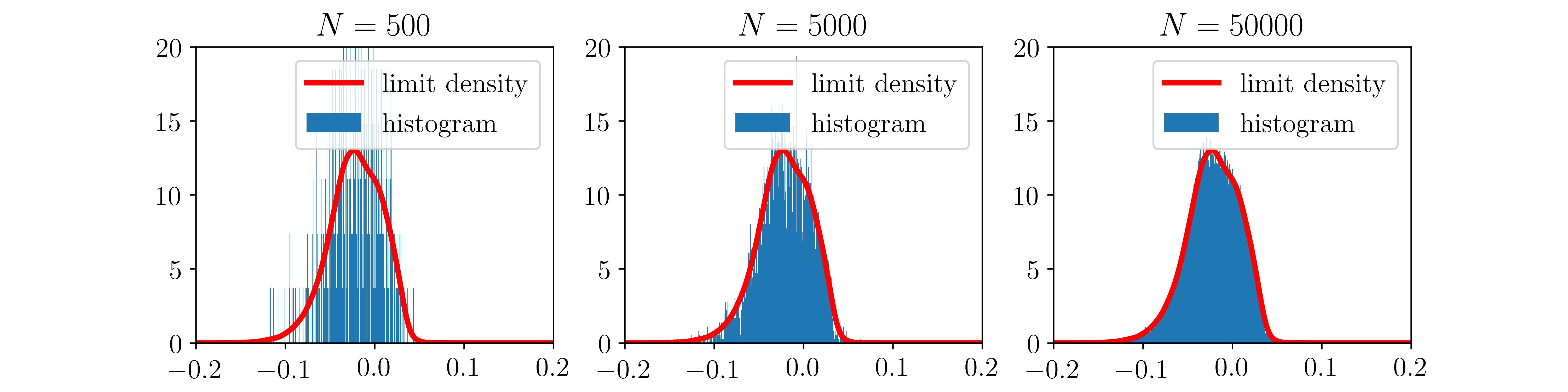}
	\includegraphics[width=1.2\linewidth]{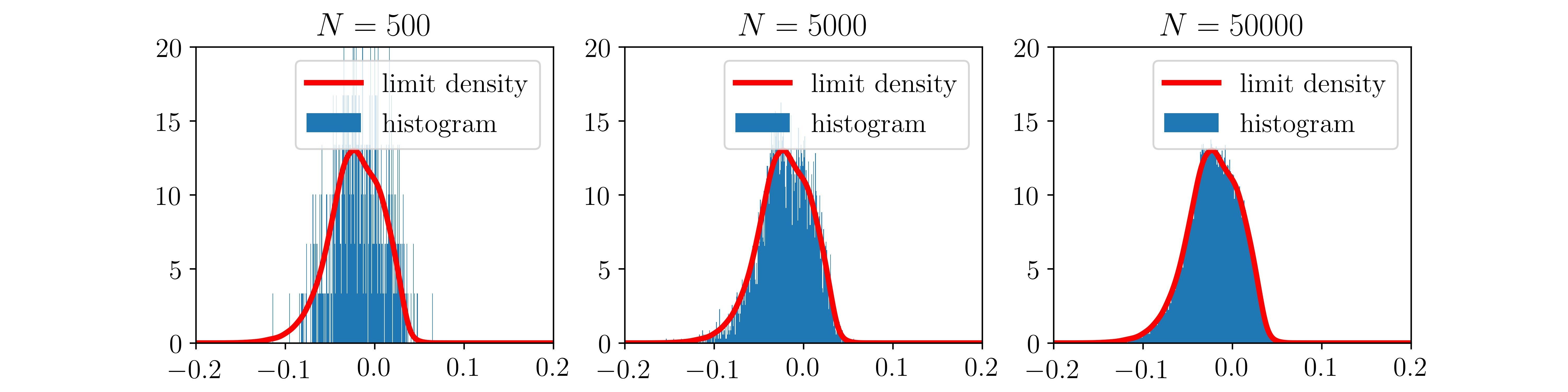}
	\includegraphics[width=1.2\linewidth]{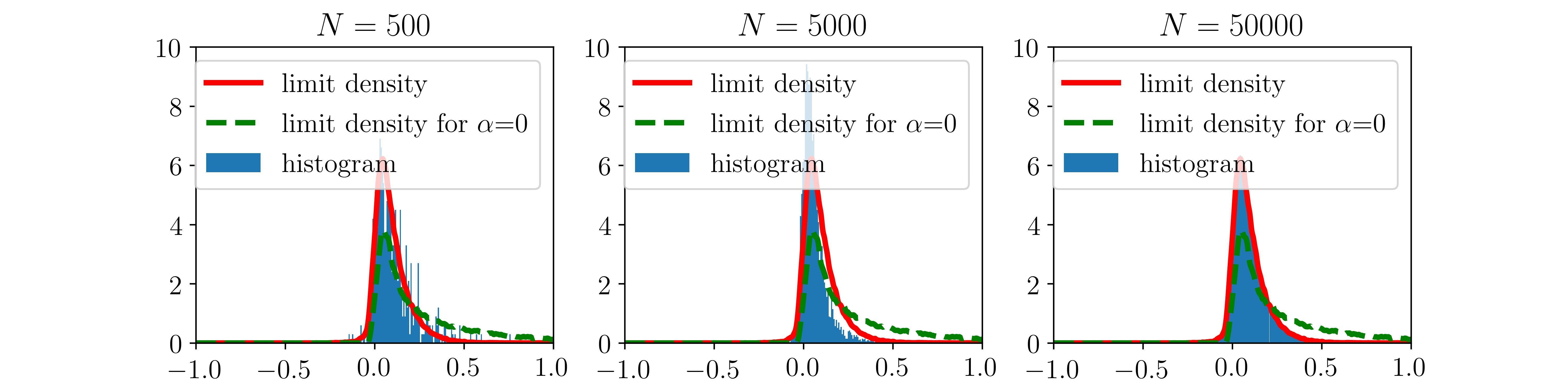}
	\caption{Convergence of the weights of the first layer as $N \to \plusinfty$ for $\alpha=0.25$ and $M=100$. The first line corresponds to $\beta=0.5$, the second to $\beta=0.75$ and the last line to $\beta=1.0$.
	\label{fig:cv_M100_alpha025_layer1}}
\end{figure}

\begin{figure}
	\centering
	\centerfloat
	\includegraphics[width=1.2\linewidth]{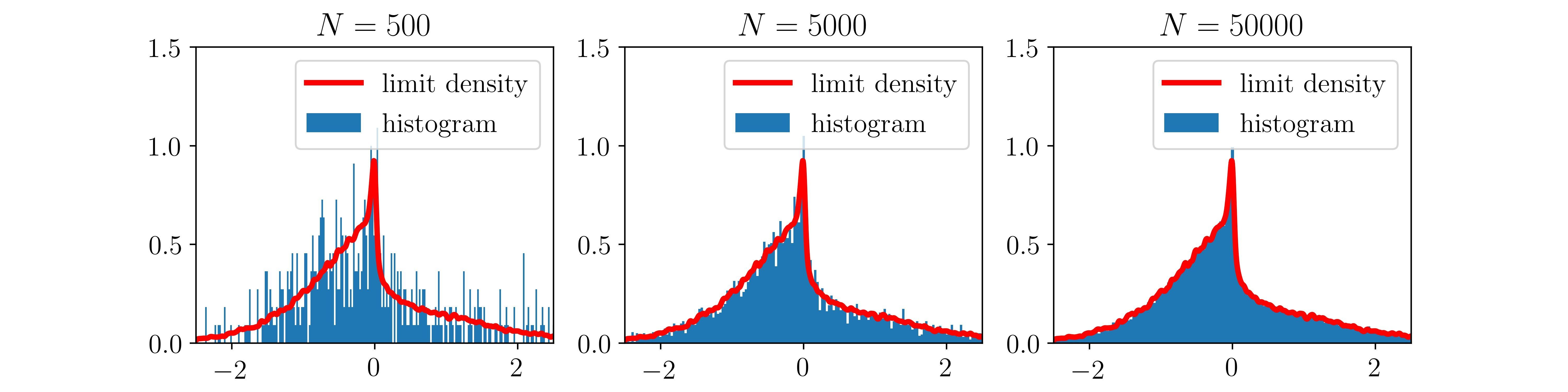}
	\includegraphics[width=1.2\linewidth]{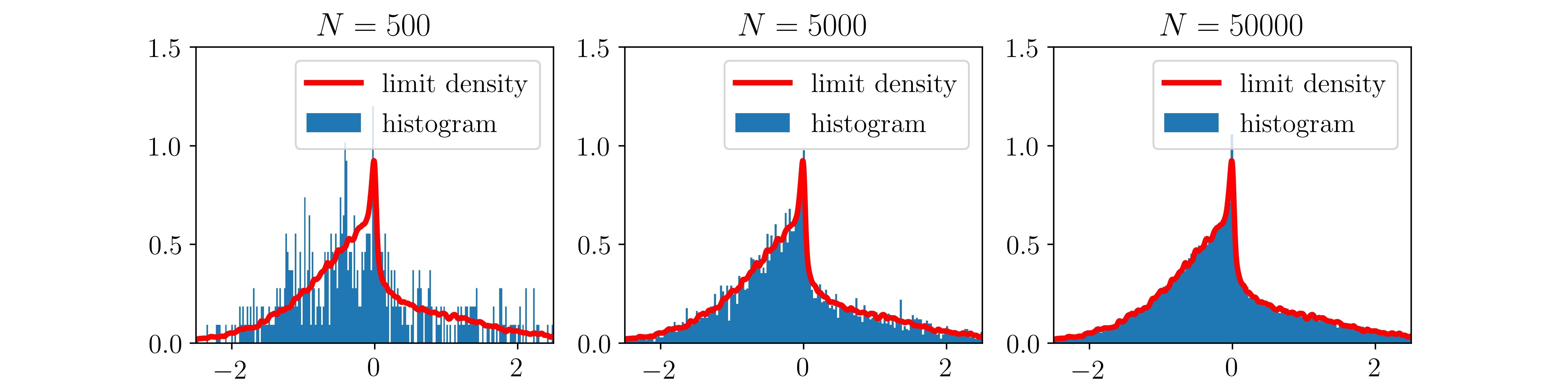}
	\includegraphics[width=1.2\linewidth]{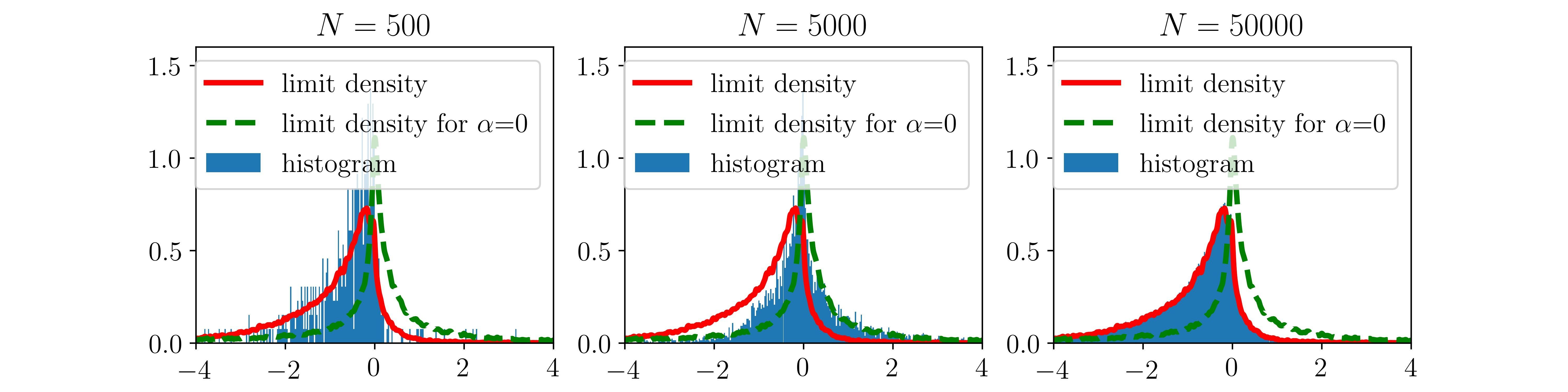}
	\caption{Convergence of the weights of the first layer as $N \to \plusinfty$ for $\alpha=0.25$ and $M=100$. The first line corresponds to $\beta=0.5$, the second to $\beta=0.75$ and the last line to $\beta=1.0$.
	\label{fig:cv_M100_alpha025_layer2}}
\end{figure}

We now study the role of the batch size $M$ on the convergence toward the mean-field regime. \Cref{fig:CV_BS} illustrates the convergence of the empirical measures in the case where $\beta<1$ (here $\beta=0.75$) of the weights of the hidden layer of the neural network, for a fixed number of neurons $N=10000$ for different batch sizes $M$. We indeed observe convergence with $M$.

\begin{figure}
	\centering
	\centerfloat
	\includegraphics[width=1.1\linewidth]{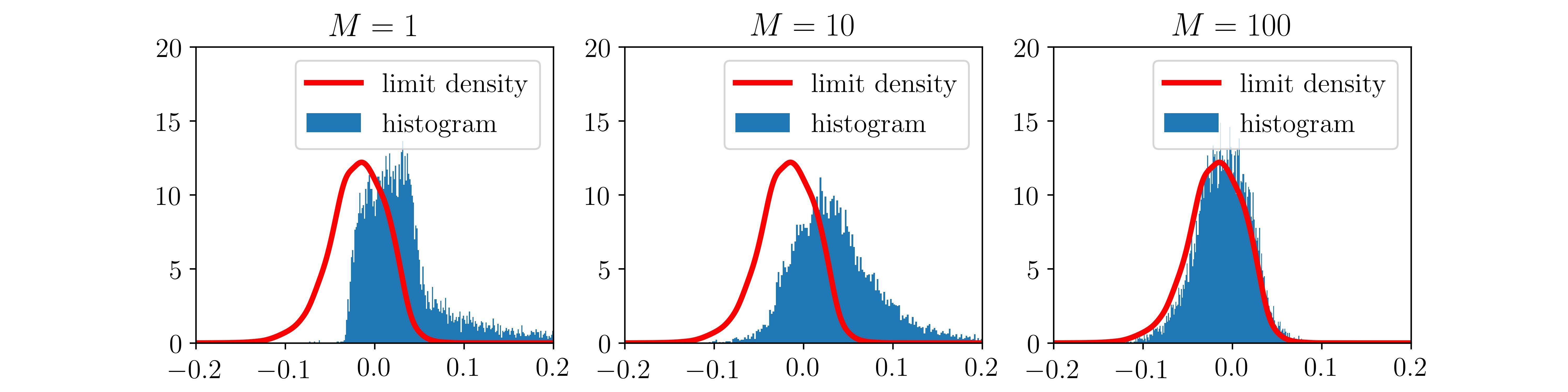}
	\caption{Convergence of the weights as $M \to \infty$\label{fig:CV_BS}}
\end{figure}


\end{document}